%% file: main.tex
\documentclass[12pt]{article}

\usepackage{enumerate}
\usepackage[OT1]{fontenc}
\usepackage{amsmath,amssymb}
\usepackage{natbib}
\usepackage[usenames]{color}

\usepackage{smile}
\usepackage{multirow}
\usepackage[breaklinks,colorlinks,
            linkcolor=red,
            anchorcolor=blue,
            citecolor=blue
            ]{hyperref}
            
\usepackage{smile}
\def\##1\#{\begin{align}#1\end{align}}
\def\$#1\${\begin{align*}#1\end{align*}}

\def\tr{\mathop{\text{tr}}\kern.2ex}

\long\def\comment#1{}

\def\tr{\mathop{\text{Tr}}}

\def\cS{{\mathcal{S}}}
\def\cK{{\mathcal{K}}}
\def\cX{{\mathcal{X}}}
\def\cD{{\mathcal{D}}}
\def\cP{{\mathcal{P}}}
\def\cL{{\mathcal{L}}}

\def\cT{{\mathcal{T}}}
\def\cB{{\mathcal{B}}}
\def\tr{{\text{Tr}}}

\newcommand{\bel}{\begin{eqnarray}\label}
\newcommand{\eel}{\end{eqnarray}}
\newcommand{\bes}{\begin{eqnarray*}}
\newcommand{\ees}{\end{eqnarray*}}

\def\algoname{ \text{OP-TENET}}

\newcommand{\la}{\langle}
\newcommand{\ra}{\rangle}

\usepackage{fullpage}

\let\emptyset\varnothing

\def \to {{\tilde{o}}}
\def\mO{{\mathbb{O}}}

\def\mZ{{\mathbb{Z}}}
\def\mB{{\mathbb{B}}}

\def\overtau{{\overline{\tau}}}
\def\overbtau{{\overline{\btau}}}

\usepackage{tikz}
\usetikzlibrary{matrix}
\usetikzlibrary{bayesnet}
\usetikzlibrary{arrows}
\usepackage{color}
\usepackage{graphicx}
\usepackage{caption}
\usetikzlibrary{backgrounds}

\begin{document}

\title{Reinforcement Learning from Partial Observation:\\ Linear Function Approximation with\\ Provable Sample Efficiency}
\author
{\normalsize
Qi Cai\thanks{Northwestern University; 
\texttt{qicai2022@u.northwestern.edu}}
\qquad Zhuoran Yang\thanks{Yale University; 
\texttt{zhuoran.yang@yale.edu}}
\qquad Zhaoran Wang\thanks{Northwestern University; 
\texttt{zhaoranwang@gmail.com}}
}
\date{}
\maketitle

\input{intro}
\input{body}

\section*{Acknowledgements}
Zhaoran Wang acknowledges National Science Foundation (Awards 2048075, 2008827, 2015568, 1934931), Simons Institute (Theory of Reinforcement Learning), Amazon, J.P. Morgan, and Two Sigma for their supports.

\newpage

\bibliographystyle{ims}
\bibliography{graphbib2.bib}

\newpage
\appendix
\input{appendix}

\end{document}

%% file: intro.tex

\begin{abstract}
We study reinforcement learning for partially observed Markov decision processes (POMDPs) with infinite observation and state spaces, which remains less investigated theoretically. To this end, we make the first attempt at bridging partial observability and function approximation for a class of POMDPs with a linear structure. In detail, we propose a reinforcement learning algorithm (\underline{Op}timistic Explora\underline{t}ion via Adv\underline{e}rsarial I\underline{n}tegral Equa\underline{t}ion or $\algoname$) that attains an $\epsilon$-optimal policy within $O(1/\epsilon^2)$ episodes. In particular, the sample complexity scales polynomially in the intrinsic dimension of the linear structure and is independent of the size of the observation and state spaces. 

The sample efficiency of $\algoname$ is enabled by a sequence of ingredients: (i) a Bellman operator with finite memory, which represents the value function in a recursive manner, (ii) the identification and estimation of such an operator via an adversarial integral equation, which features a smoothed discriminator tailored to the linear structure, and (iii) the exploration of the observation and state spaces via optimism, which is based on quantifying the uncertainty in the adversarial integral equation.
\end{abstract}

\section{Introduction}

Partial observability poses significant challenges for reinforcement learning, especially when the observation and state spaces are infinite. Given full observability, reinforcement learning is well studied empirically \citep{mnih2015human, silver2016mastering, silver2017mastering} and theoretically \citep{auer2008near, osband2016generalization, azar2017minimax, jin2018q, yang2020reinforcement, jin2020provably, ayoub2020model, kakade2020information, du2021bilinear}. In particular, for infinite state spaces, neural function approximators achieve remarkable successes empirically \citep{mnih2015human, berner2019dota, arulkumaran2019alphastar}, while linear function approximators become better understood theoretically \citep{yang2020reinforcement, jin2020provably, ayoub2020model, kakade2020information, du2021bilinear}. In contrast, reinforcement learning in partially observed Markov decision processes (POMDPs) is less investigated theoretically despite its prevalence in practice \citep{cassandra1996acting, hauskrecht2000planning, brown2018superhuman, rafferty2011faster}. 

More specifically, partial observability poses both statistical and computational challenges. From a statistical perspective, it is challenging to predict future rewards, observations, or states due to a lack of the Markov property. In particular, predicting the future often involves inferring the distribution of the state (also known as the belief state) or its functionals as a summary of the history, which is already challenging even assuming the (observation) emission and (state) transition kernels are known \citep{vlassis2012computational, golowich2022planning}. Meanwhile, learning the emission and transition kernels faces various issues commonly encountered in causal inference \citep{zhang2016markov}. For example, they are generally nonidentifiable \citep{kallus2021causal}. Even assuming they are identifiable, their estimation possibly requires a sample size that scales exponentially in the horizon and dimension \citep{jin2020sample}. Such statistical challenges are already prohibitive even for the evaluation of a policy \citep{nair2021spectral, kallus2021causal, bennett2021proximal}, which forms the basis of policy optimization. From a computational perspective, it is known that policy optimization is generally intractable \citep{vlassis2012computational, golowich2022planning}. Moreover, infinite observation and state spaces amplify both statistical and computational challenges. On the other hand, most existing results are restricted to the tabular setting \citep{azizzadenesheli2016reinforcement, guo2016pac, jin2020sample, xiong2021sublinear}, where the observation and state spaces are finite.

In this paper, we study linear function approximation in POMDPs to address the statistical challenges amplified by infinite observation and state spaces. In particular, our contribution is fourfold. First, we define a class of POMDPs with a linear structure and identify an ill conditioning measure for sample-efficient reinforcement learning. Such an ill conditioning measure corresponds to the undercompleteness in the tabular setting \citep{jin2020sample}. Second, we propose a reinforcement learning algorithm ($\algoname$), which applies to any POMDP admitting the aforementioned linear structure. Moreover, we use a minimax optimization formulation in $\algoname$ such that the algorithm can be implemented in a computation-efficient manor even if the dataset is large. Third, we prove in theory that $\algoname$ attains an $\epsilon$-optimal policy within $O(1/\epsilon^2)$ episodes. In particular, the sample complexity scales polynomially in the intrinsic dimension of the linear structure and is independent of the size of the observation and state spaces. Fourth, our algorithm and analysis are based on new tools. In particular, the sample efficiency of $\algoname$ is enabled by a sequence of ingredients: (i) a Bellman operator with finite memory, which represents the value function in a recursive manner, (ii) the identification and estimation of such an operator via an adversarial integral equation, which features a smoothed discriminator tailored to the linear structure, and (iii) the exploration of the observation and state spaces via optimism, which is based on quantifying the uncertainty in the adversarial integral equation.

\subsection{Related Work}

Our work is related to a line of recent work on the sample efficiency of reinforcement learning for POMDPs. In detail, \citet{azizzadenesheli2016reinforcement, guo2016pac, xiong2021sublinear} establish sample complexity guarantees for searching the optimal policy in POMDPs whose models are identifiable and can be estimated by spectral methods. However, \citet{azizzadenesheli2016reinforcement} and \citet{guo2016pac} add extra assumptions such that efficient exploration of the POMDP can always be achieved by running arbitrary policies. In contrast, the upper bound confidence (UCB) method is used in \citet{xiong2021sublinear} for adaptive exploration. However, they require strictly positive state transition and observation emission kernels to ensure fast convergence to the stationary distribution. The more related work is \citet{jin2020sample}, which considers undercomplete POMDPs, in other words, the observations are more than the latent states. Their proposed algorithm can attain the optimal policy without estimating the exact model, but an observable component \citep{jaeger2000observable, hsu2012spectral}, which is the same for our algorithm design, while only applies to tabular POMDPs. 

In a broader context of reinforcement learning with partial observability, our work is related to several recent works on POMDPs with special structures. For example, \citet{kwon2021rl} considers latent POMDPs, where each process has only one latent state, and the proposed algorithm efficiently infers the latent state using a short trajectory. \citet{kozuno2021model} considers POMDPs having tree-structured states with their positions in certain partitions being the observations. Compared with general POMDPs, these specially structures reduce the complexity of finding the optimal actions, and the corresponding algorithms use techniques closer to those for MDPs. Also, the aforementioned literature only consider tabular POMDPs.

In the contexture of reinforcement learning with function approximations, our work is related to a vast body of recent progress \citep{yang2020reinforcement, jin2020provably,  cai2020provably, du2021bilinear, kakade2020information, agarwal2020pc, zhou2021nearly, ayoub2020model} on the sample efficiency of reinforcement learning for MDPs with linear function approximations. These works characterize the uncertainty in the regression for estimating either the model or value function of an MDP and use the uncertainty as a bonus on the rewards to encourage exploration. However, none of these approaches directly apply to POMDPs due to the latency of the states.

\subsection{Notation} For any discrete or continuous set $\cX$ and $p\in\NN$, we denote by $L^p(\cX)$ the $L^p$ space of functions over $\cX$, and $\Delta(\cX)$ the set of probability density functions over $\cX$ when $\cX$ is continuous or probability mass functions when $\cX$ is discrete. 
For any $d\in\NN$, we denote by $[d]$ the set of integers from $1$ to $d$. For a vector $v$ and a matrix $M$, we denote by $[v]_i$ the $i$-th entry of $v$ and $[M]_{i,j}$ the entry of $M$ at the $i$-th row and $j$-th column. We denote by $\|\cdot\|_p$ the $\ell^p$-norm of a vector or $L^p$-norm of a function. Also, for an operator $M$, we denote by $\|M\|_{p\mapsto q}$ the operator norm of $M$ induced by the $\ell^p$-norm or $L^p$-norm of the domain and $\ell^q$-norm or $L^q$-norm of the range. We use the notation ${\rm linspan}(\cdot)$ and ${\rm conh}(\cdot)$ to represent the linear span and convex combination, respectively.

%% file: body.tex

\section{Background}
\subsection{POMDPs}\label{prel1}

\begin{figure}
\begin{center}
 \tikz{
 \node[latent, minimum size=10mm] (s1) {$\bs_1$};
 \node[latent, xshift=3cm, minimum size=10mm] (s2) {$\bs_2$};
 \node at (5, 0) {$\cdots$};\node at (5, -1.5) {$\cdots$};
 \node at (5, -3) {$\cdots$};\node at (5, -4.5) {$\cdots$};
 \node at (12, 0) {$\cdots$};\node at (12, -1.5) {$\cdots$};
 \node at (12, -3) {$\cdots$};\node at (12, -4.5) {$\cdots$};
 \node[latent, xshift=7cm, minimum size=10mm] (sh) {$\bs_h$};
 \node[latent, xshift=10cm, minimum size=10mm] (sh+1) {$\bs_{h+1}$};
 \node[obs, yshift=-3cm, minimum size=10mm] (o1) {$\bo_1$};
 \node[obs, xshift=3cm, yshift=-3cm, minimum size=10mm] (o2) {$\bo_2$};
 \node[obs, xshift=7cm, yshift=-3cm, minimum size=10mm] (oh) {$\bo_h$};
 \node[obs, xshift=10cm, yshift=-3cm, minimum size=10mm] (oh+1) {$\bo_{h+1}$};
 \node[obs, xshift=1.5cm, yshift=-1.5cm, minimum size=10mm] (a1) {$\ba_1$};
 \node[obs, xshift=8.5cm, yshift=-1.5cm, minimum size=10mm] (ah) {$\ba_h$};
 \node[obs, xshift=1.5cm, yshift=-4.5cm, minimum size=10mm] (r1) {$\br_1$};
 \node[obs, xshift=8.5cm, yshift=-4.5cm, minimum size=10mm] (rh) {$\br_h$};
 \node at (-1.2, -3) {$\btau_h$};\node at (-1.2, -1.5) {$\overbtau_h$};
 \draw [black] (7.8, -2.4) rectangle (-2, -3.6);
 \draw [black] (7.9, -0.7) rectangle (-2.1, -3.7);
 \edge{s1}{o1};\edge{s2}{o2};\edge{sh}{oh};\edge{sh+1}{oh+1};
 \edge{s1,a1}{s2};\edge{sh, ah}{sh+1};
 \edge{o1,a1}{r1};\edge{oh, ah}{rh};
 \edge{o1}{a1};
 \draw[->] (7.6, -2.4)--(ah)
 }
\end{center}
\caption{Directed acyclic graph of a POMDP. Here, we denote by $\btau_h=(\bo_1,\ldots,\bo_h)$ the observation history and $\overbtau_h=(\bo_1,\ba_1,\ldots,\bo_{h-1},\ba_{h-1},\bo_h)$ the full history. See Section \ref{prel1} for more details.}
\label{fig1}
\end{figure}


We consider an episodic POMDP $(\cS, \cA, \cO, H, \cT, \cE, \mu, r)$, where $\cS$, $\cA$, and $\cO$ are the state, action, and observation spaces, respectively, $H$ is the length of each episode, $\cT$ is the state transition kernel from a state-action pair to the next state, $\cE$ is the observation emission kernel from a state to its observation, $\mu$ is the initial state distribution, and $r: \cO\times\cA\rightarrow[0,1]$ is the reward function defined on the observation and action for each step. We assume that the action space $\cA$ has a finite size $A\in\NN$, but the state space $\cS$ and observation space $\cO$ can be infinite (with finite dimensions). Also, we consider the nonhomongeneous setting so that the state transition kernel and observation emission kernel can be different across each step. Hence, we use a subscript $h\in\NN$ to index the step. At the beginning of each episode, the agent receives the initial state $\bs_1\sim\mu$. Then, the agent interacts with the environment as follows. At the $h$-th step, the agent receives the observation $\bo_h\sim\cE_h(\cdot\,|\,\bs_h)$, takes an action $\ba_h$ based on the observation history 
\#\label{taudef}
\btau_h=(\bo_1,\ldots,\bo_h),
\#
and receives the reward $\br_h=r(\bo_h,\ba_h)$. Any mapping $\pi$ from the observation history to the action is called a (deterministic) policy. We denote by $\Pi$ the set of all such mappings. Note that the policy does not use the action history as an input. Such a restriction does not exclude the optimal policy, as the action history can be decoded from the observation history. Subsequently, the agent receives the next state $\bs_{h+1}$ following $\bs_{h+1}\sim\cT_h(\cdot\,|\,\bs_h,\ba_h)$. See Figure \ref{fig1} for an illustration.

In a reinforcement learning problem, the environment is unknown, that is, the state transition kernel $\cT$ and observation emission kernel $\cE$ are unknown. We denote by $\{(\cT^{\theta}, \cE^\theta): \theta\in\Theta\}$ the candidate class of $\cT$ and $\cE$, where $\theta$ is the parameter and $\Theta$ is the set of the parameter. We assume that the realizability condition holds, that is, there exists a parameter $\theta^*\in\Theta$ such that $\cT=\cT^{\theta^*}$ and $\cE=\cE^{\theta^*}$. Without loss of generality and for ease of presentation, we assume that $\mu$, $\cT_1$, $\cE_1$, and $\cE_2$ are known, which only account for the initialization. The goal is to find a policy that maximizes the expected total reward, that is,
\#\label{09061110}
&\pi^*=\argmax_{\pi\in\Pi}J(\theta^*,\pi), \notag\\
&\text{where}~
J(\theta,\pi)=\EE_{\theta,\pi}\Bigl[ \sum_{h=1}^H \br_h \Bigr]~ 
\text{for any $(\theta,\pi)\in\Theta\times\Pi$.}
\#
Here, we write $\theta$ and $\pi$ as the subscripts of the expectation to denote that the parameter of the state transition kernel and observation emission kernel have the parameter $\theta$ and the actions follow the policy $\pi$. In the sequel, we drop the subscript $\pi$ if the expectation does not depends on it.

\noindent
\textbf{Additional Notation:} Recall that we denote by $\Pi$ the set of all policies. For notational simplicity, we denote by $\overline{\Pi}$ the set of mixing policies. A mixing policy selects a policy from $\Pi$ randomly and executes such a policy throughout the episode. For any $h\in\NN$, we denote by $\overbtau_h$ the full history,
\#\label{def-full-history}
\overbtau_h=(\bo_1,\ba_1,\ldots,\bo_{h-1},\ba_{h-1},\bo_h),
\#
which includes the action history. We denote by $\Gamma_h$ and $\overline{\Gamma}_h$ the sets of all histories $\btau_h$ and $\overline{\btau}_h$, respectively. Throughout the paper, we use bold letters for states, actions, and observations to emphasize that they are random variables in a POMDP, whose parameter and policy are specified in the context, while we use regular letters when they are deterministic values.

\subsection{Linear Function Approximations}\label{prel2}
We specify the candidate class of the state transition kernel $\cT$ and observation emission kernel $\cE$. We define the following function classes of the conditional state distribution. In detail, we define
\$
\cF_{\rm s}&
=\{ p_{\theta}(\bs_h=\cdot\,|\,\bs_{h-1}=s,\ba_{h-1}=a): (h,\theta,s,a)\in[H]\times\Theta\times\cS\times\cA\}, \\
\cF'_{\rm s}&=\{
 p_{\theta,\pi}(\bs_h=\cdot\,|\,\bo_{h+1}=o,\ba_h=a):   (h,\theta,o,a,\pi)\in[H]\times\Theta\times\cO\times\cA\times\overline{\Pi} \}.
\$
Here, $p(\cdot)$ is the probability density function when the state space $\cS$ is continuous and the probability mass function when $\cS$ is discrete. The subscripts $\theta$ and $\pi$ follow from \eqref{09061110}. Note that conditioning on $\ba_h=a$ means that the agent takes the action $a$ at the $h$-th step regardless of the observation history, while the agent takes the other actions following the policy $\pi$ as specified in the subscript. In the causal inference literature \citep{pearl2009causal}, our notation corresponds to ${\rm do}(\ba_h=a)$, which denotes the interventional distribution and differs from the observational distribution. Throughout this paper, we follow such a convention. Also, note that $\cF_{\rm s}$ corresponds to the state distribution conditioning on the past, while $\cF'_{\rm s}$ corresponds to that conditioning on the future. As a special case, we have $\mu\in\cF_{\rm s}$ for $h=1$ since $\bs_0$ and $\ba_0$ do not exist. We define the following function class of the conditional observation distribution,
\$
\cF_{\rm o}&=\{
p_{\theta,\pi}(\bo_{h:h+2}=\cdot 
\,\bigl|\,\ba_h=a, \ba_{h+1}=a'):  
(h,\theta,a,a',\pi)\in[H]\times\Theta\times\cA^2\times\overline{\Pi} \}.
\$
The following assumption restricts the above function classes to two low-dimensional subspaces.

\begin{assumption} [Linear Function Approximations]\label{asmp1}
There exist $d_{\rm s}, d_{\rm o}\in\NN$ and known distribution functions $\{\psi_i\}_{i=1}^{d_{\rm s}}\subset\Delta(\cS)$ and $\{\phi_i\}_{i=1}^{d_{\rm o}}\subset\Delta(\cO^3)$ such that we have 
\begin{itemize}
\item $\cF_{\rm s},\cF'_{\rm s}\subset{\rm conh}(\{\psi_i\}_{i=1}^{d_{\rm s}})$, \item $\cF_{\rm o}\subset{\rm conh}(\{\phi_i\}_{i=1}^{d_{\rm o}})$.
\end{itemize}
\end{assumption}

For ease of presentation, we denote $\{\psi_i\}_{i=1}^{d_{\rm s}}$ and $\{\phi_i\}_{i=1}^{d_{\rm o}}$ by $\psi$ and $\phi$, respectively, for the rest of the paper. Assumption \ref{asmp1} requires that $\cF_{\rm s}$, $\cF'_{\rm s}$, and $\cF_{\rm o}$ are linearly represented by known bases $\psi$ and $\phi$. See, for example, \citet{du2021bilinear} for the corresponding assumption in MDPs. Note that, when $\psi$ and $\phi$ are the one-hot functions over $\cS$ and $\cO^3$, respectively, we recover the tabular setting \citep{jin2020sample}.

The following assumption ensures that the observation is informative for the state. For any $(h,\theta)\in[H]\times\Theta$, we define the observation operator $\mO^\theta_{h}:L^1(\cS)\rightarrow L^1(\cO)$ by
\#\label{emissionopdef}
(\mO^\theta_{h}f)(o)
=\int_\cS \cE^\theta_{h}(o\,|\,s) \cdot f(s)\ud s, 
\quad \text{for any $f\in L^1(\cS)$ and $o\in\cO$},
\#
which maps a state distribution to the observation distribution. 

\begin{assumption}[Invertible Observation Operators]\label{asmp2}
For any $(h,\theta)\in[H]\times\Theta$, there exist a known function $\cZ^\theta_h:\cS\times\cO\rightarrow\RR$ and the linear operator $\mZ^\theta_{h}:L^1(\cO)\rightarrow L^1(\cS)$ defined by
\$
(\mZ^\theta_{h}f)(s)=\int_\cO \cZ^\theta_{h}(s,o) \cdot f(o)\ud o,\quad
\text{for any $f\in L^1(\cO)$ and $s\in\cS$}
\$ 
such that we have
\begin{itemize}
\item $\mZ^\theta_{h}\mO^\theta_{h}f=f$ for any $f\in{\rm linspan}(\psi)$,
\item $\|\mZ^\theta_{h}\|_{1\mapsto1}\le\gamma$ for a constant $\gamma>0$.
\end{itemize}
\end{assumption}

Assumption \ref{asmp2} requires that the observation operator $\mO^\theta_{h}$ defined on ${\rm linspan}(\psi)$ is injective, which implies that it has a left inverse $\mZ^\theta_{h}$. Note that the domain of $\mZ^\theta_{h}$ naturally extends to $L^1(\cO)$. In other words, the observation distribution carries the full information of the state distribution. The (upper bound of the) operator norm $\gamma$ is a measure of ill conditioning, which quantifies the fundamental difficulty, in terms of the information-theoretic limit, of reinforcement learning in the POMDP. See more discussion in Section \ref{linearpomdp}, where we prove that both Assumptions \ref{asmp1} and \ref{asmp2} hold if the state transition kernel and observation emission kernel admit certain a structure. Correspondingly, we provide a detailed form of the function $\cZ^\theta_h$ in Section \ref{linearpomdp}. Also, we illustrate the connection to the tabular setting \citep{jin2020sample} therein.

\section{Algorithm}\label{sec-algo}

In this section, we first introduce the finite-memory Bellman operator in Section \ref{algo1} and discuss its estimation in Section \ref{algo2}. Then, we present Algorithm \ref{algo}, which performs optimistic exploration on top of operator estimation, in Section \ref{algo3}.

\subsection{Finite-Memory Bellman Operator}\label{algo1}
To cast a POMDP as an MDP, it is necessary to aggregate the observation history and action history as the ``state" in an MDP to retrieve the Markov property. In detail, for any $(h,\theta,\pi)\in[H]\times\Theta\times\Pi$, we define the full-memory Bellman operator $\PP^{\theta,\pi}_{h}: L^\infty(\overline{\Gamma}_{h+1}) \rightarrow L^\infty(\overline{\Gamma}_{h})$ by
\#\label{914756}
(\PP^{\theta,\pi}_{h} f)(\overtau_{h})
&=\EE_{\theta,\pi}[ f(\overbtau_{h+1})  \,|\, \overbtau_{h}=\overtau_{h}]\\
&=\int_{\cO}f\bigl(\overtau_{h}, \pi(\tau_h), o_{h+1}\bigr)  \cdot p_{\theta}
\bigl(\bo_{h+1}=o_{h+1} \, |\, 
\overbtau_{h}=\overtau_{h}, \ba_h=\pi(\tau_h) \bigr) \ud o_{h+1}, \notag
\#
for any $f\in L^\infty(\overline{\Gamma}_{h+1})$ and $\overtau_{h}\in\overline{\Gamma}_{h}$. Here, the second equality follows from $\overbtau_{h+1}=(\overbtau_h, \ba_h, \bo_{h+1})$ with $\ba_h=\pi(\overbtau_h)$, which is defined in \eqref{def-full-history}. In the sequel, the function $f$ is set as the expected total reward conditioning on the $(h+1)$-step full history $\overtau_{h+1}\in\overline{\Gamma}_{h+1}$ and $\PP^{\theta,\pi}_{h}$ maps it to the $h$-step counterpart, which resembles backward induction or dynamic programming in MDPs. We denote by $R: \overline{\Gamma}_{H+1}\rightarrow [0, H]$ the function that maps the $(H+1)$-step full history to the total reward, that is,
\#\label{1004937}
R(\overtau_{H+1})=r(o_1,a_1)+\cdots+r(o_{H},a_{H}),
\quad
\text{for any $\overtau_{H+1}\in\overline{\Gamma}_{H+1}$.}
\#
For any $h\in[H]$, the expected total reward satisfies
\#\label{914808}
\EE_{\theta,\pi}\Bigl[\sum_{i=1}^H \br_i \,\Big|\,\overbtau_h=\overtau_h\Bigr]=(\PP^{\theta,\pi}_{h}\cdots \PP^{\theta,\pi}_{H} R)(\overtau_h),
\quad
\text{for any $\overtau_h\in\overline{\Gamma}_h$,}
\#
where the equality follows from recursively applying \eqref{914756} and the tower property of conditional expectation. A direct idea is to evaluate a policy $\pi$ by estimating the parameter $\theta$ in \eqref{914808} and optimize $\pi$ in an iterative manner. However, estimating the operator $\PP^{\theta,\pi}_{h}$ suffers from the curse of dimensionality since it requires estimating a distribution conditioning on the $h$-step full history $\overbtau_h$, which is high-dimensional.

\begin{figure}[!ht]
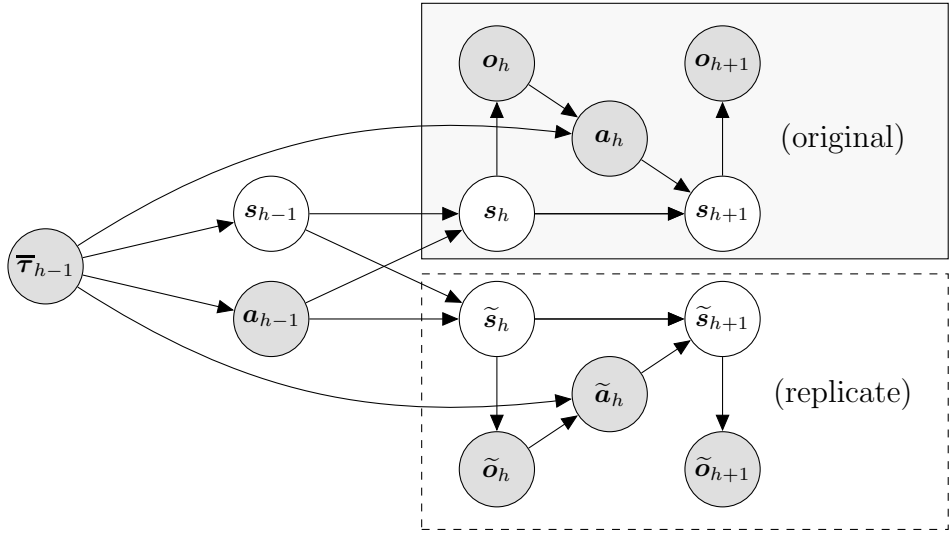

\begin{center}
 \tikz{
  \draw [draw=black, fill=black!3] (12, 3.5) rectangle (5, 0.1);
 \node[obs, minimum size=1cm] (th-1) {$\overbtau_{h-1}$};
 \node[latent, xshift=3cm, yshift=0.7cm, minimum size=1cm] (sh-1) {$\bs_{h-1}$};
 \node[latent, xshift=6cm, yshift=0.7cm, minimum size=1cm] (sh) {$\bs_{h}$};
 \node[latent, xshift=9cm, yshift=0.7cm, minimum size=1cm] (sh+1) {$\bs_{h+1}$};
 \node[latent, xshift=6cm, yshift=-0.7cm, minimum size=1cm] (tsh) {$\tilde{\bs}_{h}$};
 \node[latent, xshift=9cm, yshift=-0.7cm, minimum size=1cm] (tsh+1) {$\tilde{\bs}_{h+1}$};
 \node[obs, xshift=3cm, yshift=-0.7cm, minimum size=1cm] (ah-1) {$\ba_{h-1}$};
 \node[obs, xshift=6cm, yshift=2.7cm, minimum size=1cm] (oh) {$\bo_h$};
 \node[obs, xshift=9cm, yshift=2.7cm, minimum size=1cm] (oh+1) {$\bo_{h+1}$};
 \node[obs, xshift=6cm, yshift=-2.7cm, minimum size=1cm] (toh) {$\tilde{\bo}_h$};
 \node[obs, xshift=9cm, yshift=-2.7cm, minimum size=1cm] (toh+1) {$\tilde{\bo}_{h+1}$};
 \node[obs, xshift=7.5cm, yshift=1.7cm, minimum size=1cm] (ah) {$\ba_h$};
 \node[obs, xshift=7.5cm, yshift=-1.7cm, minimum size=1cm] (tah) {$\tilde{\ba}_h$};
 \node at (10.6, -1.7) {(replicate)}; \node at (10.6, 1.7) {(original)}; 
 \edge{th-1}{sh-1}; \edge{th-1}{ah-1};
 \edge{sh-1,ah-1}{sh}; \edge{sh-1,ah-1}{tsh}; 
 \edge{sh}{sh+1}; \edge{sh}{oh}; \edge{sh,ah}{sh+1};\edge{sh+1}{oh+1};
 \edge{tsh}{tsh+1}; \edge{tsh}{toh}; \edge{tsh,tah}{tsh+1};\edge{tsh+1}{toh+1};
 \path[every node/.style={font=\sffamily\small}] (th-1) edge[->, bend left=20] (ah);
 \path[every node/.style={font=\sffamily\small}] (th-1) edge[->, bend right=20] (tah);
 \edge{oh}{ah};\edge{toh}{tah};
 \draw [dashed] (12, -0.1) rectangle (5, -3.5);
 }
\end{center}
\caption{Illustration of the variables in the definition of $\mB^{\theta,\pi}_h$ in \eqref{bdef} and \eqref{b-func-def}. In detail, $\tilde{\bs}_h$ is an independent replicate of $\bs_{h}$, that is, they are independent and identically distributed conditioning on $\bs_{h-1}$ and $\ba_{h-1}$. Note that $\tilde{\bs}_h$ is constructed for ease of presentation, and does not exist in practice. Then, the action $\tilde{\ba}_{h}$, state $\tilde{\bs}_{h+1}$, and observations $\tilde{\bo}_{h}, \tilde{\bo}_{h+1}$ are similarly defined. In other words, their distribution conditioning on $\tilde{\bs}_h$ and $\overbtau_{h-1}$ mirrors the distribution of the action $\ba_{h}$, state $\bs_{h+1}$, and observations $\bo_{h}, \bo_{h+1}$  conditioning on $\bs_h$ and $\overbtau_{h-1}$. For notational simplicity, we define the tail-mirrored full history $\overbtau_{h}^\dagger=(\overbtau_{h-1}, \ba_{h-1}, \tilde{\bo}_h)$ and tail-mirrored observation history $\btau_{h}^\dagger=(\btau_{h-1}, \tilde{\bo}_h)$.}
\label{fig2}
\end{figure}

\noindent
\textbf{From Full Memory to Finite Memory:} We propose to bypass such an issue by exploiting the independence between the past observation and future observation conditioning on the current state. In detail, for any $(h,\theta,\pi)\in[H]\times\Theta\times\Pi$, we define the finite-memory Bellman operator $\mB^{\theta,\pi}_h: L^\infty(\overline{\Gamma}_{h+1}) \rightarrow L^\infty(\overline{\Gamma}_{h})$ by
\#\label{bdef}
&(\mB^{\theta,\pi}_{h} f)(\overtau_{h})=\int_{\cO^2}
f(\overtau_{h}^\dagger, \pi(\tau_h^\dagger), \tilde{o}_{h+1})   \cdot \cB^{\theta}_{h, \pi(\tau_h^\dagger)}(o_h,\tilde{o}_h, \tilde{o}_{h+1})\ud \tilde{o}_h\ud \tilde{o}_{h+1}, 
\#
for any $f\in L^\infty(\overline{\Gamma}_{h+1})$ and $\overtau_{h}\in\overline{\Gamma}_h$. Here, the tail-mirrored full history $\overtau_{h}^\dagger$ and tail-mirrored observation history $\tau_{h}^\dagger$ are defined by
\#\label{tail-mirrored-def}
\overtau_{h}^\dagger=(\overtau_{h-1}, a_{h-1}, \tilde{o}_h),
\quad
\tau_{h}^\dagger=(\tau_{h-1}, \tilde{o}_h),
\#
which switch the last observation $o_h$ by $\tilde{o}_h$ in the full history $\overtau_h$ and observation history $\tau_h$, that is, $\overtau_h=(\overtau_{h-1}, a_{h-1}, o_h)$ and $\tau_h=(\tau_{h-1}, o_h)$. Also, the function $\cB^{\theta}_{h, a}: \cO^3\rightarrow\RR$ is defined by
\#\label{b-func-def}
\cB^\theta_{h,a}(o_h, \tilde{o}_h, \tilde{o}_{h+1}) 
&=\int_{\cS} 
p_{\theta}(\tilde{\bo}_h=\tilde{o}_h, \tilde{\bo}_{h+1}=\tilde{o}_{h+1} \,|\, \tilde{\bs}_h=\tilde{s}_h, \tilde{\ba}_h=a)
\cdot \cZ^\theta_{h}(\tilde{s}_h, o_h) \ud \tilde{s}_h,
\#
for any $o_h, \tilde{o}_h, \tilde{o}_{h+1}\in\cO$ and $a\in\cA$, where the function $\cZ^\theta_h$ is defined in Assumption \ref{asmp2}. Figure \ref{fig2} illustrates the (random) variables in \eqref{bdef} and \eqref{b-func-def}. In detail, $\tilde{\bs}_h$ is an independent replicate of $\bs_{h}$, that is, they are independent and identically distributed conditioning on $\bs_{h-1}$ and $\ba_{h-1}$. Note that $\tilde{\bs}_h$ is constructed for ease of presentation, and does not exist in practice. Then, the action $\tilde{\ba}_{h}$, state $\tilde{\bs}_{h+1}$, and observations $\tilde{\bo}_{h}, \tilde{\bo}_{h+1}$ are similarly defined. In other words, their distribution conditioning on $\tilde{\bs}_h$ and $\overbtau_{h-1}$ mirrors the distribution of the action $\ba_{h}$, state $\bs_{h+1}$, and observations $\bo_{h}, \bo_{h+1}$  conditioning on $\bs_h$ and $\overbtau_{h-1}$. When the state transition kernel and observation emission kernel have a specific parametrization, the function $\cZ^\theta_h$ has a corresponding parametrization by Assumption \ref{asmp2}, which allows us to parametrize the function $\cB^\theta_{h,a}$ in \eqref{b-func-def}. See Section \ref{linearpomdp} for an example where the state transition kernel and observation emission kernel admit a linear structure. Compared with the full-memory Bellman operator $\PP^{\theta,\pi}_{h}$, the finite-memory Bellman operator $\mB^{\theta,\pi}_h$ does not involve the distribution of $\bo_{h+1}$ conditioning on $\overbtau_h$ and $\ba_h$. Instead, it involves the distribution of $\tilde{\bo}_h$ and $\tilde{\bo}_{h+1}$ conditioning on $\tilde{\bs}_h$, where the distribution of $\tilde{\bs}_h$ is implied by the distribution of the single observation $\bo_h$ via the function $\cZ^\theta_h$. See the following paragraph for more discussion. Moreover, estimating $\mB^{\theta,\pi}_{h}$ for each $h\in[H]$ only involves the distribution of $\bo_{h-1}$, $\bo_h$, and $\bo_{h+1}$, which is low-dimensional. See Section \ref{algo2} for more discussion. 

\noindent
\textbf{How Finite Memory Works:}
For notational simplicity, we denote by $\sigma_{h-1}$ the event
\#\label{sigmadef}
\overbtau_{h-1}=\overtau_{h-1}, \quad \ba_{h-1}=a_{h-1}
\#
for any $h\in[H+1]$. The following lemma implies that the finite-memory Bellman operator $\mB^{\theta,\pi}_{h}$ is identical to the full-memory Bellman operator $\PP^{\theta,\pi}_{h}$ in expectation conditioning on $\sigma_{h-1}$, which allows us to use $\mB^{\theta,\pi}_{h}$ as a surrogate of $\PP^\theta_{h,a}$. 
\begin{lemma}[Operators Equivalence]\label{blemma}
For any $(h,\theta,\pi,\overline{\tau}_{h-1},a_{h-1})\in[H]\times\Theta\times\Pi\times\overline{\Gamma}_{h-1}\times\cA$ and $f\in L^\infty(\overline{\Gamma}_{h+1})$, we have
\$
\EE_\theta[(\mB^{\theta,\pi}_{h} f)(\overbtau_{h})-(\PP^{\theta,\pi}_{h} f)(\overbtau_h)\,|\, \sigma_{h-1}] = 0.
\$
\end{lemma}
\begin{proof}
See Section \ref{blemmap} for a detailed proof.
\end{proof}

To see the intuition behind Lemma \ref{blemma}, note that by the definition of $\cZ^\theta_h$ in Assumption \ref{asmp2}, we have
\$
\EE_\theta[\cZ^\theta_h(\tilde{s}_h, \bo_h) \,|\, \sigma_{h-1}]=p_{\theta}(\tilde{\bs}_h=\tilde{s}_h \,|\, \sigma_{h-1}), 
\$
for any $(\tilde{s}_h, \overtau_{h-1}, a_{h-1})\in\cS\times\overline{\Gamma}_{h-1}\times\cA$. See Section \ref{blemmap} for a derivation. In other words, $\cZ^\theta_h$ serves as the bridge function in causal inference \citep{shi2020selective}, which recovers the conditional distribution of $\tilde{\bs}_{h}$ from the conditional distribution of $\bo_h$. Then, by taking the same conditional expectation on both sides of \eqref{b-func-def}, we have
\#\label{0430203}
\EE_\theta[\cB^\theta_{h,a}(\bo_h, \tilde{o}_h, \tilde{o}_{h+1}) 
 \,|\, \sigma_{h-1} ]
&=
p_{\theta}(\tilde{\bo}_h=\tilde{o}_h, \tilde{\bo}_{h+1}=\tilde{o}_{h+1} 
\,|\, \sigma_{h-1}, \tilde{\ba}_h=a)\\
&=
p_{\theta}(\bo_h=\tilde{o}_h, \bo_{h+1}=\tilde{o}_{h+1} 
\,|\, \sigma_{h-1}, \ba_h=a), \notag
\#
which is connected to the integral kernel $p_{\theta}(\bo_{h+1}=o_{h+1}\,|\,\overbtau_h=\overtau_h, \ba_h=a)$ on the right-hand side of \eqref{914756} via the same conditional expectation
\#\label{04300204}
&\EE_\theta[p_{\theta}(\bo_{h+1}=o_{h+1}\,|\,\overbtau_h, \ba_h=a) \,|\,\sigma_{h-1}]\\
&\quad=\EE_\theta[p_{\theta}(\bo_{h+1}=o_{h+1}\,|\,\sigma_{h-1}, \bo_h, \ba_h=a) \,|\,\sigma_{h-1}] \notag\\
&\quad=p_{\theta}(\bo_{h+1}=o_{h+1}\,|\,\sigma_{h-1}, \ba_h=a) \notag\\
&\quad=
\int_\cO p_{\theta}(\bo_h=o_h, \bo_{h+1}=o_{h+1}\,|\,\sigma_{h-1}, \ba_h=a) \ud o_h. \notag
\#
Here the second equality in \eqref{0430203} follows from the fact that $(\tilde{\ba}_h, \tilde{\bo}_h, \tilde{a}_h)$ is an independent replicate of $(\ba_h, \bo_h, \bo_{h+1})$, that is, they follow the same distribution conditioning on $\sigma_{h-1}$. 

\noindent
\textbf{Backward Bellman Recursion:}
For any $(h,\theta,\pi)\in[H+1]\times\cA\times\Pi$, we define the value function $V^{\theta,\pi}_h\in L^\infty(\overline{\Gamma}_h)$ by
\#\label{vdef}
V^{\theta,\pi}_h(\overtau_h) =( \mB^{\theta,\pi}_{h}\cdots \mB^{\theta,\pi}_{H} R)(\overtau_h),
\quad
\text{for any $\overtau_h\in\overline{\Gamma}_h$,}
\#
which gives the backward Bellman recursion
\#\label{bellmanrecur}
V^{\theta,\pi}_h(\overtau_h) = (\mB^{\theta,\pi}_{h}V^{\theta,\pi}_{h+1})(\overtau_h),
\quad
\text{for any $\overtau_h\in\overline{\Gamma}_h$.}
\#
The following corollary is implied by Lemma \ref{blemma}, which relates the value function $V^{\theta,\pi}_h$ to the expected total reward. Note that $V^{\theta,\pi}_h$ does not correspond to the ``reward-to-go" in the usual value function definition in MDPs since it involves all rewards across the $H$ steps.
\begin{corollary}\label{coro1}
For any $(h,\theta,\pi)\in[H+1]\times\Theta\times\Pi$, we have
\#\label{1004938}
\EE_{\theta,\pi}\Bigl[V^{\theta,\pi}_h(\overbtau_h)- \sum_{i=1}^H \br_i  \,\Big|\, \sigma_{h-1} \Bigr] = 0,
\#
for any $(\overline{\tau}_{h-1},a_{h-1})\in\overline{\Gamma}_{h-1}\times\cA$. For $h=1$, we have $J(\theta,\pi)=\EE[V^{\theta,\pi}_1(\bo_1)]$ since $\overbtau_1=\bo_1$ and $\sigma_0=\emptyset$, which follows from the definition of $\overbtau_h$ in \eqref{def-full-history}.
\end{corollary}
\begin{proof}
See Section \ref{coro1p} for a detailed proof.
\end{proof}

Corollary \ref{coro1} allows us to evaluate a policy $\pi$ by estimating $\{\mB^{\theta^*,\pi}_h\}_{h=1}^H$ instead of $\{\PP^{\theta^*,\pi}_h\}_{h=1}^H$. Meanwhile, $\{V^{\theta,\pi}_h\}_{h=1}^H$ play a critical role in analyzing the sample complexity.

\subsection{Operator Estimation via Minimax Optimization}\label{algo2}

Although the finite-memory Bellman operator $\mB^{\theta,\pi}_h$ defined in \eqref{bdef} does not involve the observation distribution conditioning on the history, that is, the distribution of $\bo_{h+1}$ conditioning on $\overbtau_h$ and $\bo_h$, it remains unclear how to estimate $\mB^{\theta^*,\pi}_h$ in a sample-efficient manner. Note that, by the definition of $\mB^{\theta^*,\pi}_h$, it suffices to estimate functions $\{\cB^{\theta^*}_{h,a}\}_{a\in\cA}$. To this end, we define the operator $\mathbb{F}^\theta_{h,a}: L^\infty(\cO^3)\rightarrow L^\infty(\cO^3)$ for any $(h,a,\theta)\in\{2,\ldots, H\}\times\cA\times\Theta$ by
\#\label{fopdef}
(\mathbb{F}^\theta_{h,a} f)(o_{h-1}, o_h, o_{h+1})
=\int_{\cO^2}
f(o_{h-1}, \tilde{o}_h, \tilde{o}_{h+1})   \cdot \cB^{\theta}_{h, a}(o_h,\tilde{o}_h, \tilde{o}_{h+1})\ud \tilde{o}_h\ud \tilde{o}_{h+1}, 
\#
for any $f\in L^1(\cO^3)$ and $o_{h-1}$, $o_h\in\cO$. Note that $\mathbb{F}^\theta_{h,a}$ is a truncated version of $\mB^{\theta,\pi}_h$, which drops a few variables that are redundant for operator estimation. The following lemma motivates the estimator of $\cB^{\theta^*}_{h,a}$, which uses the definition of $\mathbb{F}^\theta_{h,a}$.
\begin{lemma}\label{esteq-lemma}
For any $(h,a,a',\pi)\in\{2,\ldots,H\}\times\cA^2\times\overline{\Pi}$, we have
\$
\EE_{X\sim \rho^\pi_{h,a,a'}}[ (\mathbb{F}^{\theta^*}_{h,a'} f -  f)(X)] = 0, 
\quad
\text{for any $f\in L^\infty(\cO^3)$}.
\$
Here, the distribution $\rho^\pi_{h,a,a'}\in\Delta(\cO^3)$ is defined by
\$
\rho^\pi_{h,a,a'}(o_{h-1}, o_h, o_{h+1})
=p_{\theta^*,\pi}\bigl(
\bo_{h-1}=o_{h-1}, \bo_h=o_h, \bo_{h+1}=o_{h+1} \,|\,\ba_{h-1}=a,\ba_h=a' \bigr),
\$
for any $o_{h-1}$, $o_h$, $o_{h+1}\in\cO$. Also, we have $\|\mathbb{F}^{\theta^*}_{h,a'}\|_{\infty\rightarrow\infty}\le \gamma$.
\end{lemma}
\begin{proof}
See Section \ref{esteq-lemmap} for a detailed proof.
\end{proof}

\noindent
\textbf{Minimax Optimization:}
For any $(h,a', \pi)\in\{2,\ldots,H\}\times\cA\times\overline{\Pi}$, Lemma \ref{esteq-lemma} allows us to estimate $\cB^{\theta^*}_{h,a'}$ based on a dataset $\{\cD_{h,a,a'}\}_{a\in\cA}$, where the data points in $\cD_{h,a,a'}$ are collected from the distribution ${\rho}^\pi_{h,a,a'}$. In other words, each episode involves three steps: (a) we execute the exploration policy $\pi$, which takes the actions $\ba_1,\ldots, \ba_{h-2}$, (b) we take the actions $\ba_{h-1}=a$ and $\ba_h=a'$ regardless of the observations, and (c) we add the observation tuple $(\bo_{h-1}, \bo_h, \bo_{h+1})$ to $\cD_{h,a,a'}$. Based on $\{\cD_{h,a,a'}\}_{(h, a,a')\in\{2,\ldots,H\}\times\cA^2}$, we estimate $\{\cB^{\theta^*}_{h,a'}\}_{(h,a')\in\{2,\ldots,H\}\times\cA}$ by solving the following minimax optimization problem,
\#\label{minimax}
\min_{\theta\in\Theta}
~ \max_{f\in L^\infty(\cO^3): \|f\|_\infty\le1}
~\max_{(h, a,a')\in\{2,\ldots,H\}\times\cA^2} 
~\EE_{X\sim\hat{\cD}_{h,a,a'}}
[  (\mathbb{S} \mathbb{F}^\theta_{h,a'} f -   \mathbb{S}f)(X) ].
\#
Here, $\hat{\cD}_{h,a,a'}$ is the empirical distribution induced by the dataset $\cD_{h,a,a'}$. Also, the projection operator $\mathbb{S}: L^1(\cO^3)\rightarrow L^1(\cO^3)$ satisfies that
\#\label{proj-eq}
\EE_{X\sim p}[(\mathbb{S}f)(X)] = \int_{\cO^3} f(x) \cdot p^\dagger(x)\ud x,
\#
for any $f\in L^\infty(\cO^3)$ and $p\in\Delta(\cO^3)$. Here, $p^\dagger\in L^1(\cO^3)$ is the projection of $p$ onto ${\rm linspan}(\{\phi_i\}_{i=1}^{d_{\rm o}})$. See the definition of $\mathbb{S}$ in the next paragraph. The minimax optimization problem in \eqref{minimax} is motivated by generative adversarial networks. To see the intuition behind \eqref{minimax}, note that $f$ serves as the discriminator and $F^{\theta}_{h,a'}$ serves as the generator. In detail, note that the function $\mathbb{F}^\theta_{h,a} f$ in \eqref{fopdef} is constant with respect to the variable $o_{h+1}$. Thus, Lemma \ref{esteq-lemma} implies that the true generator $F^{\theta^*}_{h,a'}$ recovers the distribution of $(\bo_{h-1}, \bo_h, \bo_{h+1})\sim d^{\pi}_{h,a,a'}$ (corresponding to the true parameter $\theta^*$) from the marginal distribution of $(\bo_{h-1}, \bo_h)$. In this case, the true distribution and the (fake) distribution recovered by the generator can not be distinguished by any discriminator in $L^\infty(\cO^3)$. When we train the generator and discriminator on a dataset, the discriminator class $L^\infty(\cO^3)$ has a too large capacity. Therefore, we employ the projection operator $\mathbb{S}$ to enforce the finite-dimensional linear structure of $d^{\pi}_{h,a,a'}$, which reduces the capacity of the discriminator class. Such a projection operator guarantees the generalization power of the solution to \eqref{minimax}.

\noindent
\textbf{Projection Operator via RKHS:}
In the following, we define the projection operator $\mathbb{S}$. To this end, we consider an RKHS $\cH$ induced by a kernel function $\cK:\cO^3\times\cO^3\rightarrow\RR$. We define the corresponding RKHS embedding $\mathbb{K}: L^1(\cO^3)\rightarrow\cH$ by
\#\label{rkhsembedding}
(\mathbb{K}p)(x) = \int_{\cO^3} \cK(x,y)p(y)\ud y,
\# 
for any $p\in L^1(\cO^3)$ and $x\in\cO^3$. Moreover, we define the matrix $G\in\RR^{d_{\rm o}\times d_{\rm o}}$ by
\#\label{gmatdef}
[G]_{i,j}=\la \mathbb{K}\phi_i, \mathbb{K}\phi_j \ra_\cH=\EE_{X\sim\phi_i, X'\sim\phi_j}[\cK(X,X')],
\quad
\text{for any $i,j\in[d_{\rm o}]$.}
\#
Recall that the distribution functions $\{\phi_i\}_{i=1}^{d_{\rm o}}$ are defined in Assumption \ref{asmp2}. The following assumption specifies the regularity condition on $\cK$ and $\{\phi_i\}_{i=1}^{d_{\rm o}}$.
\begin{assumption}\label{asmp3}
The kernel function $\cK$ is bounded  and continuous. In particular, we have $|\cK(x,y)|\le1$ for any $x,y\in\cO^3$. Also, we have $\alpha=\lambda_{\min}(G)>0$, where we denote by $\lambda_{\min}(\cdot)$ the minimum eigenvalue of a matrix and the matrix $G$ is defined in \eqref{gmatdef}.
\end{assumption}

Here, the continuity of $\cK$ is defined with respect to the topology space $\cO^3$. For example, $\cO^3$ is (embedded as) a subset of some Euclidean space and the continuity of $\cK$ is defined with respect to the corresponding Euclidean distance. The boundedness of $\cK$ is satisfied by many kernel functions, for example, the radial basis function (RBF) kernel \citep{smola1998learning}. For the positive definiteness of the matrix $G$, note that, for any $v=(v_1, \ldots, v_{d_{\rm o}})\in\RR^{d_{\rm o}}$, we have
\$
v^\top G v
=\sum_{i,j=1}^{d_{\rm o}} v_iv_j \cdot \la \mathbb{K}\phi_i, \mathbb{K}\phi_j \ra_\cH
=\Bigl\|  \sum_{i=1}^{d_{\rm o}} v_i \cdot \KK \phi_i \Bigr\|_{\cH}^2.
\$
Therefore, to make $G$ positive definite, it suffices to require $\KK\phi_1, \ldots, \KK\phi_{d_{\rm o}}$ to be linearly independent in $\cH$. With the the kernel function $\cK$ and matrix $G$ defined above, we can verify that \eqref{proj-eq} holds for the operator $\mathbb{S}$ defined by
\# \label{sopdef}
(\mathbb{S}f)(x)
=\sum_{i,j\in[d_{\rm o}]}
[G^{-1}]_{i,j} \cdot
\EE_{Y\sim\phi_i, Y'\sim\phi_j}\bigl[\cK(x, Y) \cdot f(Y') \bigr],
\#
for any $f\in L^\infty(\cO^3)$ and $x\in\cO^3$. Here, the distance in the projection from $p$ to $p^\dagger$ in \eqref{proj-eq} is defined by
\#\label{distance-def}
d(p_1, p_2)=\| \mathbb{K}p_1 -  \mathbb{K}p_2\|_{\cH}, 
\quad
\text{for any $p_1, p_2\in L^1(\cO^3)$.}
\#
See Section \ref{proj-conjugate-p} for a derivation.

\subsection{Online Exploration via Optimistic Planning}\label{algo3}
We present the {\it \underline{Op}timistic Explora\underline{t}ion via Adv\underline{e}rsarial I\underline{n}tegral Equa\underline{t}ion} ($\algoname$) algorithm, which incorporates operator estimation into optimistic planning to perform online exploration. In other words, we update the exploration policy in Section \ref{algo2} in an iterative manner. We initialize $\algoname$ with any policy $\pi_0\in\Pi$ and a dataset
\#\label{dataset}
\{\cD_{h,a,a'}\}_{(h,a,a')\in\{2,\ldots,H\}\times\cA^2}=\emptyset,
\# 
which are updated subsequently in the $K$ iterations. Each iteration consists of an exploration phase and a planning phase. In the following, we describe the $k$-th iteration for any $k\in[K]$.

\noindent
\textbf{Exploration Phase:} Given the exploration policy $\pi_{k-1}$, we run an episode of the POMDP for each tuple $(h,a,a')\in\{2,\ldots,H\}\times\cA^2$ following the data collecting scheme defined in Section \ref{algo2} to add an observation tuple $(o_{h-1}, o_{h}, o_{h+1})$ into the dataset $\cD_{h,a,a'}$. After the exploration phase of the $k$-th iteration, we have $k$ observation tuples in the dataset $\cD_{h,a,a'}$ for any $(h,a,a')$. Although the dataset is collected by the exploration policies $\pi_0,\ldots,\pi_{k-1}$ in the $k$ iterations, we can regard it as a dataset collected by the mixing policy
\#\label{mixpolicy}
\overline{\pi}_k={\rm mixing}\{\pi_0,\ldots,\pi_{k-1}\}.
\#
where each policy is sampled uniformly at random as defined in Section \ref{prel1}. 

\noindent
\textbf{Planning Phase:} We apply the operator estimation method defined in Section \ref{algo2} to the updated dataset in \eqref{dataset} and construct a confidence set of the model parameter $\theta$
\#\label{csdef}
\Theta_k=
\Bigl\{\theta\in\Theta: L(\theta) \le \beta\cdot k^{-1/2}  \Bigr\},
\#
for a constant $\beta>0$, where $L(\theta)$ is defined as
\#\label{risk}
L(\theta)=\max_{f\in L^\infty(\cO^3): \|f\|_\infty\le1}
~\max_{(h, a,a')\in\{2,\ldots,H\}\times\cA^2} 
~\EE_{X\sim\hat{\cD}_{h,a,a'}}
[  (\mathbb{S} \mathbb{F}^\theta_{h,a'} f -   \mathbb{S}f)(X) ].
\#
Given the confidence set defined in \eqref{csdef}, we update the exploration policy by
\#\label{pistar}
\pi_{k} = \argmax_{\pi\in\Pi} \max_{\theta\in\Theta_k} ~J(\theta,\pi),
\#
which is the optimal policy with respect to the optimistic value estimator over parameters $\theta$ in the confidence set $\Theta_k$. Recall that $J(\theta, \pi)$ is defined in \eqref{09061110}. Note that we can perform the computation of \eqref{pistar} via a planning oracle for POMDPs \citep{golowich2022planning}. In detail, we can reformulate \eqref{pistar} as
\$
\theta_k&=
\argmax_{\theta\in\Theta_k}
J\bigl(\theta, \hat{\pi}(\theta)\bigr), \quad
\pi_k=\hat{\pi}(\theta_k),
\$
where the planning oracle $\hat{\pi}(\cdot)$ outputs the optimal policy with respect to any parameter. The constraint $\theta\in\Theta_k$ can be further transformed as a part of the objective via the Lagrangian relaxation. Then, we can apply the stochastic gradient method to obtain $\theta_k$ in a computation-efficient manner. See Section \ref{computation-sec} for more details. At the $(k+1)$-th iteration, we execute the exploration policy $\pi_k$ to collect data, which serves as the next exploration phase. We present $\algoname$ in Algorithm \ref{algo}.


\begin{algorithm}
\caption{Optimistic Exploration via Adversarial Integral Equation ($\algoname$)}
\begin{algorithmic}[1] 
\STATE \textbf{Input:} number of iterations $K\in\NN$, confidence level $\beta>0$
\STATE \textbf{Initialization:} set $\pi_0$ as a deterministic policy
\STATE \textbf{Initialization:} update the dataset $\cD_{h,a,a'}\leftarrow \emptyset$ for $(h, a, a')\in\{2, \ldots, H\}\times\cA^2$
\STATE \textbf{For} $k=1$ to $K$ \textbf{do} 
\STATE \hspace{0.2in} \textbf{For} $(h, a, a')\in\{2, \ldots, H\}\times\cA^2$ \textbf{do} 
\STATE \hspace{0.40in} Start a new episode
\STATE \hspace{0.40in} Execute the policy $\pi_{k-1}$ to take the first $(h-2)$ actions
\STATE \hspace{0.40in} Receive the observation $o_{h-1}$
\STATE \hspace{0.40in} Take the action $a$ and receive the observation $o_{h}$
\STATE \hspace{0.40in} Take the action $a'$ and receive the observation $o_{h+1}$
\STATE \hspace{0.40in} End the current episode
\STATE \hspace{0.40in} Update the dataset $\cD_{h,a,a'}\leftarrow \cD_{h,a,a'} \cup\{(o_{h-1}, o_{h}, o_{h+1})\}$
\STATE \hspace{0.2in} Construct the confidence set $\Theta_k$ by \eqref{csdef}
\STATE \hspace{0.2in} Update the policy $\pi_{k}\leftarrow \argmax_{\pi\in\Pi} \max_{\theta\in\Theta_k} J(\theta, \pi)$
\STATE \textbf{Output:} policy set $\{\pi_1, \ldots, \pi_K\}$
\end{algorithmic}\label{algo}
\end{algorithm} 

\section{Theory}
In this section, we analyze $\algoname$ in Algorithm \ref{algo}. In Section \ref{sec-thm}, we prove that the policies generated by Algorithm \ref{algo} converge to the optimal policy with a polynomial sample complexity. In Section \ref{sketch}, we sketch the proof by three key lemmas.

\subsection{Sample Efficiency}\label{sec-thm}

The following theorem characterizes the sample complexity of $\algoname$ in Algorithm \ref{algo}.

\begin{theorem}\label{mainthm}
Under Assumptions \ref{asmp1}, \ref{asmp2}, and \ref{asmp3}, for any $\delta>0$, if we choose a confidence level $\beta$ in Algorithm \ref{algo} to such that
\#\label{beta-cond}
\beta\ge d^{3/2}_{\rm o}(\gamma+1)/\alpha\cdot\sqrt{8\log(2KHA^2/\delta)},
\#
then, with probability at least $1-\delta$, we have
\#\label{929536}
&\frac{1}{K}\sum_{k=1}^K \bigl(J(\theta^*,\pi^*) - J(\theta^*, \pi_k)\bigr) \le
\frac{4d_{\rm s}\gamma^2\beta H^2A^2\cdot\log K}{K^{1/2}}  + \frac{4d_{\rm s}\gamma H^2}{ K}. 
\#
Recall that $d_{\rm s}$ and $d_{\rm o}$ are defined in Assumption \ref{asmp1}, $\gamma$ is defined in Assumption \ref{asmp2}, and $\alpha$ is defined in Assumption \ref{asmp3}.
\end{theorem}

Note that the first term on the right-hand side of \eqref{929536} is the leading term for a sufficiently large number of iterations $K$. Recall that the state distribution dimension $d_{\rm s}$ and observation distribution dimension $d_{\rm o}$ are defined in Assumption \ref{asmp1}. Also, quantities $\gamma$ and $\alpha$ are defined in Assumptions \ref{asmp2} and \ref{asmp3}, respectively. By Theorem \ref{mainthm}, if we run $\algoname$ for $K$ iterations and sample a policy from $\{\pi_1, \pi_K\}$ uniformly at random, the expected suboptimality of such a policy converges to zero with high probability at the rate of $K^{-1/2}$ up to logarithmic factors. Meanwhile, such a rate depends on $H$, $\cA$, $d_{\rm s}$, $d_{\rm o}$, $\gamma$, and $1/\alpha$ polynomially. In other words, to obtain an $\varepsilon$-optimal policy for any suboptimality $\varepsilon>0$, it suffices to run 
\#\label{korder}
K={\rm poly}(H,A,d_{\rm s}, d_{\rm o},\gamma, 1/\alpha)\cdot\tilde{O}(1/\varepsilon^2)
\# 
iterations in $\algoname$ to collect the data set. Note that the total number of episodes in $K$ iterations is $(H-1)A^2K$. To our best knowledge, Theorem \ref{mainthm} is the first polynomial sample complexity upper bound for reinforcement learning in POMDPs that is independent of the number of states and observations. Moreover, the order of $\varepsilon$ is optimal even in the MDP setting \citep{ayoub2020model}, which is a special case of POMDPs. In contrast to the sample complexity results in MDPs, a key difference of Theorem \ref{mainthm} is that it involves the (upper bound of the) operator norm $\gamma$ of the bridge operator $\mZ^\theta_h$, which is the left inverse of the observation operator $\mO^\theta_h$. Recall that such a left inverse is defined with respect to the finite-dimensional subspace ${\rm linspan}(\{\phi_i\}_{i=1}^{d_{\rm o}})$ of $L^1(\cO^3)$ in Assumption \ref{asmp2}. The (upper bound of the) operator norm $\gamma$ is a measure of ill conditioning, which quantifies the fundamental difficulty, in terms of the information-theoretic limit, of reinforcement learning in the POMDP. In the degenerate case where $\mO^\theta_h$ is not invertible, Theorem \ref{mainthm} provides a trivial upper bound since we have $\gamma=\infty$. On the other hand, such a case contains examples that are fundamentally impossible to solve in a sample-efficient manner, which is implied by information theory \citep{jin2020sample}.
 

\subsection{Proof Sketch}\label{sketch}
In this section, we sketch the proof of Theorem \ref{mainthm}. In detail, we prove the theorem by three key lemmas. The following lemma provides a decomposition of the difference of expected total rewards in two POMDPs when the parameters are different and the policies are identical. Recall that the function $J$ is defined in \eqref{09061110}.

\begin{lemma}[Value Decomposition]\label{lemma-dec}
Under Assumptions \ref{asmp1} and \ref{asmp2}, we have
\$
&J(\theta,\pi) - J(\theta',\pi) =\sum_{h=1}^{H} \EE_{\theta', \pi}[ (\mB^{\theta,\pi}_{h}V^{\theta,\pi}_{h+1})(\overbtau_h) - (\mB^{\theta',\pi}_{h} V^{\theta,\pi}_{h+1})(\overbtau_{h}) ],
\$
for any $\theta,\theta'\in\Theta$ and $\pi\in\Pi$. Here, the function $V^{\theta,\pi}_{h+1}$ is defined in \eqref{vdef}.
\end{lemma}
\begin{proof}
See Section \ref{lemma-decp} for a detailed proof.
\end{proof}
For any $k\in[K]$, we denote by $\theta_k\in\Theta$ the model parameter that is selected in the planning phase of the $k$-th iteration $\algoname$ (Algorithm \ref{algo}), which is defined in \eqref{pistar}, that is,
\#\label{915459}
(\theta_k, \pi_k)=\argmax_{(\theta,\pi)\in\Theta_k\times\Pi}J(\theta,\pi).
\#
We define the state-dependent error $e^k_h:\cS\rightarrow\RR$ by
\#\label{edef}
&e^k_h(s_{h-1})= \bigl|  \EE_{\theta^*,\pi_k}[(\mB^{\theta_k,\pi_k}_{h}V^{\theta_k,\pi_k}_{h+1})(\overbtau_h) - (\mB^{\theta^*,\pi_k}_{h} V^{\theta_k,\pi_k}_{h+1})(\overbtau_{h})
\,|\, \bs_{h-1}=s_{h-1}]  \bigr|, 
\#
for any $(k,h,s_{h-1})\in[K]\times[H]\times\cS$. Conditioning on the event $\theta^*\in\Theta_k$, which is shown to occur with high probability in the following lemma, we have
\#\label{vdiff}
&J(\theta^*,\pi^*)-J(\theta^*,\pi_k)   \le J(\theta_k,\pi_k) - J(\theta^*,\pi_k)  \le \sum_{h=1}^{H}\EE_{\theta^*,\pi_k}[e^k_h(\bs_{h-1})],
\#
which follows from Lemma \ref{lemma-dec}. Also, the following lemma characterizes the right-hand side of \eqref{vdiff} when we replace the policy $\pi_k$ by the mixing policy $\overline{\pi}_k$ defined in \eqref{mixpolicy}. 
\begin{lemma}[Statistical Guarantee]\label{lemma-acc}
Under Assumptions \ref{asmp1}, \ref{asmp2}, and \ref{asmp3}, for any $\delta>0$, by choosing the confidence level $\beta$ in $\algoname$ (Algorithm \ref{algo}) such that it satisfies \eqref{beta-cond}, with probability at least $1-\delta$, we have
\begin{itemize}
\item $\theta^*\in\Theta_k$,
\item $\EE_{\theta^*, \overline{\pi}_k}[ e^k_h(\bs_{h-1}) ] \le 2HA^2\gamma^2\beta\cdot k^{-1/2}$,
\end{itemize}
for any $(k,h)\in[K]\times[H]$.
\end{lemma}
\begin{proof}
See Section \ref{lemma-accp} for a detailed proof.
\end{proof}

To characterize the right-hand side of \eqref{vdiff}, it remains to connect $\EE_{\theta^*,\pi_k}[e^k_h(\bs_{h-1})]$ with $\EE_{\theta^*, \overline{\pi}_k}[ e^k_h(\bs_{h-1}) ]$, which involve different state distributions. The connection is established in the following lemma.

\begin{lemma}[Telescope of Error]\label{lemma-tel}
Under Assumptions \ref{asmp1} and \ref{asmp2}, for any $h\in[H]$, we have
\$
&\sum_{k=1}^K \EE_{\theta,\pi_k}[e^k_{h}(\bs_{h-1})]  \le 
4\gamma d_{\rm s} H
+2d_{\rm s}\log K \cdot
\max_{k\in[K]} \bigl(k\cdot\EE_{\theta,\overline{\pi}_k}[e^k_{h}(\bs_{h-1})]\bigr).
\$
\end{lemma}
\begin{proof}
See Section \ref{lemma-telp} for a detailed proof.
\end{proof}
Combining \eqref{vdiff} with Lemmas \ref{lemma-acc} and \ref{lemma-tel}, we obtain Theorem \ref{mainthm}. See Section \ref{mainthmp} for a detailed proof.

%% file: appendix.tex

\section{Examples: Linear Kernel POMDPs}\label{linearpomdp}

In this section, we show examples of the candidate class of the state transition kernels and observation emission kernels, which satisfy our assumptions in the main paper. In particular, we consider the following definition of linear kernel POMDPs.

\begin{definition}[Linear Kernel POMDPs] \label{linearpomdp-def}
We say that 
\$
\cL=\{(\cS, \cA, \cO, H, \cT^\theta, \cE^\theta, \mu, r): {\theta\in\Theta}\}
\$ 
is a linear kernel POMDP set, if each state transition kernel $\cT^\theta$ and observation emission kernel $\cE^\theta$ in the set take the form,
\#\label{linear-kernel}
\cT^\theta_h(s'\,|\,s,a)&=u(s')^\top M^\theta_{h,a}v(s),\quad
\cE^\theta_h(o\,|\,s)=q(o)^\top g^\theta_h(s),
\#
for any $(h,\theta,s,s',a,o)\in[H]\times\Theta\times\cS^2\times\cA\times\cO$. Here, $u$, $v$, $q$, and $g^\theta_h$ are non-negative vector-valued functions with dimensions $d_u$, $d_v$, $d_q$, and $d_q$, respectively. The matrix $M^\theta_{h,a}\in\RR^{d_u\times d_v}$ has non-negative entries. Moreover, we have $\mu\in{\rm conh}(\{[u(\cdot)]_i\}_{i=1}^{d_u})$ and
\$
[u(\cdot)]_i\in\Delta(\cS), \quad  [q(\cdot)]_\ell\in\Delta(\cO),
\quad
\text{for any $(i,j)\in[d_u]\times[d_q]$.}
\$ 
\end{definition}

The following lemma shows that tabular POMDPs are linear kernel POMDPs.
\begin{lemma}\label{1}
For any finite state space $\cS$, finite observation space $\cO$, action space $\cA$, episode length $H$, initial distribution $\mu$, and reward function $r$, we can define a linear kernel POMDP set $\{(\cS, \cA, \cO, H, \cT^\theta, \cE^\theta, \mu, r): {\theta\in\Theta}\}$ following Definition \ref{linearpomdp-def}, which consists of all possible POMDPs with the aforementioned elements.
\end{lemma}

\begin{proof}
We define $\Theta$ as the set of all possible pair $(\tilde{\cT}, \tilde{\cE})$ such that $\tilde{\cT}$ is a state transition kernel and $\tilde{\cE}$ is an observation emission kernel with respect to the state space $\cS$, action space $\cA$, observation space $\cO$, and episode length $H$. We let $d_u=d_v=|\cS|$ and $d_q=|\cO|$. For any $\theta=(\tilde{\cT}, \tilde{\cE})\in\Theta$, we define
\$
[M^{\theta}_{h,a}]_{s',s}=\tilde{\cT}_{h}(s'\,|\,s,a),\quad
[q^\theta_h(\cdot)]_o=\tilde{\cE}_h(o\,|\, \cdot),
\$
for any $(h,s,s',a,o)\in[H]\times\cS^2\times\cA\times\cO$. Also, we define
\$
[u(\cdot)]_s=\ind\{s=\cdot\}, \quad
[v(\cdot)]_s=\ind\{s=\cdot\}, \quad
[q(\cdot)]_o=\ind\{o=\cdot\},
\$
for any $(s,o)\in\cS\times\cO$. Then, by noting that we have $\cT^\theta=\tilde{\cT}$ and $\cE^\theta=\tilde{\cE}$ following the definitions of $\cT^\theta$ and $\cE^\theta$ in Definition \ref{linearpomdp-def}, we conclude the proof of Lemma \ref{1}.
\end{proof}

\subsection{Verification of Assumption \ref{asmp1}}\label{vasmp1-sec}
Recall that in reinforcement learning for a POMDP, the state transition kernel and observation emission kernel are unknown elements of the POMDP. We say that the candidate class of the POMDP is a linear kernel POMDP set $\cL=\{(\cS, \cA, \cO, H, \cT^\theta, \cE^\theta, \mu, r): {\theta\in\Theta}\}$ when the candidate class of the state transition kernel and observation emission kernel is $\{(\cT^\theta, \cE^\theta): \theta\in\Theta\}$ and other elements of the POMDP are determined as $(\cS, \cA, \cO, H, \mu, r)$. The following lemma shows that any linear kernel POMDP set satisfies the linear function approximation assumption (Assumption \ref{asmp1}).

\begin{lemma}\label{asmp1-ver}
When the candidate class of the POMDP is a linear kernel POMDP set $\cL$ as defined in Definition \ref{linearpomdp-def}, we have that Assumption \ref{asmp1} holds with
\$
d_{\rm s}\le d_u(d_v+1)\quad\text{and}\quad d_{\rm o}\le d_q^3.
\$
Recall that $d_u$, $d_v$, and $d_q$ are the vector-valued function dimensions in the definition of $\cL$, and $d_{\rm s}, d_{\rm o}$ are the number of basis distribution functions in Assumption \ref{asmp1}.
\end{lemma}
\begin{proof}
We prove the lemma by constructing the basis distribution functions $\{\psi_i\}_{i=1}^{d_{\rm s}}$ and $\{\phi_i\}_{i=1}^{d_{\rm o}}$ satisfying Assumption \ref{asmp1}. 

\noindent
\textbf{State Distribution:} Note that to make $\cF_{\rm s}\in{\rm conh}(\{\psi_i\}_{i=1}^{d_{\rm s}})$, it suffices to let $\{[u]_i\}_{i=1}^{d_u}$ be a subset of $\{\psi_i\}_{i=1}^{d_{\rm s}}$. In the following, we construct the rest elements of $\{\psi_i\}_{i=1}^{d_{\rm s}}$ to make $\cF'_{\rm s}\in{\rm conh}(\{\psi_i\}_{i=1}^{d_{\rm s}})$. For any $(h,\theta,\pi, s_h, a_h, o_{h+1})\in [H]\times\Theta\times\cS\times\cA\times\cO$, we have
\#\label{9181205}
&p_{\theta,\pi}(\bs_h=s_h, \bo_{h+1}=o_{h+1} \,|\, \ba_h=a_h)\notag\\
&\quad=\int_\cS
\cE^{\theta}_{h+1}(o_{h+1}\,|\,s_{h+1})\cdot
\cT^{\theta}_h(s_{h+1}\,|\,s_h,a_h)\ud s_{h+1} 
\cdot p_{\theta,\pi}(\bs_h=s_h)\notag\\
&\quad=
\Bigl(\int_\cS
\cE^{\theta}_{h+1}(o_{h+1}\,|\,s_{h+1})\cdot
u(s_{h+1})^\top M^\theta_{h,a_h} \ud s_{h+1}\Bigr)  v_h(s_h)
\cdot p_{\theta,\pi}(\bs_h=s_h),
\#
where the second equality is by the form of $\cT^\theta_h$ in Definition \ref{linearpomdp-def}. Similarly, we have
\#\label{9181206}
p_{\theta,\pi}(\bs_h=s_h)&=\EE_{\theta,\pi}[ \cT^{\theta}_{h-1}
\bigl(s_{h}\,\big|\,\bs_{h-1},\ba_{h-1})] \notag\\
&=u(s_h)^\top \EE_{\theta,\pi}[ M^\theta_{h,\ba_{h-1}} v_{h-1}(\bs_{h-1})].
\#
For notational simplicity, we define the vectors
\#
\zeta_1&=\Bigl(\int_\cS
\cE^{\theta}_{h+1}(o_{h+1}\,|\,s_{h+1})\cdot
u(s_{h+1})^\top M^\theta_{h,a_h} \ud s_{h+1}\Bigr)^\top, \label{9181207}\\
\zeta_2&=\EE_{\theta,\pi}[ M^\theta_{h,\ba_{h-1}} v_{h-1}(\bs_{h-1})]. \label{91812071}
\#
Then, combining \eqref{9181205}-\eqref{91812071}, we can write
\#\label{9181208}
&
p_{\theta,\pi}(\bs_h=s_h \,|\, \bo_{h+1}=o_{h+1},\ba_h=a_h) \notag\\
&\quad= 
\frac{p_{\theta,\pi}(\bs_h=s_h, \bo_{h+1}=o_{h+1} \,|\, \ba_h=a_h) }
{p_{\theta,\pi}(\bo_{h+1}=o_{h+1}\,|\,\ba_h=a_h) }
=\frac{\zeta_1^\top v_h(s_h)u(s_h)^\top \zeta_2}{
p_{\theta,\pi}(\bo_{h+1}=o_{h+1}\,|\,\ba_h=a_h)}.
\#
Note that we can rewrite \eqref{9181208} in a linear form, 
\$
p_{\theta,\pi}(\bs_h=\cdot \,|\, \bo_{h+1}=o_{h+1},\ba_h=a_h)
=\Bigl\la  v(\cdot)u(\cdot)^\top, 
\frac{ \zeta_2 \zeta_1^\top }{
p_{\theta,\pi}(\bo_{h+1}=o_{h+1}\,|\,\ba_h=a_h)}
\Bigr\ra_{\rm tr},
\$
where $\la \cdot, \cdot\ra_{\rm tr}$ represents the trace inner product of matrices. Therefore, we know that any function in $\cF'_{\rm s}$ can be represented as a convex combination of the functions
\#\label{918800}
\{[u(\cdot)]_i\cdot[v(\cdot)]_j\}_{i\in[d_u], j\in[d_v]}.
\#
Then, by normalizing each function in \eqref{918800} as a probability distribution function, we obtain $d_vd_u$ distribution functions, whose convex combination contains all elements of $\cF_{\rm s}'$. Thus, by 
denoting the set of such distribution functions plus $\{u_i\}_{i=1}^{d_{\rm u}}$ by $\{\psi_i\}_{i=1}^{d_{\rm s}}$, we have
\$
d_{\rm s}=d_u(d_v+1), \quad \text{and} \quad
\cF_{\rm s}', \cF_{\rm s}' \subset {\rm conh}(\{\psi_i\}_{i=1}^{d_{\rm s}}).
\$

\noindent
\textbf{Observation Distribution:}
In the following, we construct the basis distribution functions $\{\phi_i\}_{i=1}^{d_{\rm o}}$ such that $\cF_{\rm o}\subset{\rm conh}(\{\phi_i\}_{i=1}^{d_{\rm o}})$. Note that, for any $(h,\pi, o_{h},o_{h+1},o_{h+2}, a_{h}, a_{h+1})\in[H-1]\times\Theta\times\cO^3\times\cA^2$, we have
\#\label{22514502}
&p_{\theta,\pi}(\bo_h=o_h,\bo_{h+1}=o_{h+1},\bo_{h+2}=o_{h+2}\,|\,\ba_h=a_h,\ba_{h+1}=a_{h+1})\\
&\quad=
\int_{\cS^3}
p_{\theta,\pi}(\bs_{h}=s_h)
\cdot\cE^\theta_{h}(o_{h}\,|\,s_{h},a_h)
\cdot\cT^\theta_h(s_{h+1}\,|\,s_h,a_h)
\cdot \cE^\theta_{h+1}(o_{h+1}\,|\,s_{h+1}) \notag \\
&\quad\qquad\qquad
\cdot \cT^\theta_{h+1}(s_{h+2}\,|\,s_{h+1},a_{h+1})
\cdot \cE^\theta_{h}(o_{h}\,|\,s_{h})\ud s_h\ud s_{h+1}\ud s_{h+2} \notag \\
&\quad=
\sum_{i,j,\ell=1}^{d} \omega_{i,j,\ell}\cdot
 q_i(o_h)\cdot q_j(o_{h+1})\cdot q_\ell(o_{h+2}) \notag
\#
where $\{\omega_{i,j,\ell}\}_{i,j,\ell\in[d_q]}$ are defined by
\$
\omega_{i,j,\ell}&=\int_{\cS^3}
p_{\theta,\pi}(\bs_{h}=s_h)\cdot \cT^\theta_h(s_{h+1}\,|\,s_h,a_h)
\cdot \cT^\theta_{h+1}(s_{h+2}\,|\,s_{h+1},a_{h+1})  \\
&\quad\qquad\qquad
\cdot [g^\theta_h(s_h)]_i
\cdot [g^\theta_h(s_{h+1})]_j
\cdot [g^\theta_h(s_{h+2})]_\ell \ud s_h\ud s_{h+1}\ud s_{h+2}. \notag
\$
following the definition of $\cE^\theta$ in Definition \ref{linearpomdp-def}. For any $i,j,\ell\in [d_q]$, we define the distribution function $\phi_{i,j,\ell}\in \Delta(\cO^3)$ by
\$
\phi_{i,j,\ell}(o_h, o_{h+1}, o_{h+2})=q_i(o_h)\cdot q_j(o_{h+1})\cdot q_\ell(o_{h+2})
\$
for any $o_h, o_{h+1}, o_{h+2}\in\cO$. Then, by \eqref{22514502}, we have $\cF_{\rm o} \subset {\rm conh}(\{ \phi_{i,j,\ell} \}_{i,j,\ell=1}^{d_q})$. Reorganizing the index, we can write $\{ \phi_{i} \}_{i=1}^{d_{\rm o}}=\{ \phi_{i,j,\ell} \}_{i,j,\ell=1}^{d_q}$ with $d_{\rm o}=d_q^3$.

Therefore, we conclude the proof of Lemma \ref{asmp1-ver}.
\end{proof}

\subsection{Verification of Assumption \ref{asmp2}}
For any $(h,i)\in[H]\times[d_{\rm s}]$, we define $\nu_{h,i}\in\Delta(\cO)$ by
\#\label{nudef}
\nu_{h,i}(o)=\int_\cS \cE^\theta_h( o_h \,|\,s_h) \cdot \psi_i(s_h) \ud s_h
\#
for any $o\in\cO$. Recall that $\{\psi_i\}_{i=1}^{d_{\rm s}}$ are the basis distribution functions in Assumption \ref{asmp1}, and we prove their existence when the candidate class of the POMDP is a linear kernel POMDP set in Section \ref{vasmp1-sec}. Let $\tilde{\cK}$ be a kernel function (different from the kernel function $\cK$ in Section \ref{algo2}) defined on $\cO\times\cO$. For any $h\in[H]$, we define the matrix $\Lambda_h\in\RR^{d_{\rm s}\times d_{\rm s}}$ by
\#\label{lambda-def}
[\Lambda_h]_{i,j}=\EE_{\bo\sim \nu_{h,i}, \bo'\sim \nu_{h, j}}[ \tilde{\cK}(\bo, \bo')  ],
\quad
\text{for any $i,j\in [d_{\rm s}]$.}
\#
Similar to Assumption \ref{asmp3}, the following assumption specifies the regularity condition on $\tilde{\cK}$ and $\{\nu_{h,i}\}_{h\in[H], i\in[d_{\rm o}]}$.
\begin{assumption}\label{asmp4}
The kernel function $\tilde{\cK}$ is bounded. In particular, we have $|\tilde{\cK}(x,y)|\le1$ for any $x,y\in\cO$. Also, we have $\Lambda_h\succ0$ for any $h\in[H]$. 
\end{assumption}

Note that for any $h\in[H]$, similar to the discussion under Assumption \ref{asmp3}, the positive definiteness of the matrix $\Lambda_h$ requires the RKHS embedding of $\nu_{h,1}, \ldots, \nu_{h,d_{\rm s}}$ to be linearly independent. Here, the RKHS and the corresponding RKHS embedding (operator) are defined with respect to the kernel function $\tilde{\cK}$. As a special case, when the state space $\cS$ and observation space $\cO$ are finite, we let 
\$
\tilde{\cK}(x, y)=\ind\{x=y\}, \quad \text{for any $x, y\in\cO$},
\$
and $\{\psi_i\}_{i=1}^{d_{\rm s}}=\{\ind\{s=\cdot\}\}_{s\in\cS}$ with $d_{\rm s}=|\cS|$. Note that for a finite state space $\cS$, the integral in \eqref{nudef} is defined with respect to the counting measure over $\cS$. Then, Assumption \ref{asmp4} is equivalent to requiring the vectors $\{ \cE_h(\cdot\,|\,s)\in\RR^{|\cO|}\}_{s\in\cS}$ to be linearly independent for any $h\in[H]$, which recovers the undercompleteness assumption in \cite{jin2020sample}.

The following lemma shows that any linear kernel POMDP set satisfies the invertible observation operators assumption (Assumption \ref{asmp2}), given the aforementioned linear independence condition.

\begin{lemma}\label{22520303}
Suppose the candidate class of the POMDP is a linear kernel POMDP set $\cL$ as defined in Definition \ref{linearpomdp-def} and Assumption \ref{asmp4} holds. Then, we have that Assumption \ref{asmp2} holds with
\#\label{225151241}
\cZ^\theta_h(s, o)= \sum_{i,j=1}^{d_{\rm s}} \psi_i(s) \cdot  [(\Lambda_h)^{-1}]_{i,j} \cdot 
\EE_{\bo \sim \nu_{h, j}}[\tilde{\cK}(\bo, o)] 
\#
for any $(h,\theta,s, o)\in[H]\times\Theta\times\cS\times\cO$ and 
\$
\gamma=d\cdot \max_{h\in[H]}\| (\Lambda_h)^{-1} \|_{1\rightarrow1}
=d\cdot\max_{(h,j)\in[H]\times[d_{\rm s}]} 
\sum_{i=1}^{d_{\rm s}} \bigl| [(\Lambda_h)^{-1}]_{i,j} \bigr|.
\$
Here, the matrix $\Lambda_h$ is defined in \eqref{lambda-def}.
\end{lemma}
\begin{proof}
By the definitions of the operators $\mZ^\theta_{h}$ and $\mO^\theta_{h}$, and function $\cZ^\theta_h$ in Assumption \ref{asmp2}, \eqref{emissionopdef}, and \eqref{225151241}, respectively, for any $f\in L^1(\cS)$, we have
\#\label{225151227}
(\mZ^\theta_{h}\mO^\theta_{h}f)(s)
&=
\sum_{i,j=1}^{d_{\rm s}} \int_{\cS\times\cO^2} \psi_i(s) \cdot  [(\Lambda_h)^{-1}]_{i,j} \cdot \nu_{h, j}(o') \cdot \tilde{\cK}(o', o) \cdot \cE^\theta_h(o\,|\,s') \cdot f(s') \ud s' \ud o \ud o'.
\#
When $f\in{\rm linspan}(\{\psi_i\}_{i=1}^{d_{\rm s}})$ with $f=\sum_{i=1}^{d_{\rm s}} \psi_i \cdot c_i$, we have
\#\label{225151234}
&\int_\cS \cE^\theta_h(o'\,|\,s') \cdot f(s') \ud s' =
\int_\cS  \sum_{\ell=1}^{d_{\rm s}} 
\cE^\theta_h(o'\,|\,s') \cdot \psi_\ell(s') \cdot c_\ell
=\sum_{\ell=1}^{d_{\rm s}}  \nu_{h,\ell}(o') \cdot c_\ell,
\#
where the last equality is by the definition of $\nu_{h,\ell}$ in \eqref{nudef}. Plugging \eqref{225151234} into the right-hand side of \eqref{225151227}, we obtain
\$
(\mZ^\theta_{h}\mO^\theta_{h}f)(s)
&=
\sum_{i,j=1}^{d_{\rm s}} \int_{\cO^2} \psi_i(s) \cdot  [(\Lambda_h)^{-1}]_{i,j} \cdot \nu_{h, j}(o') \cdot \tilde{\cK}(o', o) \cdot \sum_{\ell=1}^{d_{\rm s}}  \nu_{h,\ell}(o') \cdot c_\ell \ud o \ud o'  \\
&=
\sum_{i,j,\ell=1}^{d_{\rm s}}  \psi_i(s) \cdot  [(\Lambda_h)^{-1}]_{i,j} \cdot 
 \Bigl(\int_{\cO^2}
\nu_{h, j}(o') \cdot \tilde{\cK}(o', o) \cdot \nu_{h,\ell}(o')  \ud o \ud o' \Bigr) \cdot c_\ell  \\
&=
\sum_{i,j,\ell=1}^{d_{\rm s}}  \psi_i(s) \cdot  [(\Lambda_h)^{-1}]_{i,j} \cdot 
 [\Lambda_h]_{j, \ell} \cdot c_\ell .
\$
Here, the last equality uses the definition of the matrix $\Lambda_h$ in \eqref{lambda-def}. By the definition of the inverse matrix, we have
\$
\sum_{j=1}^{d_{\rm s}}[(\Lambda_h)^{-1}]_{i,j} \cdot 
 [\Lambda_h]_{j, \ell} = \ind\{i=\ell\},
\$
which implies
\$
(\mZ^\theta_{h}\mO^\theta_{h}f)(s)=
\sum_{i=1}^{d_{\rm s}}  \psi_i(s) \cdot c_i = f(s),\quad
\text{for any $s\in\cS$.}
\$

In the following, we characterize the operator norm $\|\cdot\|_{1\rightarrow1}$ of $\mZ^\theta_h$. For any $(h,\theta, o)\in[H]\times\Theta\times\cO$, we have
\#\label{22515123}
\int_{\cS} |\cZ^\theta_h(s, o)| \ud s
&=\int_\cS \Bigl| \sum_{i,j=1}^{d_{\rm s}} \psi_i(s) \cdot  [(\Lambda_h)^{-1}]_{i,j} 
\cdot  \int_\cO \nu_{h, j}(o') \cdot \tilde{\cK}(o', o) \ud o' \Bigr| \ud s \\
&\le\int_\cS
\sum_{i=1}^{d_{\rm s}} \psi_i(s) \cdot 
 \Bigl| \sum_{j=1}^{d_{\rm s}}  [(\Lambda_h)^{-1}]_{i,j} 
\cdot  \int_\cO \nu_{h, j}(o') \cdot \tilde{\cK}(o', o) \ud o' \Bigr| \ud s  \notag\\
&=\sum_{i=1}^{d_{\rm s}} 
 \Bigl| \sum_{j=1}^{d_{\rm s}}  [(\Lambda_h)^{-1}]_{i,j} 
\cdot  \int_\cO \nu_{h, j}(o') \cdot \tilde{\cK}(o', o) \ud o' \Bigr|, \notag
\#
where the last equality is by the fact that $\psi_i$ is a distribution function over $\cS$ for any $i\in[d_{\rm s}]$. The right-hand side of \eqref{22515123} is upper bounded by
\#\label{22515229}
&\sum_{i=1}^{d_{\rm s}} 
 \Bigl| \sum_{j=1}^{d_{\rm s}}  [(\Lambda_h)^{-1}]_{i,j} 
\cdot  \int_\cO \nu_{h, j}(o') \cdot \tilde{\cK}(o', o) \ud o' \Bigr| \\
&\quad \le \sum_{i=1}^{d_{\rm s}} 
\sum_{j=1}^{d_{\rm s}}  \bigl|  [(\Lambda_h)^{-1}]_{i,j} \bigr| \cdot
 \Bigl|  \int_\cO \nu_{h, j}(o') \cdot \tilde{\cK}(o', o) \ud o' \Bigr| \notag \\
&\quad \le
\Bigl( \max_{j\in[d_{\rm s}]} \sum_{i=1}^{d_{\rm s}}  \bigl| [(\Lambda_h)^{-1}]_{i,j}  \bigr| \Bigr)
\cdot  \sum_{j=1}^{d_{\rm s}} 
 \Bigl|  \int_\cO \nu_{h, j}(o') \cdot \tilde{\cK}(o', o) \ud o' \Bigr|. \notag
\#
For any $(h,j)\in[H]\times[d_{\rm s}]$, because the kernel function $\tilde{\cK}$ is uniformly bounded by $1$, following Assumption \ref{asmp4}, and $\nu_{h, j}$ is a distribution function over $\cO$, we have
\#\label{22520244}
\Bigl| \int_\cO \nu_{h, j}(o') \cdot \tilde{\cK}(o', o) \ud o' \Bigr| \le 1.
\#
Combining \eqref{22515123}-\eqref{22520244}, we have
\#\label{22520301}
\int_{\cS} |\cZ^\theta_h(s, o)| \ud s 
\le \Bigl( \max_{j\in[d_{\rm s}]} \sum_{i=1}^{d_{\rm s}}  \bigl| [(\Lambda_h)^{-1}]_{i,j}  \bigr| \Bigr) \cdot d,
\#
for any $(h,\theta, o)\in[H]\times\Theta\times\cO$. Note that for any $f\in L^1(\cO)$, we have
\$
\| \mZ^\theta_h f \|_1
&=\int_\cS \Bigl| \int_\cO \cZ^\theta_h(s,o)\cdot f(o)\ud o \Bigr| \ud s \\
&\le
\int_{\cS\times\cO} |\cZ^\theta_h(s,o)|\cdot |f(o)|\ud o \ud s  \\
&=
\int_{\cO} \Bigl(\int_\cS |\cZ^\theta_h(s,o)| \ud s \Bigr) \cdot  |f(o)|\ud o, 
\$
combining which with \eqref{22520301}, we obtain
\$
\| \mZ^\theta_h f \|_1
\le \Bigl( \max_{j\in[d_{\rm s}]} \sum_{i=1}^{d_{\rm s}}  \bigl| [(\Lambda_h)^{-1}]_{i,j}  \bigr| \Bigr) \cdot d \cdot \|f\|_1.
\$
Therefore, we conclude the proof of Lemma \ref{22520303}.
\end{proof}

\section{Minimax Optimization in $\algoname$} \label{computation-sec}

In this section, we discuss the details on how to implement the computation of $\algoname$ in practice. Recall that in the planning phase (introduced in Section \ref{algo2}) of each iteration of $\algoname$, we only consider the parameter $\theta$ such that 
\$
L(\theta)=\max_{f\in L^\infty(\cO^3): \|f\|_\infty\le1}
~\max_{(h, a,a')\in\{2,\ldots,H\}\times\cA^2} 
~\EE_{X\sim\hat{\cD}_{h,a,a'}}
[  (\mathbb{S} \mathbb{F}^\theta_{h,a'} f -   \mathbb{S}f)(X) ]
\$
is sufficiently small. Here, $\mathbb{S}$ and $\mathbb{F}^\theta_{h,a'}$ are operators defined in \eqref{sopdef} and \eqref{fopdef}, respectively. Also, $\hat{\cD}_{h,a,a'}$ is the empirical distribution induced by $\cD_{h,a,a'}$, which consists of $k$ observation tuples with $k$ being the iteration index. For ease of presentation, we assume that we have access to the following planning oracle.
\begin{oracle}
We denote by $\hat{\pi}$ a planning oracle for any given POMDP. In other word, the mapping $\hat{\pi}:\Theta\rightarrow\Pi$ satisfies $\hat{\pi}(\theta)\in\argmax_{\pi\in\Pi} J(\theta, \pi)$ for any $\theta\in\Theta$.
\end{oracle}
With the planning oracle defined above, we select the parameter $\theta_k$ by solving the following constrained optimization problem,
\#\label{2205241239}
\min_{\theta\in\Theta} J\bigl(\theta, \hat{\pi}(\theta) \bigr)
\quad
\text{s.t.}\quad L(\theta) \le \beta \cdot k^{-1/2}.
\#
Then, we select the policy $\pi_k=\hat{\pi}(\theta_k)$.  

\subsection{Lagrangian Relaxation}
In the sequel, we handle the constraint in \eqref{2205241239} via the Lagrangian relaxation. In detail, solving \eqref{2205241239} is equivalent to solving the minimax optimization problem,
\#\label{220524152}
\min_{\theta\in\Theta} \max_{\lambda\ge0} ~
-J\bigl(\theta, \hat{\pi}(\theta)\bigr) 
+ \sum_{(h,a,a')\in\cI} 
\lambda_{h,a,a'} \cdot 
\bigl( \tilde{L}_{h,a,a'}(\theta) -  \beta\cdot k^{-1/2} \bigr),
\#
where $\lambda=(\lambda_{h,a,a'})_{(h,a,a')\in\cJ}\in\RR^{(H-1)A^2}$ and
$\tilde{L}_{h,a,a'}(\theta)$ is defined by
\#\label{220524135}
\tilde{L}_{h,a,a'}(\theta)
=\max_{f\in L^\infty(\cO^3): \|f\|_\infty\le1}
\EE_{X\sim\hat{\cD}_{h,a,a'}}
[  (\mathbb{S} \mathbb{F}^\theta_{h,a'} f
 -   \mathbb{S}f)(X) ] .
\#
Here, for notational simplicity, we denote by $\cI$ the set $\cI=\{2,\ldots, H\}\times\cA^2$. Note that for each $(h,a,a')\in\cI$, we need to search a function within a ball in $L^\infty(\cO^3)$. To this end, we propose to search functions within a large function approximator class, for example, a sufficiently large neural network. In detail, we denote by $f^w_{h,a,a'}$ the parametrization of the function approximator, where $w$ is the parameter with a candidate set $\cW$. For example, we can build a neural network whose input space is $\cO^3$ and output space is $\RR^{(H-1)A^2}$. Then, $w$ represents the weights of all layers and 
\$
\bigl(f^w_{h,a,a'}(x)\bigr)_{(h,a,a')\in\cI} \in \RR^{(H-1)A^2}
\$
is the output of the neural network corresponding any input $x\in\cO^3$. Moreover,
by properly choosing the activation function of the output layer, we are able to make $\|f^w_{h,a,a'}\|_\infty\le1$ for any $w\in\cW$ and $(h,a,a')\in\cI$. Then, we approximately compute $\tilde{L}_{h,a,a'}(\theta)$ in \eqref{220524135} by computing $\max_{w\in\cW}\hat{L}^w_{h,a,a'}(\theta)$, where
\#\label{220524138}
\hat{L}^w_{h,a,a'}(\theta)=
\EE_{X\sim\hat{\cD}_{h,a,a'}}
[  (\mathbb{S} \mathbb{F}^\theta_{h,a'} f^w_{h,a,a'}
 -   \mathbb{S}f^w_{h,a,a'})(X) ],
\#
for any $w\in\cW$. Combining \eqref{220524152}-\eqref{220524138}, we approximately solve the constrained optimization problem in \eqref{2205241239} by solving
\#\label{220524154}
\min_{\theta\in\Theta} \max_{\lambda\ge0, w\in\cW} ~
-J\bigl(\theta, \hat{\pi}(\theta)\bigr) 
+ \sum_{(h,a,a')\in\cI} 
\lambda_{h,a,a'} \cdot 
\bigl( \hat{L}^w_{h,a,a'}(\theta) -  \beta\cdot k^{-1/2} \bigr).
\#

\subsection{Stochastic Gradient Method}

We denote by 
\$
\cL(\theta, \lambda, w)=
-J\bigl(\theta, \hat{\pi}(\theta)\bigr) 
+ \sum_{(h,a,a')\in\cI} 
\lambda_{h,a,a'} \cdot 
\bigl( \hat{L}^w_{h,a,a'}(\theta) -  \beta\cdot k^{-1/2} \bigr)
\$ 
the minimax objective in \eqref{220524154}. In the sequel, we consider the stochastic gradient method for solving the minimax optimization problem in \eqref{220524154}. In detail, suppose that we have unbiased stochastic gradient estimators $g_\theta$, $g_\lambda$, and $g_w$ such that
\$
\EE[g_\theta(\theta, \lambda, w)]&= \nabla_\theta \cL(\theta, \lambda, w), \\
\EE[g_\lambda(\theta, \lambda, w)]&= \nabla_\lambda \cL(\theta, \lambda, w), \\
\EE[g_w(\theta, \lambda, w)]&= \nabla_w \cL(\theta, \lambda, w),
\$
for any $(\theta, \lambda, w)\in\Theta\times\RR^{(H-1)A^2}\times\cW$. Also, computing $g_\theta$, $g_\lambda$, and $g_w$ does not require access to the full data set $\cD=\{\cD_{h,a,a'}\}_{h\in\{2,\ldots,H\}\times\cA^2}$ and thus has a low computation cost. In each iteration, starting from any $(\theta,\lambda,w)$ within the candidate set, we first update $\lambda$ and $w$ by running
\$
\lambda \leftarrow \lambda + \eta_\lambda \cdot g_\lambda(\theta, \lambda, w), \quad
w \leftarrow w + \eta_w \cdot g_w (\theta, \lambda, w)
\$
for $N_{\rm dual}$ steps. Then, we update $\theta$ by running
\$
\theta \leftarrow \theta - \eta_\theta \cdot g_\theta(\theta,\lambda,w).
\$
for $N_{\rm primal}$ steps. Here, $\eta_\theta$, $\eta_\lambda$, $\eta_w$ are constant stepsizes. Note that after each update, if the updated parameter is not in the candidate set, we need to run an extra projection step, which replaces the updated parameter by its closest neighbor within the candidate set. Because we need to call the planning oracle $\hat{\pi}$ after updating $\theta$, which has a relatively high computation cost, it is better to set $N_{\rm primal}=1$ and set $N_{\rm dual}$ as a large number.

In the sequel, we construct unbiased gradient estimators for the objective $\cL(\theta, \lambda, w)$. 

\noindent
\textbf{Construction of $g_\lambda$:}
To construct $g_\lambda(\theta,\lambda,x)$, note that we have
\$
\partial_{\lambda_{h,a,a'}} \cL(\theta,\lambda,w)=
\hat{L}^w_{h,a,a'}(\theta) -  \beta\cdot k^{-1/2}
\$
for any $(h,a,a')\in\{2,\ldots,H\}\times\cA^2$. Let $\cB$ be a batch of data sampled from $\cD_{h,a,a'}$ uniformly at random. Then, by computing the batch average, we have
\#\label{2205261034}
\EE\Bigl[\frac{1}{|\cB|}\sum_{x\in \cB }
  (\mathbb{S} \mathbb{F}^\theta_{h,a'} f^w_{h,a,a'}
 -   \mathbb{S}f^w_{h,a,a'})(x) -  \beta\cdot k^{-1/2} \Bigr]
 =\partial_{\lambda_{h,a,a'}} \cL(\theta,\lambda,w).
\#
For any $(i,j,x)\in[d_{\rm o}]\times[d_{\rm o}]\times\cO^3$, let $Y_{i,j,x}$ and $Y'_{i,j,x}$ be independent random variables in $\cO^3$ sampled from the distributions $\phi_i$ and $\phi_j$, respectively. Then, by the definition of the operator $\mathbb{S}$ in \eqref{sopdef}, we have
\#\label{2205261118}
(\mathbb{S}f^w_{h,a,a'})(x)
=\sum_{i,j\in[d_{\rm o}]}
\EE\bigl[
[G^{-1}]_{i,j} \cdot
\cK(x, Y_{i,j,x}) \cdot f^w_{h,a,a'}(Y'_{i,j,x}) \bigr]. 
\#
Similarly, applying the definition of the operator $\mathbb{F}^\theta_{h,a'}$ in \eqref{fopdef}, we have
\#\label{2205261019}
&(\mathbb{S} \mathbb{F}^\theta_{h,a'} f^w_{h,a,a'})(x)\\
&\quad=\sum_{i,j\in[d_{\rm o}]}
\EE\Bigl[
[G^{-1}]_{i,j} \cdot
\cK(x, Y_{i,j,x}) \cdot 
\int_{\cO^2}
f^w_{h,a,a'}(\bo_{h-1}, \tilde{o}_h, \tilde{o}_{h+1})   \cdot \cB^{\theta}_{h, a}(\bo_h,\tilde{o}_h, \tilde{o}_{h+1})\ud \tilde{o}_h\ud \tilde{o}_{h+1}
 \Bigr], \notag
\#
where we denote $Y'_{i,j,x}=(\bo_{h-1}, \bo_{h}, \bo_{h+1})$. Moreover, let $\phi_{\rm ip}$ be a distribution supported on $\cO^2$ and $\tilde{Y}_{i,j,x}=(\tilde{\bo}_h, \tilde{\bo}_{h+1})$ be a random variable in $\cO^2$ sampled from $\phi_{\rm ip}$. Then, following the idea of importance sampling, we can rewrite \eqref{2205261019} as
\#\label{2205261020}
&(\mathbb{S} \mathbb{F}^\theta_{h,a'} f^w_{h,a,a'})(x)\\
&\quad=\sum_{i,j\in[d_{\rm o}]}
\EE\Bigl[
[G^{-1}]_{i,j} \cdot
\cK(x, Y_{i,j,x}) \cdot 
\frac{
f^w_{h,a,a'}(\bo_{h-1}, \tilde{\bo}_h, \tilde{\bo}_{h+1})   \cdot \cB^{\theta}_{h, a}(\bo_h,\tilde{\bo}_h, \tilde{\bo}_{h+1})}{\phi_{\rm ip}(\tilde{\bo}_h, \tilde{\bo}_{h+1})}
 \Bigr].\notag
\#
Combining \eqref{2205261034}-\eqref{2205261020}, we construct $g_\lambda(\theta,\lambda,x)$ by
\$
[g_\lambda(\theta,\lambda,x)]_{h,a,a'}=
\frac{1}{|\cB|}\sum_{x\in \cB } \bigl(
  \hat{\mathbb{S} \mathbb{F}^\theta_{h,a'} f^w_{h,a,a'}}(x)
 -   \hat{\mathbb{S}f^w_{h,a,a'}}(x) \bigr) -  \beta\cdot k^{-1/2},
\$
for any $(h,a,a')\in\{2,\ldots,H\}\times\cA^2$, where
\$
\hat{\mathbb{S} \mathbb{F}^\theta_{h,a'} f^w_{h,a,a'}}(x)
&=\sum_{i,j\in[d_{\rm o}]}
[G^{-1}]_{i,j} \cdot
\cK(x, Y_{i,j,x}) \cdot 
\frac{
f^w_{h,a,a'}(\bo_{h-1}, \tilde{\bo}_h, \tilde{\bo}_{h+1})   \cdot \cB^{\theta}_{h, a}(\bo_h,\tilde{\bo}_h, \tilde{\bo}_{h+1})}
{\phi_{\rm ip}(\tilde{\bo}_h, \tilde{\bo}_{h+1})},\\
\hat{\mathbb{S}f^w_{h,a,a'}}(x)
&=
\sum_{i,j\in[d_{\rm o}]}
[G^{-1}]_{i,j} \cdot
\cK(x, Y_{i,j,x}) \cdot f^w_{h,a,a'}(Y'_{i,j,x}).
\$

\noindent
\textbf{Construction of $g_w$:} To construct $g_w(\theta,\lambda,x)$, note that we have
\$
\nabla_w \cL(\theta,\lambda,w)&=
\sum_{(h,a,a')\in\cI} 
\lambda_{h,a,a'} \cdot \nabla_w\hat{L}^w_{h,a,a'}(\theta) \\
&=
\sum_{(h,a,a')\in\cI} 
\lambda_{h,a,a'} \cdot 
\EE_{X\sim\hat{\cD}_{h,a,a'}}
[  (\nabla_w\mathbb{S} \mathbb{F}^\theta_{h,a'} f^w_{h,a,a'}
 -   \nabla_w\mathbb{S}f^w_{h,a,a'})(X) ].
\$
Thus, following the similar argument as in \eqref{2205261118}-\eqref{2205261020}, we construct $g_w(\theta,\lambda,w)$ as
\$
g_w(\theta,\lambda,w)=
\sum_{(h,a,a')\in\cI} 
\lambda_{h,a,a'} \cdot 
\frac{1}{|\cB|} \sum_{x\in\cB}
\bigl( \hat{\nabla_w\mathbb{S} \mathbb{F}^\theta_{h,a'} f^w_{h,a,a'}}(x)
 -   \hat{\nabla_w\mathbb{S}f^w_{h,a,a'}}(x) \bigr)
\$
where
\$
\hat{\nabla_w\mathbb{S} \mathbb{F}^\theta_{h,a'} f^w_{h,a,a'}}(x)
&=\sum_{i,j\in[d_{\rm o}]}
[G^{-1}]_{i,j} \cdot
\cK(x, Y_{i,j,x}) \cdot 
\frac{
\nabla_wf^w_{h,a,a'}(\bo_{h-1}, \tilde{\bo}_h, \tilde{\bo}_{h+1})   \cdot \cB^{\theta}_{h, a}(\bo_h,\tilde{\bo}_h, \tilde{\bo}_{h+1})}
{\phi_{\rm ip}(\tilde{\bo}_h, \tilde{\bo}_{h+1})},\\
\hat{\nabla_w\mathbb{S}f^w_{h,a,a'}}(x)
&=
\sum_{i,j\in[d_{\rm o}]}
[G^{-1}]_{i,j} \cdot
\cK(x, Y_{i,j,x}) \cdot \nabla_w f^w_{h,a,a'}(Y'_{i,j,x}).
\$
Here, $Y_{i,j,x}$, $\bo_{h-1}$, $\bo_h$, $\tilde{\bo}_h$, and $\tilde{\bo}_{h+1}$ are random variables defined the same as in the construction of $g_\lambda$.

\noindent
\textbf{Construction of $g_\theta$:} 
To construct $g_\theta(\theta,\lambda,x)$, note that we have
\#\label{2205261159}
\nabla_w \cL(\theta,\lambda,w)&=
-\nabla_\theta J\bigl( \theta, \hat{\pi}(\theta) \bigr) +
\sum_{(h,a,a')\in\cI} 
\lambda_{h,a,a'} \cdot \nabla_\theta \hat{L}^w_{h,a,a'}(\theta) \notag\\
&=
-\nabla_\theta J\bigl( \theta, \hat{\pi}(\theta) \bigr) +
\sum_{(h,a,a')\in\cI} 
\lambda_{h,a,a'} \cdot  
\EE_{X\sim\hat{\cD}_{h,a,a'}}
[  (\nabla_\theta\mathbb{S} \mathbb{F}^\theta_{h,a'} f^w_{h,a,a'})
(X) ].
\#
Following the similar argument as in \eqref{2205261019}-\eqref{2205261020}, we have the following unbiased estimator of the expectation in the second term on the right-hand side of \eqref{2205261159},
\#\label{220527555}
&\frac{1}{|\cB|}\sum_{x\in\cB}\hat{\nabla_\theta\mathbb{S} \mathbb{F}^\theta_{h,a'} f^w_{h,a,a'}}(x) \\
&\quad=\frac{1}{|\cB|}\sum_{x\in\cB}\sum_{i,j\in[d_{\rm o}]}
[G^{-1}]_{i,j} \cdot
\cK(x, Y_{i,j,x}) \cdot 
\frac{
f^w_{h,a,a'}(\bo_{h-1}, \tilde{\bo}_h, \tilde{\bo}_{h+1})   \cdot 
\nabla_\theta\cB^{\theta}_{h, a}(\bo_h,\tilde{\bo}_h, \tilde{\bo}_{h+1})}
{\phi_{\rm ip}(\tilde{\bo}_h, \tilde{\bo}_{h+1})}. \notag
\#
It remains to construct an unbiased estimator of the first term on the right-hand side of \eqref{2205261159}. Note that $\hat{\pi}(\theta)$ is the maximizer of $J(\theta,\cdot)$. Thus, by the envelop theorem, it suffices to estimate
\#\label{220527556}
\bigl(\nabla_\theta J(\theta, \pi)\bigr)\big|_{\pi=\hat{\pi}(\theta)}.
\#
Note that for any $(\theta,\pi)\in\Theta\times\Pi$, we have
\$
J(\theta,\pi)
&=\int_{\cS^H\times\cO^H} R\bigl(o_1,\pi(\tau_1),\ldots,o_H,\pi(\tau_H) \bigr) 
\cdot \mu(s_1)\cdot \cE_1(o_1\,|\,s_1) \\
&\quad\qquad\cdot
\Bigl(\prod_{h=1}^{H-1}
\cT^\theta_h\bigl(s_{h+1}\,|\,s_h, \pi(\tau_h) \bigr)
\cdot
\cE^\theta_{h+1}(o_{h+1}\,|\,s_{h+1}) \Bigr)
\ud o_1\cdots\ud o_H
\ud s_1 \cdots \ud s_H.
\$
Then, using the chain rule of the derivative, we have
\#\label{220527513}
\nabla_\theta J(\theta,\pi)
&=\sum_{i=1}^{H-1}\int_{\cS^H\times\cO^H} R\bigl(o_1,\pi(\tau_1),\ldots,o_H,\pi(\tau_H) \bigr) 
\cdot \mu(s_1)\cdot \cE_1(o_1\,|\,s_1) \\
&\quad\qquad\cdot 
\Bigl( \nabla_\theta \cT^\theta_h\bigl(s_{h+1}\,|\,s_h, \pi(\tau_h) \bigr)
\cdot
\cE^\theta_{h+1}(o_{h+1}\,|\,s_{h+1}) \notag\\
&\quad\qquad\qquad 
+ \cT^\theta_h\bigl(s_{h+1}\,|\,s_h, \pi(\tau_h) \bigr)
\cdot
\nabla_\theta \cE^\theta_{h+1}(o_{h+1}\,|\,s_{h+1}) \Bigr) \notag \\
&\quad\qquad\cdot
\Bigl(\prod_{h=1, h\neq i}^{H-1}
\cT^\theta_h\bigl(s_{h+1}\,|\,s_h, \pi(\tau_h) \bigr)
\cdot
\cE^\theta_{h+1}(o_{h+1}\,|\,s_{h+1}) \Bigr)
\ud o_1\cdots\ud o_H
\ud s_1 \cdots \ud s_H. \notag
\#
Using the relation $\nabla \ln f =\nabla f / f$, we can rewrite the right-hand side of \eqref{220527513} to obtain
\$
\nabla_\theta J(\theta,\pi)
&=\sum_{i=1}^{H-1}\int_{\cS^H\times\cO^H} R\bigl(o_1,\pi(\tau_1),\ldots,o_H,\pi(\tau_H) \bigr) 
\cdot \mu(s_1)\cdot \cE_1(o_1\,|\,s_1) \\
&\quad\qquad\cdot 
\Bigl( \nabla_\theta \ln \cT^\theta_i\bigl(s_{i+1}\,|\,s_i, \pi(\tau_i) \bigr)
+ \nabla_\theta \ln \cE^\theta_{i+1}(o_{i+1}\,|\,s_{i+1}) \Bigr)\\
&\quad\qquad\cdot
\Bigl(\prod_{h=1}^{H-1}
\cT^\theta_h\bigl(s_{h+1}\,|\,s_h, \pi(\tau_h) \bigr)
\cdot
\cE^\theta_{h+1}(o_{h+1}\,|\,s_{h+1}) \Bigr)
\ud o_1\cdots\ud o_H
\ud s_1 \cdots \ud s_H.
\$
Therefore, we have the following unbiased estimator of $\nabla_\theta J(\theta,\pi)$,
\$
\hat{\nabla_\theta J}(\theta,\pi)
=
R(\bo_1,\ba_1,\ldots,\bo_H,\ba_H)
\cdot \sum_{i=1}^{H-1}
\bigl( \nabla_\theta \ln \cT^\theta_i(\bs_{i+1}\,|\, \bs_i, \ba_i )
+ \nabla_\theta \ln \cE^\theta_{i+1}(\bo_{i+1}\,|\,\bs_{i+1}) \bigr),
\$
where $(\bs_1, \bo_1, \ba_1, \ldots, \bs_H, \bo_H, \ba_H)$ is a trajectory of the POMDP with respect to the parameter $\theta$ and policy $\pi$. Note that for the given parameter $\theta$, the trajectory can be obtained from a simulator rather than the real environment, which does not affect the sample complexity result in the main paper. Combining the above estimator with \eqref{2205261159}-\eqref{220527556}, we construct $g_\theta(\theta,\lambda,w)$ as
\$
g_{\theta}(\theta,\lambda,w)=
\hat{\nabla_\theta J}\bigl(\theta, \hat{\pi}(\theta) \bigr) +
\sum_{(h,a,a')\in\cI} 
\lambda_{h,a,a'} \cdot 
\frac{1}{|\cB|}\sum_{x\in\cB}\hat{\nabla_\theta\mathbb{S} \mathbb{F}^\theta_{h,a'} f^w_{h,a,a'}}(x),
\$
where the second term is defined in \eqref{220527555}.

\section{Proofs for Section \ref{sec-algo}}
In this section, we present the proofs for the results in Section \ref{sec-algo}.
\subsection{Proof of Lemma \ref{blemma}}\label{blemmap}
\begin{proof}
Following the notation in Lemma \ref{blemma} and by the definition of $\mB^{\theta,\pi}_{h}$ in \eqref{bdef}, we have
\#\label{9136240}
&\EE_{\theta}[ (\mB^{\theta,\pi}_{h} f)(\overbtau_h)\,|\, \sigma_{h-1}] \\
&\quad=
\int_{\cS\times\cO^3} f \bigl(\overtau_h^\dagger, \pi(\tau_h^\dagger), \to_{h+1} \bigr) \cdot 
p_{\theta}\bigl(\bo_h=\tilde{o}_h, \bo_{h+1}=\tilde{o}_{h+1} \,|\, \bs_h=\tilde{s}_h, \ba_h=\pi(\tau_h^\dagger) \bigr) \notag\\
&\quad\qquad\qquad\qquad
\cdot \cZ^\theta_{h}(\tilde{s}_h, o_h)
\cdot p_\theta(\bo_h=o_h \,|\,\sigma_{h-1}) \ud o_h  \ud \to_h \ud \to_{h+1} \ud \tilde{s}_h.
 \notag
\#
Here, invoking Lemma \ref{bridge-lm}, we have
\$
\int_\cO \cZ^\theta_{h}(\tilde{s}_h, o_h)
\cdot p_\theta(\bo_h=o_h \,|\,\sigma_{h-1})\ud o_h 
= p_\theta(\bs_h=\tilde{s}_h \,|\, \sigma_{h-1}).
\$
Thus, we can rewrite \eqref{9136240} as
\#\label{913624}
&\EE_{\theta}[ (\mB^{\theta,\pi}_{h} f)(\overbtau_h)\,|\, \sigma_{h-1}] \\
&\quad=
\int_{\cS\times\cO^2} f \bigl(\overtau_h^\dagger, \pi(\tau_h^\dagger), \to_{h+1} \bigr) \cdot 
p_{\theta}\bigl(\bo_h=\tilde{o}_h, \bo_{h+1}=\tilde{o}_{h+1} \,|\, \bs_h=\tilde{s}_h, \ba_h=\pi(\overtau_h^\dagger) \bigr) \notag\\
&\quad\qquad\qquad\qquad
\cdot p(\bs_h=\tilde{s}_h \,|\,\sigma_{h-1})  \ud \to_h \ud \to_{h+1} \ud \tilde{s}_h
 \notag\\
 &\quad=
\int_{\cO^2} f \bigl(\overtau_h^\dagger, \pi(\tau_h^\dagger), \to_{h+1} \bigr) \cdot 
p_{\theta,\pi}(\bo_h=\tilde{o}_h, \bo_{h+1}=\tilde{o}_{h+1} \,|\, \sigma_{h-1} )
\ud \to_h \ud \to_{h+1}. \notag
\#
where the second equality uses the independence between $(\bo_h, \bo_{h+1})$ and $\btau_{h-1}$ conditioning on $(\bs_{h}, \ba_h)$. Replacing the notations $\to_h$ and $\to_{h+1}$ of the integral variables on the right-hand side of \eqref{913624} by $o_h$ and $o_{h+1}$, respectively, we obtain
\#\label{220507754}
&\EE_{\theta}[ (\mB^{\theta,\pi}_{h} f)(\overbtau_h)\,|\, \sigma_{h-1}] \\
&\quad=
\int_{\cO^2} f \bigl(\overtau_h, \pi(\tau_h), o_{h+1} \bigr) \cdot 
p_{\theta,\pi}(\bo_h=o_h, \bo_{h+1}=o_{h+1} \,|\, \sigma_{h-1} )
\ud o_h \ud o_{h+1}\notag\\
&\quad=
\EE_{\theta,\pi}[f(\overbtau_{h+1}) \,|\,\sigma_{h-1}],\notag
\#
where we denote $\overtau_h=(\overtau_{h-1}, a_{h-1}, o_h)$ and $\tau_h=(\tau_{h-1}, o_h)$. On the other hand, by the tower property of the expectation and the definition of $\PP^{\theta,\pi}_h$ in \eqref{914756}, we have
\#\label{220507755}
\EE_{\theta,\pi}[f(\overbtau_{h+1}) \,|\,  \sigma_{h-1}] 
=\EE_{\theta,\pi}\bigl[
\EE_{\theta,\pi}[f(\overbtau_{h+1}) \,|\, \overbtau_h]
 \,\big|\, \sigma_{h-1}\bigr] 
& =\EE_{\theta,\pi}[ (\PP^{\theta,\pi}_hf)(\overbtau_{h}) \,|\, \sigma_{h-1}]. 
\#
Combining \eqref{220507754} and \eqref{220507755}, we conclude the proof of Lemma \ref{blemma}.
\end{proof}

\subsection{Proof of Corollary \ref{coro1}}\label{coro1p}
\begin{proof}
We prove the result in \eqref{1004938} by induction over $h\in[H+1]$. When $h=H+1$, by the definition of the value function in \eqref{vdef}, we have
\$
\EE_{\theta,\pi}[V^{\theta,\pi}_{H+1}(\overbtau_{H+1}) \,|\, \sigma_H]
=\EE_{\theta,\pi}[R(\overbtau_{H+1}) \,|\, \sigma_H ]
=\EE_{\theta,\pi}\Bigl[ \sum_{i=1}^H \br_i \,\Big|\, \sigma_H  \Bigr].
\$
for any $(\overtau_H,a_H)\in\overline{\Gamma}_H\times\cA$. Recall that the variables $\overtau_H$ and $a_H$ appear in the event $\sigma_H$, which is defined in \eqref{sigmadef}. 

Assume that \eqref{1004938} holds when $h=j+1$ for some fixed $j\le H$. In other words, assume that we have
\#\label{220507121}
\EE_{\theta,\pi}[V^{\theta,\pi}_{j+1}(\overbtau_{j+1}) \,|\, \sigma_{j}]
=\EE_{\theta,\pi}\Bigl[ \sum_{i=1}^H \br_i \,\Big|\, \sigma_{j}  \Bigr]
\#
for any $(\overtau_j,a_j)\in\overline{\Gamma}_j\times\cA$. Then, by the definition of the value function in \eqref{vdef} and invoking Lemma \ref{blemma}, we have
\#\label{220507254}
\EE_{\theta,\pi}[V^{\theta,\pi}_{j}(\overbtau_{j}) \,|\, \sigma_{j-1}]
&=\EE_{\theta,\pi}[(\mB^{\theta,\pi}_jV^{\theta,\pi}_{j+1})(\overbtau_{j}) \,|\, \sigma_{j-1}] \\
&=\EE_{\theta,\pi}[(\PP^{\theta,\pi}_jV^{\theta,\pi}_{j+1})(\overbtau_{j}) \,|\, \sigma_{j-1}]\notag\\
&=\EE_{\theta,\pi}[V^{\theta,\pi}_{j+1}(\overbtau_{j+1}) \,|\, \sigma_{j-1}], \notag
\#
where the second equality uses the tower property of the conditional expectation. Combining \eqref{220507254} with the induction assumption in \eqref{220507121}, we have that \eqref{1004938} holds when $h=j$. Thus, by induction we have that \eqref{1004938} holds for any $h\in[H+1]$.

Therefore, we conclude the proof of Corollary \ref{coro1}.
\end{proof}

\subsection{Proof of Lemma \ref{esteq-lemma}} \label{esteq-lemmap}
We prove a more general version of Lemma \ref{esteq-lemma}. In detail, we replace the true parameter $\theta^*$ in Lemma \ref{esteq-lemma} by any parameter $\theta\in\Theta$.
\begin{lemma}[General Version of Lemma \ref{esteq-lemma}]\label{esteq-lemma-g}
For any $(h,\theta,a,a',\pi)\in\{2,\ldots,H\}\times\Theta\times\cA^2\times\overline{\Pi}$, we have
\$
\EE_{X\sim {\rho}^{\theta,\pi}_{h,a,a'}}[ (\mathbb{F}^{\theta}_{h,a'} f -  f)(X)] = 0, 
\quad
\text{for any $f\in L^\infty(\cO^3)$}.
\$
Here, the distribution ${\rho}^{\theta,\pi}_{h,a,a'}\in\Delta(\cO^3)$ is defined by
\$
{\rho}^{\theta,\pi}_{h,a,a'}(o_{h-1}, o_h, o_{h+1})
=p_{\theta,\pi}\bigl(
\bo_{h-1}=o_{h-1}, \bo_h=o_h, \bo_{h+1}=o_{h+1} \,|\,\ba_{h-1}=a,\ba_h=a' \bigr),
\$
for any $o_{h-1}$, $o_h$, $o_{h+1}\in\cO$. Also, we have $\|\mathbb{F}^{\theta}_{h,a'}\|_{\infty\rightarrow\infty}\le \gamma$.
\end{lemma}
\begin{proof}
By the definition of $\mathbb{F}^{\theta}_{h,a'}$ in \eqref{fopdef}, we have
\#\label{220507452}
&\EE_{X\sim {\rho}^{\theta,\pi}_{h,a,a'}}[ (\mathbb{F}^{\theta}_{h,a'}f)(X)]\\
&\quad=
\int_{\cO^3} (\mathbb{F}^{\theta}_{h,a'}f)(o_{h-1}, o_h, o_{h+1})
\cdot  {\rho}^{\theta,\pi}_{h,a,a'}(o_{h-1}, o_h, o_{h+1}) \ud o_{h-1}\ud o_h \ud o_{h+1}\notag\\
&\quad=
\int_{\cO^5} f(o_{h-1}, \tilde{o}_h, \tilde{o}_{h+1})
\cdot \cB^{\theta}_{h,a'}(o_{h}, \tilde{o}_h, \tilde{o}_{h+1})\notag\\
&\quad\qquad\qquad\qquad
\cdot  {\rho}^{\theta,\pi}_{h,a,a'}(o_{h-1}, o_h, o_{h+1}) \ud o_{h-1}\ud o_h \ud o_{h+1}\ud \to_h \ud \to_{h+1}.\notag
\#
Here, by the definition of ${\rho}^{\theta,\pi}_{h,a,a'}$ in Lemma \ref{esteq-lemma}, we have
\#\label{220507455}
\int_{\cO} {\rho}^{\theta,\pi}_{h,a,a'}(o_{h-1}, o_h, o_{h+1})\ud o_{h+1}
=p_{\theta,\pi}(\bo_{h-1}=o_{h-1}, \bo_h=o_{h} \,|\, \ba_{h-1}=a).
\#
Combining \eqref{220507452} and \eqref{220507455}, we obtain
\#\label{220508619}
&\EE_{X\sim {\rho}^{\theta,\pi}_{h,a,a'}}[ (\mathbb{F}^{\theta}_{h,a'}f)(X)]\\
&\quad=
\int_{\cO^4} f(o_{h-1}, \tilde{o}_h, \tilde{o}_{h+1})
\cdot \cB^{\theta}_{h,a'}(o_{h}, \tilde{o}_h, \tilde{o}_{h+1})\notag\\
&\quad\qquad\qquad\qquad
\cdot  p_{\theta,\pi}(\bo_{h-1}=o_{h-1}, \bo_h=o_{h} \,|\, \ba_{h-1}=a) \ud o_{h-1}\ud o_h \ud o_{h+1}\ud \to_h \ud \to_{h+1}, \notag
\#
where $\cB^{\theta}_{h,a'}(o_{h}, \tilde{o}_h, \tilde{o}_{h+1})$ takes the form
\$
\cB^{\theta}_{h,a'}(o_{h}, \tilde{o}_h, \tilde{o}_{h+1})
=\int_{\cS} 
p_{\theta}(\bo_h=\tilde{o}_h, \bo_{h+1}=\tilde{o}_{h+1} \,|\, \bs_h=\tilde{s}_h, \ba_h=a')
\cdot \cZ^\theta_{h}(\tilde{s}_h, o_h) \ud \tilde{s}_h
\$
following the definition in \eqref{b-func-def}. By the Markov property of the POMDP, we can write
\#\label{220507549}
&p_{\theta,\pi}(\bo_{h-1}=o_{h-1}, \bo_h=o_{h} \,|\, \ba_{h-1}=a) \\
&\quad=
\int_{\cS^2} \cE^\theta_h(o_h \,|\, s_h)
\cdot \cT^\theta_h(s_h\,|\,s_{h-1}, a)
\cdot p_{\theta,\pi}(\bs_{h-1}=s_{h-1}, \bo_{h-1}=o_{h-1} ) \ud s_h \ud s_{h-1}\notag
\#
By Assumptions \ref{asmp1} and \ref{asmp2}, we have
\#\label{220507557}
\int_{\cS\times\cO} \cZ^\theta_{h}(\tilde{s}_h, o_h)
\cdot \cE^\theta_h(o_h \,|\, s_h)
\cdot \cT^\theta_h(s_h\,|\,s_{h-1}, a) \ud o_h \ud s_h
=\cT^\theta_h(\tilde{s}_h\,|\,s_{h-1}, a).
\#
Combining \eqref{220507549} and \eqref{220507557}, we obtain
\$
&\int_{\cO} \cZ^\theta_{h}(\tilde{s}_h, o_h)
\cdot p_{\theta,\pi}(\bo_{h-1}=o_{h-1}, \bo_h=o_{h} \,|\, \ba_{h-1}=a) 
\ud o_h \ud s_h \\
&\quad=\int_{\cS} 
 \cT^\theta_h(\tilde{s}_h\,|\,s_{h-1}, a)
\cdot p_{\theta,\pi}(\bs_{h-1}=s_{h-1}, \bo_{h-1}=o_{h-1} ) \ud s_h \ud s_{h-1}\\
&\quad=
p_{\theta,\pi}(\bo_{h-1}=o_{h-1}, \bs_{h}=\tilde{s}_h \,|\, \ba_{h-1}=a),
\$
which implies
\#\label{220507618}
&\int_\cO \cB^{\theta}_{h,a'}(o_{h}, \tilde{o}_h, \tilde{o}_{h+1})
\cdot p_{\theta,\pi}(\bo_{h-1}=o_{h-1}, \bo_h=o_{h} \,|\, \ba_{h-1}=a) \ud o_{h} \\
&\quad=p_{\theta,\pi}(\bo_{h-1}=o_{h-1}, \bo_h=\to_{h}, \bo_{h+1}=\to_{h+1} \,|\, \ba_{h-1}=a, \ba_h=a') \notag\\
&\quad={\rho}^{\theta,\pi}_{h,a,a'}(o_{h-1}, \tilde{o}_h, \tilde{o}_{h+1}). \notag
\#
Then, combining \eqref{220508619} and \eqref{220507618}, we have
\$
&\EE_{X\sim {\rho}^{\theta,\pi}_{h,a,a'}}[ (\mathbb{F}^{\theta}_{h,a'}f)(X)]\\
&\quad=
\int_{\cO^3} f(o_{h-1}, \tilde{o}_h, \tilde{o}_{h+1})
\cdot {\rho}^{\theta,\pi}_{h,a,a'}(o_{h-1}, \tilde{o}_h, \tilde{o}_{h+1}) \ud o_{h-1}\ud\to_h \ud \to_{h+1}
=\EE_{X\sim {\rho}^{\theta,\pi}_{h,a,a'}}[ f(X)].
\$

In the sequel, we prove $\|\mathbb{F}^{\theta}_{h,a'}\|_{\infty\rightarrow\infty}\le \gamma$. It suffices to prove
\#\label{22510409}
|(\mathbb{F}^{\theta}_{h,a'}f)(o_{h-1},o_h,o_{h+1})| 
=\Bigl| \int_{\cO^2}
f(o_{h-1}, \tilde{o}_h, \tilde{o}_{h+1})   \cdot \cB^{\theta}_{h, a}(o_h,\tilde{o}_h, \tilde{o}_{h+1})\ud \tilde{o}_h\ud \tilde{o}_{h+1} \Bigr| \le \gamma,
\#
for any $f\in L^\infty(\cO^3)$ such that $\|f\|_\infty\le1$ and $o_{h-1},o_h,o_{h+1}\in\cO$. By the definition of the function $\cB^\theta_{h,a}$ in \eqref{b-func-def}, we have
\#\label{22510403}
&\Bigl| \int_{\cO^2}
f(o_{h-1}, \tilde{o}_h, \tilde{o}_{h+1})   \cdot \cB^{\theta}_{h, a}(o_h,\tilde{o}_h, \tilde{o}_{h+1})\ud \tilde{o}_h\ud \tilde{o}_{h+1} \Bigr|\\
&\quad \le \int_{\cO^2}
 |\cB^{\theta}_{h, a}(o_h,\tilde{o}_h, \tilde{o}_{h+1}) |
 \ud \tilde{o}_h
 \ud \tilde{o}_{h+1} \notag\\
&\quad\le
\int_{\cS\times\cO^2}  p_{\theta}(\tilde{\bo}_h=\tilde{o}_h, \tilde{\bo}_{h+1}=\tilde{o}_{h+1} \,|\, \tilde{\bs}_h=\tilde{s}_h, \tilde{\ba}_h=a)
\cdot |\cZ^\theta_{h}(\tilde{s}_h, o_h) |
\ud \tilde{s}_h  \ud \tilde{o}_h \ud \tilde{o}_{h+1} \notag \\
&\quad=
\int_{\cS}  |\cZ^\theta_{h}(\tilde{s}_h, o_h) |
\ud \tilde{s}_h, \notag
\#
for any $o_{h-1}, o_h\in\cO$. Note that by Assumption \ref{asmp2}, we have
\#\label{22510404}
\int_\cS |\cZ^\theta_h(\tilde{s}_h,o_h)| \ud \tilde{s}_h
= \| \mZ^\theta_h \delta_{o_h} \|_1 \le \gamma  \cdot \|\delta_{o_h} \|_1 =\gamma,
\#
for any $(h,\theta,o_h)\in[H]\times\Theta\times\cO$. Here, $\delta_{o_h}$ is the Dirac delta function defined on $\cO$, whose value is zero everywhere except at $o_h$, and whose integral over $\cO$ is equal to one. Combining \eqref{22510403} and \eqref{22510404}, we have that \eqref{22510409} holds.

Therefore, we conclude the proof of Lemma \ref{esteq-lemma-g}.
\end{proof}

\section{Proof of Theorem \ref{mainthm}}\label{mainthmp}
\begin{proof}
For any $\delta>0$, by the definition of $(\theta_k, \pi_k)$ in \eqref{915459} and the first statement in Lemma \ref{lemma-acc},  with probability at least $1-\delta$, it holds that
\#\label{1004948}
J(\theta^*,\pi^*)-J(\theta^*,\pi_k)\le J(\theta_k,\pi_k) - J(\theta^*,\pi_k)
\#
for all $k\in[K]$. By further applying Lemma \ref{lemma-dec} to the right-hand side of \eqref{1004948} and using the definition of the error function $e^k_h$ in \eqref{edef}, we obtain
\#\label{vdiff2}
J(\theta^*,\pi^*)-J(\theta^*,\pi_k)
&\le \sum_{h=1}^{H} \EE_{\theta^*, \pi_k}[ (\mB^{\theta_k,\pi_k}_{h}V^{\theta_k,\pi_k}_{h+1})(\overbtau_h) - (\mB^{\theta^*,\pi_k}_{h} V^{\theta_k,\pi_k}_{h+1})(\overbtau_{h}) ] \notag \\
&= \sum_{h=1}^{H} \EE_{\theta^*, \pi_k}\bigl[\EE_{\theta^*, \pi_k}[ (\mB^{\theta_k,\pi_k}_{h}V^{\theta_k,\pi_k}_{h+1})(\overbtau_h) - (\mB^{\theta^*,\pi_k}_{h} V^{\theta_k,\pi_k}_{h+1})(\overbtau_{h}) ]
\,\bigl|\,\bs_{h-1}\bigr]
 \notag \\
&\le \sum_{h=1}^{H}\EE_{\theta^*,\pi_k}[e^k_h(\bs_{h-1})],
\#
where the equality uses the tower property of the expectation. Telescoping both sides of \eqref{vdiff2} for $k\in[K]$ and applying Lemma \ref{lemma-tel}, we obtain
\#\label{10011038}
&\sum_{k=1}^K J(\theta^*,\pi^*)-J(\theta^*,\pi_k)\notag\\
&\quad\le
 \sum_{h=1}^{H} \sum_{k=1}^K \EE_{\theta^*,\pi_k}[e^k_h(\bs_{h-1})] \notag\\
&\quad\le Hd_{\rm s}
\Bigl(4\gamma H+2\log K \cdot
\max_{k\in[K]} \bigl(k\cdot\EE_{\theta,\overline{\pi}_k}[e^k_{h}(\bs_{h-1})]\bigr)  \Bigr)
\#
By applying the second statement of Lemma \ref{lemma-acc} to the right-hand side of \eqref{10011038}, we further obtain
\$
&\sum_{k=1}^K J(\theta^*,\pi^*)-J(\theta^*,\pi_k)\notag\\
&\quad\le Hd_{\rm s}
\Bigl(4\gamma H+2\log K \cdot
\max_{k\in[K]} \bigl(k\cdot 2HA^2\gamma^2\beta\cdot k^{-1/2} \bigr)  \Bigr) \\
&\quad\le Hd_{\rm s}
\bigl(4\gamma H+2\log K \cdot 2HA^2\gamma^2\beta\cdot K^{1/2}   \bigr), 
\$
which concludes the proof of Theorem \ref{mainthm}.
\end{proof}

\section{Proofs for Section \ref{sketch}}\label{sketchp}
In this section, we present the proofs for the results in Section \ref{sketch}.
\subsection{Proof of Lemma \ref{lemma-dec}}\label{lemma-decp}
\begin{proof}
By Corollary \ref{coro1} and the definition of the value function in \eqref{vdef}, we can write
\$
J(\theta,\pi)=\EE_{\theta',\pi}[V^{\theta,\pi}_1(\overbtau_1)],
\quad
J(\theta',\pi)=\EE_{\theta',\pi}[V^{\theta,\pi}_{H+1}(\overbtau_{H+1})],
\$
which implies
\#\label{910846}
J(\theta,\pi)-J(\theta',\pi)
=\sum_{h=1}^H \bigl(
\EE_{\theta',\pi}[V^{\theta,\pi}_h(\overbtau_h)] - \EE_{\theta',\pi}[V^{\theta,\pi}_{h+1}(\overbtau_{h+1})] \bigr).
\#
By the definition of $V^{\theta,\pi}_h$ in \eqref{vdef}, we have
\#\label{9108481}
\EE_{\theta',\pi}[V^{\theta,\pi}_h(\overbtau_h)]
&=\EE_{\theta',\pi}[(\mB^{\theta,\pi}_{h}V^{\theta,\pi}_{h+1})(\overbtau_h)].
\#
Also, by Lemma \ref{blemma}, we have
\#
\EE_{\theta',\pi}[V^{\theta,\pi}_{h+1}(\overbtau_{h+1})]
&=\EE_{\theta',\pi}[(\PP^{\theta',\pi}_{h}V^{\theta,\pi}_{h+1})(\overbtau_h)]
=\EE_{\theta',\pi}[(\mB^{\theta',\pi}_{h}V^{\theta,\pi}_{h+1})(\overbtau_h)].\label{9108482}
\#
Plugging \eqref{9108481} and \eqref{9108482} into the right-hand side of \eqref{910846}, we obtain
\$
J(\theta,\pi)-J(\theta',\pi)
=\sum_{h=1}^H
\EE_{\theta',\pi}[(\mB^{\theta,\pi}_{h}V^{\theta,\pi}_{h+1})(\overbtau_h)
-(\mB^{\theta',\pi}_{h}V^{\theta,\pi}_{h+1})(\overbtau_h)] ,
\$
which concludes the proof of Lemma \ref{lemma-dec}.
\end{proof}

\subsection{Proof of Lemma \ref{lemma-acc}}\label{lemma-accp}
Before proving Lemma \ref{lemma-acc}, we present several auxiliary lemmas for the proof of Lemma \ref{lemma-acc}. Recall that we define the projection operator $\mathbb{S}$ in \eqref{sopdef}. The following lemma verifies the projection property of $\mathbb{S}$ as mentioned in \eqref{proj-eq}.
\begin{lemma}\label{proj-conjugate}
For any $f\in L^\infty(\cO^3)$ and $\rho\in\Delta(\cO^3)$, we have
\$
\EE_{X\sim \rho}[(\mathbb{S}f)(X)] = \int_{\cO^3} f(x) \cdot \hat{\rho}(x)\ud x,
\$
where $\hat{\rho}$ is the projection of $\rho$ onto ${\rm linspan}(\{\phi_i\}_{i=1}^{d_{\rm o}})$ with respect to the distance defined in \eqref{distance-def} and takes the form
\#\label{pdagger-def}
\hat{\rho}(o_{h-1}, o_h, o_{h+1})=
\sum_{j\in[d_{\rm o}]} \phi_j(o_{h-1}, o_h, o_{h+1}) \cdot 
\sum_{i\in[d_{\rm o}]} [ G^{-1}]_{i,j} \cdot \EE_{X\sim \phi_i, Y\sim \rho}[\cK(X, Y)],
\#
for any $o_{h-1}, o_h, o_{h+1}\in\cO^3$. 
\end{lemma}
\begin{proof}
See Section \ref{proj-conjugate-p} for a detailed proof. 
\end{proof}
Recall that $\{\cD_{h,a,a'}\}_{(h,a,a')\in\{2,\ldots, H\}\times\cA^2}$ is the dataset in Algorithm \ref{algo}, which is updated in each iteration. For any $k\in[K]$, we denote by $\cD^k_{h,a,a'}$ the status of $\cD_{h,a,a'}$ after the exploration phase of the $k$-th iteration of Algorithm \ref{algo}. We denote by $\hat{\cD}^k_{h,a,a'}$ the empirical distribution induced by the dataset $\cD^k_{h,a,a'}$. For any $(k, h, a, a')\in[K]\times\{2,\ldots, H\}\times\cA^2$, as a special case of Lemma \ref{proj-conjugate} for $\rho=\hat{\cD}^k_{h,a,a'}$, we define the function $\hat{\rho}^k_{h,a,a'}:\cO^3\rightarrow\RR$ by
\$
\hat{\rho}^k_{h,a,a'}(o_{h-1}, o_h, o_{h+1})=
\sum_{j\in[d_{\rm o}]} \phi_j(o_{h-1}, o_h, o_{h+1}) \cdot [\hat{w}^k_{h,a,a'}]_j,
\$
where the vector $\hat{w}^k_{h,a,a'} \in \RR^{d_{\rm o}}$ is defined by
\#\label{225101139}
[\hat{w}^k_{h,a,a'}]_j=\sum_{i\in[d_{\rm o}]} [ G^{-1}]_{i,j} \cdot \EE_{X\sim \phi_i, Y\sim \hat{\cD}^k_{h,a,a'}}[\cK(X, Y)],
\#
for any $j\in[d_{\rm o}]$. Here, the matrix $G\in \RR^{d_{\rm o}\times d_{\rm o}}$ is defined in \eqref{gmatdef}. The following lemma shows that,  with high probability, $\hat{\rho}^k_{h,a,a'}$ converges to $\rho^{\overline{\pi}_k}_{h,a,a'}$ as $k$ goes to infinity. The convergence is with respect to the $L^1$-norm in $\cO^3$, which guarantees the generalization power of the solution to the minimax problem in \eqref{minimax}.
\begin{lemma}\label{event-lemma}
For any fixed $\delta>0$, we define the event $\cG$ as
\$
\|\hat{\rho}^k_{h,a,a'}
- \rho^{\overline{\pi}_k}_{h,a,a'} \|_1\le
d_{\rm o}^{3/2}/\alpha \cdot \sqrt{ 8\log(2KHA^2/\delta)} \cdot k^{-1/2},
\$
for any $(k,h,a,a')\in[K]\times\{2,\ldots,H\}\times\cA^2$. Then, it holds that $\cG$ happens with probability at least $1-\delta$.
\end{lemma}
\begin{proof}
See Section \ref{event-lemma-p} for a detailed proof. 
\end{proof}
Moreover, for ease of presentation, we define the operator $\mathbb{V}^\theta_{h,a}: L^1(\cO^3)\rightarrow L^1(\cO^3)$ for any $(h,a,\theta)\in\{2,\ldots, H\}\times\cA\times\Theta$ by
\#\label{vopdef}
(\mathbb{V}^\theta_h f)(o_{h-1},\to_h, \to_{h+1})=\int_{\cO^2}
 \cB^{\theta}_{h, a}(o_h,\tilde{o}_h, \tilde{o}_{h+1})
 \cdot f(o_{h-1}, o_h, o_{h+1})
 \ud o_h\ud o_{h+1}, 
\#
for any $f\in L^1(\cO^3)$ and $o_{h-1}, \to_h, \to_{h+1}\in\cO$, which is the conjugate (i.e., transpose) of the operator $\mathbb{F}^\theta_{h,a}$ defined in \eqref{fopdef}. Recall that we define $\rho^{\theta,\pi}_{h,a,a'}$ in Lemma \ref{esteq-lemma-g}, which is the general form of $\rho^{\pi}_{h,a,a'}$ for any $\theta\in\Theta$. The following lemma mirrors Lemma \ref{esteq-lemma-g}.
\begin{lemma}\label{vop-lemma}
For any $(h,\theta,a,a',\pi)\in\{2,\ldots, H\}\times\Theta\times\cA^2\times\overline{\Pi}$, we have
\$
\|\mathbb{V}^\theta_{h,a'} \rho^{\theta,\pi}_{h,a,a'} 
-\rho^{\theta,\pi}_{h,a,a'}\|_1=0.
\$
Also, we have $\|\mathbb{V}^\theta_{h,a'}\|_{1\rightarrow1}\le \gamma$.
\end{lemma}
\begin{proof}
See Section \ref{vop-lemma-p} for a detailed proof. 
\end{proof}
\noindent
\textbf{Proof of Lemma \ref{lemma-acc}:}
\begin{proof}
In the following, we condition on the event $\cG$ defined in Lemma \ref{event-lemma}, which happens with probability at least $1-\delta$.

\noindent
\textbf{Proof of the first statement:} Recall that we denote by $\cD^k_{h,a,a'}$ the status of $\cD_{h,a,a'}$ after the exploration phase of the $k$-th iteration of Algorithm \ref{algo}. Correspondingly, we define the function $L^k:\Theta\rightarrow\RR$ by
\$
L^k(\theta)=\max_{f\in L^\infty(\cO^3): \|f\|_\infty\le1}
~\max_{(h, a,a')\in\{2,\ldots,H\}\times\cA^2} 
~\EE_{X\sim\hat{\cD}^k_{h,a,a'}}
[  (\mathbb{S} \mathbb{F}^\theta_{h,a'} f -   \mathbb{S}f)(X) ],
\$
for any $\theta\in\Theta$, which corresponds to the function $L$ in \eqref{risk} in the planning phase of the $k$-th iteration of Algorithm \ref{algo}. To prove $\theta^*\in\Theta^k$, it suffices to prove that
\#\label{220510211}
L^k(\theta^*) \le \beta \cdot k^{-1/2}.
\#
For notational simplicity, we write $\rho^k_{h,a,a'}$ in short for $\rho^{\overline{\pi}_k}_{h,a,a'}$. For any $f\in L^\infty(\cO^3)$ such that $\|f\|_\infty\le1$ and $(k,h, a,a')\in[K]\times\{2,\ldots,H\}\times\cA^2$, we have
\#\label{220510145}
&\EE_{X\sim\hat{\cD}^k_{h,a,a'}}
[  (\mathbb{S} \mathbb{F}^{\theta^*}_{h,a'} f -   \mathbb{S}f)(X) ] \\
&\quad=\EE_{X\sim\hat{\cD}^k_{h,a,a'}}
[  (\mathbb{S} \mathbb{F}^{\theta^*}_{h,a'} f -   \mathbb{S}f)(X) ]
- \EE_{X\sim {\rho^k_{h,a,a'}} }
[  ( \mathbb{F}^{\theta^*}_{h,a'} f -   f)(X) ] \notag\\
&\quad=
\Bigl( \int_{\cO^3} ( \mathbb{F}^{\theta^*}_{h,a'} f)(x) \cdot \hat{\rho}^k_{h,a,a'}(x)\ud x
- \EE_{X\sim {\rho^k_{h,a,a'}} }
[  ( \mathbb{F}^{\theta^*}_{h,a'} f)(X)] \Bigr) \notag\\
&\quad\qquad +  \Bigl( \int_{\cO^3} f(x) \cdot \hat{\rho}^k_{h,a,a'}(x)\ud x
- \EE_{X\sim {\rho^k_{h,a,a'}} }[ f(X)] \Bigr), \notag
\#
where the first equality uses
\$
\EE_{X\sim {\rho^k_{h,a,a'}} }
[  ( \mathbb{F}^{\theta^*}_{h,a'} f -   f)(X) ]=0,
\$
following Lemma \ref{esteq-lemma} and the second equality is by Lemma \ref{proj-conjugate}. Recall that $\|f\|_\infty\le1$. By Holder's inequality and the definition of the event $\cG$ in Lemma \ref{event-lemma}, we have
\#\label{220510146}
& \int_{\cO^3} f(x) \cdot \hat{\rho}^k_{h,a,a'}(x)\ud x
- \EE_{X\sim {\rho^k_{h,a,a'}} }[ f(X)]\\
&\quad\le \| f\|_\infty \cdot
\| \hat{\rho}^k_{h,a,a'} - \rho^k_{h,a,a'} \|_1 
\le 
d_{\rm o}^{3/2}/\alpha \cdot \sqrt{ 8\log(2KHA^2/\delta)} \cdot k^{-1/2}. \notag
\#
Similarly, we have
\#\label{220510147}
&\int_{\cO^3} ( \mathbb{F}^{\theta^*}_{h,a'} f)(x) \cdot \hat{\rho}^k_{h,a,a'}(x)\ud x
- \EE_{X\sim {\rho^k_{h,a,a'}} }
[  ( \mathbb{F}^{\theta^*}_{h,a'} f)(X)] \\
&\quad\le \|\mathbb{F}^{\theta^*}_{h,a'} f\|_\infty \cdot
\| \hat{\rho}^k_{h,a,a'} - \rho^k_{h,a,a'} \|_1  \notag\\
&\quad\le 
d_{\rm o}^{3/2}\gamma/\alpha \cdot \sqrt{ 8\log(2KHA^2/\delta)} \cdot k^{-1/2} \notag,
\#
where the second inequality uses the fact $\|\mathbb{F}^{\theta^*}_{h,a'}\|_{\infty\rightarrow\infty}\le \gamma$ from Lemma \ref{esteq-lemma}. Then, by combining \eqref{220510145}, \eqref{220510146}, and \eqref{220510147} with the condition of $\beta$ in \eqref{beta-cond}, we have that the inequality in \eqref{220510211} holds for all $k\in[K]$.

\noindent
\textbf{Proof of the second statement:} Invoking Lemma \ref{regret-transform}, we have
\#\label{913107}
\EE_{\theta^*,\overline{\pi}_k}[e^k_h(\bs_h)]
&\le \gamma^2H\cdot\sum_{a,a'\in\cA}
\bigl\| \mathbb{V}^{\theta_k}_{h,a'} \rho^{\overline{\pi}_k}_{h,a,a'} - \rho^{\overline{\pi}_k}_{h,a,a'}\bigr\|_1,
\#
for any $(k,h)\in[K]\times\{2,\ldots,H\}$. By the triangle inequality, we can write
\$
&\bigl\| \mathbb{V}^{\theta_k}_{h,a'} \rho^{\overline{\pi}_k}_{h,a,a'} - \rho^{\overline{\pi}_k}_{h,a,a'}\bigr\|_1 \\
&\quad\le
\bigl\| \mathbb{V}^{\theta_k}_{h,a'} \rho^{\overline{\pi}_k}_{h,a,a'} - \mathbb{V}^{\theta_k}_{h,a'}\rho^k_{h,a,a'}\bigr\|_1
+\bigl\| \mathbb{V}^{\theta_k}_{h,a'} \rho^k_{h,a,a'} - \rho^k_{h,a,a'}\bigr\|_1
+\bigl\| \rho^k_{h,a,a'} - \rho^{\overline{\pi}_k}_{h,a,a'} \bigr\|_1.
\$
By the definition of $\Theta_k$ in \eqref{csdef} and the fact $\theta_k\in\Theta_k$, we have
\$
\bigl\| \mathbb{V}^{\theta_k}_{h,a'} \rho^k_{h,a,a'} - \rho^k_{h,a,a'}\bigr\|_1 \le \beta\cdot k^{-1/2}
\$
By the definition of the event $\cG$ in Lemma \ref{event-lemma}, we have
\$
\bigl\| \rho^k_{h,a,a'} - \rho^{\overline{\pi}_k}_{h,a,a'} \bigr\|_1
\le d_{\rm o}^{3/2}/\alpha \cdot \sqrt{ 8\log(2KHA^2/\delta)} \cdot k^{-1/2}.
\$
Similarly, we have
\#\label{220510229}
\bigl\| \mathbb{V}^{\theta_k}_{h,a'} \rho^{\overline{\pi}_k}_{h,a,a'} - \mathbb{V}^{\theta_k}_{h,a'}\rho^k_{h,a,a'}\bigr\|_1 
&\le \| \mathbb{V}^{\theta_k}_{h,a'}\|_{1\rightarrow1} \cdot 
\bigl\|\rho^{\overline{\pi}_k}_{h,a,a'} - \rho^k_{h,a,a'}\bigr\|_1  \\
&\le  d_{\rm o}^{3/2} \gamma /\alpha \cdot \sqrt{ 8\log(2KHA^2/\delta)} \cdot k^{-1/2}, \notag
\#
where the second inequality uses the fact $\|\mathbb{V}^{\theta^*}_{h,a'}\|_{1\rightarrow1}\le \gamma$ from Lemma \ref{vop-lemma}. Combing \eqref{913107}-\eqref{220510229} with the condition of $\beta$ in \eqref{beta-cond}, we obtain
\$
\EE_{\theta,\overline{\pi}_k}[e^k_h(\bs_h)]
\le 2 \gamma^2 \beta HA^2 \cdot k^{-1/2},
\$
for any $(k,h)\in[K]\times\{2,\ldots,H\}$.

Therefore, we conclude the proof of Lemma \ref{lemma-acc}.
\end{proof}

\subsection{Proof of Lemma \ref{lemma-tel}}\label{lemma-telp}
\begin{proof}
For any $h\in[H]$ and $\pi\in\overline{\Pi}$, we denote by $\mu^\pi_h$ the marginal distribution of $\bs_h$ with respect to the policy $\pi$ and the true parameter $\theta^*$. By Assumption \ref{asmp1}, we have
\$
\mu^{\pi_k}_{h}\in{\rm conh}(\psi), \quad
\mu^{\overline{\pi}_k}_{h}
=\frac{1}{k}\sum_{i=0}^{k-1}\mu^{\pi_i}_{h}\in{\rm conh}(\psi),
\$
for any $(k,h)\in\{0,\ldots, K\}\times[H]$. Here, $\pi_0$ is the initial policy and $\pi_k$ with $k\in[K]$ is the trained policy in the $k$-th iteration of Algorithm \ref{algo}. Thus, there exist vector $c^k_h,\overline{c}^k_h\in\Delta([d_{\rm s}])\subset\RR^{d_{\rm s}}$ such that
\$
\mu^{\pi_k}_{h}(\cdot)=\psi(\cdot)^\top c^k_h,\quad
\mu^{\overline{\pi}_k}_{h}(\cdot)=\psi(\cdot)^\top \overline{c}^k_h,\quad
\overline{c}^k_h=(1/k)\cdot\sum_{i=0}^{k-1} c^i_h.
\$
Also, we define the vector $b^k_h\in\RR^{d_{\rm s}}$ by
\#\label{920410}
[b^k_h]_i=\EE_{\bs_h\sim\psi_i}[e^k_{h+1}(\bs_h)],
\quad
\text{for any $i\in[d_{\rm s}]$.}
\#
Then, it holds that
\#\label{220512458}
\EE_{\theta^*,\pi_k}[e^k_{h+1}(\bs_h)]=(b^k_h)^\top c^k_h, \quad
\EE_{\theta^*,\overline{\pi}_k}[e^k_{h+1}(\bs_h)]=(b^k_h)^\top \overline{c}^k_h.
\#

For any $i\in[d_{\rm s}]$, we define $\underline{k}_i$ by
\#\label{k0def}
\underline{k}_i=\min\Bigl\{k\in[K]: \sum_{j=1}^{k} [ c^j_h]_i \ge 1
~\text{or}~k=K \Bigr\}.
\#
Then, we can write
\#\label{220512429}
\sum_{k=1}^K \EE_{\theta^*,\pi_k}[e^k_{h+1}(\bs_h)]
&=\sum_{i=1}^{d_{\rm s}} \sum_{k=1}^K 
[b^k_h]_i \cdot [c^k_h]_i \\
&=\sum_{i=1}^{d_{\rm s}} \Bigl(\sum_{k=1}^{\underline{k}_i}
[b^k_h]_i \cdot [c^k_h]_i 
+ \sum_{k=\underline{k}_i+1}^K 
[b^k_h]_i \cdot [c^k_h]_i  \Bigr) \notag
\#
The first summation term on the right-hand side of \eqref{220512429} can be upper bounded as
\#\label{220512430}
\sum_{i=1}^{d_{\rm s}} \sum_{k=1}^{\underline{k}_i}
[b^k_h]_i \cdot [c^k_h]_i 
\le 
2\gamma H\cdot\sum_{i=1}^{d_{\rm s}} \sum_{k=1}^{\underline{k}_i}
 [c^k_h]_i  
 \le 4d_{\rm s}\gamma H.
\#
Here, the first inequality uses Lemma \ref{eupperbound}, which provides an upper bound $2\gamma H$ for each $[b^k_h]_i$. Recall that $[b^k_h]_i$ is defined in \eqref{920410}. Also, the second inequality uses the fact
\$
 \sum_{k=1}^{\underline{k}_i}
 [c^k_h]_i
 = [c^{\underline{k}_i}_h]_i + \sum_{k=1}^{\underline{k}_i-1}
 [c^k_h]_i \le 1 + 1 \le 2,
\$
which is by the definition of $\underline{k}_i$ in \eqref{k0def} and the fact $[c^{\underline{k}_i}_h]_i \le 1$ since $c^{\underline{k}_i}_h\in\Delta([d_{\rm s}])$. In the sequel, we characterize the second summation term on the right-hand side of \eqref{220512429}. For any $i\in[d_{\rm s}]$ and $k\ge\underline{k}_i+1$, we have
\#\label{920420}
[b^k_h]_i \cdot [ c^k_h]_i
&\le \EE_{\theta^*,\overline{\pi}_k}[e^k_{h+1}(\bs_h)]\cdot \frac{[b^k_h]_i \cdot [ c^k_h]_i}{[b^k_h]_i \cdot [ \overline{c}^k_h]_i} \\
&=\bigl(k\cdot \EE_{\theta^*,\overline{\pi}_k}[e^k_{h+1}(\bs_h)] \bigr)
\cdot
\frac{[ c^k_h]_i}{\sum_{j=0}^{k-1} [c^j_h]_i} \notag \\
&\le\bigl(\max_{j\in[K]} \ell\cdot \EE_{\theta^*,\overline{\pi}_\ell}[e^\ell_{h+1}(\bs_h)] \bigr)
\cdot
\frac{[ c^k_h]_i}{\sum_{j=1}^{k-1} [c^j_h]_i}, \notag
\#
where the first inequality is by \eqref{220512458} and the second inequality uses the fact $ [c^j_h]_0\ge0$. Note that for any $(i, k)$ specified above, we have 
\$
\frac{[ c^k_h]_i}{\sum_{j=0}^{k-1} [c^j_h]_i}\in[0,1]
\$ 
since it holds that $[ c^k_h]_i\in[0,1]$ and $\sum_{j=0}^{k-1} [c^j_h]_i \ge 1$ by the definition of $\underline{k}_i$. Then, by applying the inequality $x\le 2\log(1+x)$ for any $x\in[0,1]$, we have
\#\label{220512503}
\frac{[ c^k_h]_i}{\sum_{j=1}^{k-1} [c^j_h]_i}
\le 2\log\Bigl( 1 + \frac{[ c^k_h]_i}{\sum_{j=1}^{k-1} [c^j_h]_i} \Bigr)
=2\log\sum_{j=1}^{k} [c^j_h]_i - 2\log\sum_{j=1}^{k-1} [c^j_h]_i
\#
Combining \eqref{920420} and \eqref{220512503}, we obtain
\#\label{220512431}
\sum_{i=1}^{d_{\rm s}} \sum_{k=\underline{k}_i+1}^K
[b^k_h]_i \cdot [c^k_h]_i  
&\le \bigl(\max_{\ell\in[K]} \ell\cdot \EE_{\theta^*,\overline{\pi}_\ell}[e^\ell_{h+1}(\bs_h)] \bigr)
\cdot \sum_{i=1}^{d_{\rm s}} \sum_{k=\underline{k}_i+1}^K
 \bigl( 
2\log\sum_{j=1}^{k} [c^j_h]_i - 2\log\sum_{j=1}^{k-1} [c^j_h]_i
\bigr) \notag\\
&= 2\bigl(\max_{\ell\in[K]} \ell\cdot \EE_{\theta^*,\overline{\pi}_\ell}[e^\ell_{h+1}(\bs_h)] \bigr)
\cdot \sum_{i=1}^{d_{\rm s}} 
 \bigl( 
\log\sum_{j=1}^{K} [c^j_h]_i - \log\sum_{j=1}^{\underline{k}_i} [c^j_h]_i
\bigr) \notag\\
&\le
2\bigl(\max_{\ell\in[K]} \ell\cdot \EE_{\theta^*,\overline{\pi}_\ell}[e^\ell_{h+1}(\bs_h)] \bigr)
\cdot d_{\rm s} \cdot \log K.
\#
Plugging \eqref{220512430} and \eqref{220512431} into the right-hand side of \eqref{220512429}, we obtain
\$
\sum_{k=1}^K \EE_{\theta^*,\pi_k}[e^k_{h+1}(\bs_h)]
\le 
4d_{\rm s}\gamma H
+2\bigl(\max_{\ell\in[K]} \ell\cdot \EE_{\theta^*,\overline{\pi}_\ell}[e^\ell_{h+1}(\bs_h)] \bigr)
\cdot d_{\rm s} \cdot \log K,
\$
which concludes the proof of Lemma \ref{lemma-tel}.
\end{proof}

\section{Auxiliary Lemmas}

In this section, we present (the proofs for) the auxiliary lemmas invoked in previous sections.

\subsection{Proof of Lemma \ref{proj-conjugate}}\label{proj-conjugate-p}
\begin{proof} 
To see that $\hat{\rho}$ defined in \eqref{pdagger-def} is the projection, we consider the minimization problem
\#\label{22610324}
\min_{\rho'\in {\rm linspan}(\{\phi_i\}_{i=1}^{d_{\rm o}}) }\| \mathbb{K} \rho' - \mathbb{K}\rho \|_\cH^2
\#
for any $\rho\in \Delta(\cO^3)$. Note that the objective can be written as
\#\label{22610316}
&\| \mathbb{K} \rho' - \mathbb{K}\rho \|_\cH^2 \notag\\
&\quad=\la \mathbb{K} \rho', \mathbb{K} \rho'\ra_\cH
-2\cdot \la \mathbb{K} \rho', \mathbb{K} \rho\ra_\cH
+\la \mathbb{K} \rho, \mathbb{K} \rho'\ra_\cH \notag\\
&\quad=\EE_{X\sim \rho', Y \sim \rho'}[\cK(X,Y)]
-2\cdot\EE_{X\sim \rho', Y \sim \rho}[\cK(X,Y)]
+\EE_{X\sim \rho, Y \sim \rho}[\cK(X,Y)].
\#
Since we have $\rho'\in {\rm linspan}(\{\phi_i\}_{i=1}^{d_{\rm o}})$, there exists $w=(w_1,\ldots, w_{d_{\rm o}})^\top\in\RR^{d_{\rm o}}$ such that
\$
\rho'(x)=\sum_{j\in[d_{\rm o}]} w_j \cdot \phi_j(x), \quad
\text{for any $x\in\cO^3$.}
\$ 
By the above form of $\rho'$ and the definition of the matrix $G$ in \eqref{gmatdef}, we can further rewrite the right-hand side of \eqref{22610316} to obtain
\#\label{22610323}
&\| \mathbb{K} \rho' - \mathbb{K}\rho \|_\cH^2 
=w^\top G w - 2\cdot \sum_{i=1}^{d_{\rm o}} 
w_i \cdot \EE_{X\sim \phi_i, Y\sim \rho}[\cK(X,Y)]
+\EE_{X\sim \rho, Y \sim \rho}[\cK(X,Y)]
\#
Plugging \eqref{22610323} into \eqref{22610324} and solving the obtained quadratic programming problem, we see that $\hat{\rho}$ defined in \eqref{pdagger-def} is the projection of $\rho$.

By the definition of $\mathbb{S}$ in \eqref{sopdef}, we have
\#\label{22510552}
\EE_{X\sim \rho}[(\mathbb{S}f)(X)] = 
\sum_{i,j\in[d_{\rm o}]}
[G^{-1}]_{i,j} \cdot
\EE_{X\sim \rho, Y\sim\phi_i, Y'\sim\phi_j}\bigl[\cK(X, Y) \cdot f(Y') \bigr],
\#
where $X,Y,Y'$ are independent variables in $\cO^3$. By reorganizing terms in the summation, we can further write the right-hand side of \eqref{22510552} as
\$
&\sum_{i,j\in[d_{\rm o}]}
[G^{-1}]_{i,j} \cdot
\EE_{X\sim p, Y\sim\phi_i, Y'\sim\phi_j}\bigl[\cK(X, Y) \cdot f(Y') \bigr] \\
&\quad=
\sum_{j\in[d_{\rm o}]} \EE_{Y'\sim\phi_j}[f(Y')]
\cdot \sum_{i\in[d_{\rm o}]} [G^{-1}]_{i,j} \cdot \EE_{X\sim p, Y\sim\phi_i}[\cK(X,Y)] \\
&\quad=
\int_{\cO}
f(x) \cdot \sum_{j\in[d_{\rm o}]} \phi_j(x) \cdot \sum_{i\in[d_{\rm o}]} [G^{-1}]_{i,j} 
\cdot \EE_{X\sim p, Y\sim\phi_i}[\cK(X,Y)] \ud x,
\$
combining which with the definition of $\hat{\rho}$ in \eqref{pdagger-def}, we conclude the proof of Lemma \ref{proj-conjugate}.
\end{proof}

\subsection{Proof of Lemma \ref{event-lemma}}\label{event-lemma-p}
\begin{proof}
For any $(k,h,a,a')\in[K]\times[2,\ldots,H]\times\cA^2$, let $w^k_{h,a,a'}\in\RR^{d_{\rm o}}$ be the vector such that
\#\label{9121115}
\rho^{\overline{\pi}_k}_{h,a,a'}(o_{h-1}, o_h, o_{h+1})
=\sum_{j\in[d_{\rm o}]}\phi_i(o_{h-1}, o_h, o_{h+1}) 
\cdot [w^k_{h,a,a'}]_j,
\#
for any $o_{h-1}, o_h, o_{h+1}\in\cO$, where $\{\phi_i\}_{i=1}^{d_{\rm o}}$ are distribution functions defined in Assumption \ref{asmp1} and the existence of $w^k_{h,a,a'}$ is guaranteed by the assumption therein. Also, recall that we define $\hat{w}^k_{h,a,a'}\in\RR^{d_{\rm o}}$ in \eqref{225101139} and we have
\#\label{225101140}
\hat{\rho}^k_{h,a,a'}(o_{h-1}, o_h, o_{h+1})
=\sum_{j\in[d_{\rm o}]}\phi_j(o_{h-1}, o_h, o_{h+1}) 
\cdot [\hat{w}^k_{h,a,a'}]_j.
\#
For notational simplicity, we denote $\phi=(\phi_1, \ldots, \phi_{d_{\rm o}})$ and write $\rho^k_{h,a,a'}$ in short for $\rho^{\overline{\pi}_k}_{h,a,a'}$. Then, we can rewrite \eqref{9121115} and \eqref{225101140} as
\$
\rho^k_{h,a,a'}(\cdot)=\phi(\cdot)^\top w^k_{h,a,a'}, \quad
\hat{\rho}^k_{h,a,a'}(\cdot)=\phi(\cdot)^\top \hat{w}^k_{h,a,a'}.
\$
Following the above definitions, we have
\#\label{22511241}
&\| \hat{\rho}^k_{h,a,a'} - \rho^k_{h,a,a'} \|_1 \\
&\quad=
\int_{\cO^3} 
|\phi(o_{h-1}, o_h, o_{h+1})^\top 
(\hat{w}^k_{h,a,a'}-w^k_{h,a,a'})|
\ud o_{h-1} \ud o_h \ud o_{h+1} \notag\\
&\quad\le
\int_{\cO^3} 
\|\phi(o_{h-1}, o_h, o_{h+1})\|_2 \cdot
\|\hat{w}^k_{h,a,a'}-w^k_{h,a,a'} \|_2
\ud o_{h-1} \ud o_h \ud o_{h+1} \notag\\
&\quad\le
d_{\rm o}  \cdot \|\hat{w}^k_{h,a,a'}-w^k_{h,a,a'} \|_2, \notag
\#
where the first inequality is by the Cauchy-Schwarz inequality and the last inequality uses
\$
&\int_{\cO^3} \|\phi(o_{h-1}, o_h, o_{h+1})\|_2 \ud o_{h-1} \ud o_h \ud o_{h+1}\\
&\quad\le\int_{\cO^3}\|\phi(o_{h-1}, o_h, o_{h+1})\|_1 \ud o_{h-1} \ud o_h \ud o_{h+1}\\
&\quad=
\sum_{j=1}^{d_{\rm o}}\int_{\cO^3}\phi_j(o_{h-1}, o_h, o_{h+1}) \ud o_{h-1} \ud o_h \ud o_{h+1}
=d_{\rm o},
\$
as $\{\phi_i\}_{i=1}^{d_{\rm o}}$ are distribution functions over $\cO^3$. Thus, to upper bound $\| \hat{\rho}^k_{h,a,a'} - \rho^k_{h,a,a'} \|_1$, it suffices to upper bound $\|\hat{w}^k_{h,a,a'}-w^k_{h,a,a'} \|_2$.

To this end, we define the vector $U^k_{h,a,a'}\in\RR^{d_{\rm o}}$ by
\#\label{22511251}
[U^k_{h,a,a'}]_i = \EE_{X\sim\phi_i, Y\sim\hat{\cD}^k_{h,a,a'}}[\cK(X, Y)],
\#
where $\hat{\cD}^k_{h,a,a'}$ is the empirical distribution over $\cO^3$ induced by the dataset $\cD^k_{h,a,a'}$. Then, we can rewrite the definition of $\hat{w}^k_{h,a,a'}$ in \eqref{225101139} as
\$
\hat{w}^k_{h,a,a'}=G^{-1}U^k_{h,a,a'},
\$
where the matrix $G\in \RR^{d_{\rm o}\times d_{\rm o}}$ is defined in \eqref{gmatdef}. Then, we can write
\#\label{22511255}
\| \hat{w}^k_{h,a,a'} - w^k_{h,a,a'}\|_2
=\| G^{-1}U^k_{h,a,a'} - G^{-1}Gw^k_{h,a,a'}\|_2
\le 1/\alpha \cdot
\|U^k_{h,a,a'} - Gw^k_{h,a,a'}\|_2.
\#
Moreover, by the definition of $G$ in \eqref{gmatdef}, we have
\#\label{22511246}
[Gw^k_{h,a,a'}]_{i}&=\sum_{j\in[d_{\rm o}]} [G]_{i,j} \cdot [w^k_{h,a,a'}]_j \\
&=\sum_{j\in[d_{\rm o}]}
\EE_{X\sim\phi_i, Y\sim\phi_j}[\cK(X,Y)] \cdot [w^k_{h,a,a'}]_j \notag\\
&=
\int_{\cO^3\times\cO^3} 
\phi_i(x) \cdot \cK(x, y) \cdot \sum_{j\in[d_{\rm o}]} \phi_j(y) \cdot [w^k_{h,a,a'}]_j  \ud x\ud y,\notag
\#
By the definition of $w^k_{h,a,a'}$ in \eqref{9121115}, we can write \eqref{22511246} as
\#\label{22511250}
[Gw^k_{h,a,a'}]_{i}=
\int_{\cO^3\times\cO^3} 
\phi_i(x) \cdot \cK(x, y) \cdot \rho^k_{h,a,a'}(y) \ud x\ud y
=\EE_{X\sim\phi_i, Y\sim \rho^k_{h,a,a'}}[\cK(X, Y)]
\#
Using the notation of the RKHS $\cH$, we can further rewrite \eqref{22511251} and \eqref{22511250} as
\$
[U^k_{h,a,a'}]_i
=\la \mathbb{K}\phi_i, \mathbb{K}\hat{\cD}^k_{h,a,a'}  \ra_\cH, 
\quad
[Gw^k_{h,a,a'}]_{i}=
\la \mathbb{K}\phi_i, \mathbb{K}\rho^k_{h,a,a'}  \ra_\cH.
\$
Recall that $\mathbb{K}$ is defined in \eqref{rkhsembedding}. Therefore, using the Cauchy-Schwarz inequality for the inner product in $\cH$, we have
\#\label{22511406}
\|U^k_{h,a,a'} - Gw^k_{h,a,a'}\|^2_2
&=\sum_{i\in[d_{\rm o}]}
\bigl(\la \mathbb{K}\phi_i, \mathbb{K}\hat{\cD}^k_{h,a,a'} 
- \mathbb{K}\rho^k_{h,a,a'}  \ra_\cH \bigr)^2 \\
&\le\sum_{i\in[d_{\rm o}]}
\| \mathbb{K}\phi_i \|^2_\cH  \cdot \|\mathbb{K}\hat{\cD}^k_{h,a,a'} 
- \mathbb{K}\rho^k_{h,a,a'}  \|^2_\cH. \notag
\#
Since $\cK$ is uniformly bounded by 1 as specified in Assumption \ref{asmp3}, we have
\#\label{9121122}
\| \mathbb{K} \phi_{i}\|_\cH^2
=\EE_{X\sim\phi_i, Y\sim\phi_i}[\cK(X,Y)]
\le1.
\#
In the sequel, we characterize $ \|\mathbb{K}\hat{\cD}^k_{h,a,a'} - \mathbb{K}\rho^k_{h,a,a'}  \|_\cH$ on the right-hand side of \eqref{22511406} for any fixed $(h,a,a')\in\{2,\ldots,H\}\times\cA^2$. For notational simplicity, we denote by $Y_j$ the data point that is added to $\cD_{h,a,a'}$ in the $j$-th iteration of Algorithm \ref{algo}. In other words, we have
\$
\cD^k_{h,a,a'}=\{Y_1, \ldots, Y_k\}, \quad 
\text{for any $k\in[K]$.}
\$
Then, the random function process $\{\cM_j\}_{j\ge1}$ defined by 
\#\label{martdef}
\cM_j(\cdot)=(1/k)\cdot \Bigl(\sum_{i=1}^{\min\{j,k\}} \cK(Y_i,\cdot)
-\sum_{i=1}^{\min\{j,k\}} (\mathbb{K} \rho^{\pi_{i-1}}_{h,a,a'})(\cdot) \Bigr)
\#
is a martingale in $\cH$ adapted to the data filtration $\{\cU_j\}_{j\ge0}$ of Algorithm \ref{algo}. In detail, for any $j\in[K]$, we have that $\cU_j$ contains the information of all data collected in the first $j$ iterations of Algorithm \ref{algo}. Note that the data point $Y_j$  follows from the distribution $\rho^{\pi_{j-1}}_{h,a,a'}$ conditioning on $\cU_{j-1}$. Therefore, we have
\$
&\EE[\cK(Y_j, x) - (\mathbb{K} \rho^{\pi_{j-1}}_{h,a,a'})(x) \,|\, \cU_{j-1}] =
\EE[\cK(Y_j, x) \,|\, \cU_{j-1}]
-\EE_{Y\sim \rho^{\pi_{j-1}}_{h,a,a'}}[\cK(Y, x)]=0,
\$
for any fixed $x\in\cO^3$, which implies that $\{\cM_j\}_{j\ge1}$ defined above is a martingale. Moreover, we have that the total quadratic variation of $\{\cM_j\}_{j\ge1}$ is upper bounded by
\$
\sum_{i=1}^k (1/k^2)\cdot \| \cK(Y_i,\cdot)
- (\mathbb{K}\rho^{\pi_i}_{h,a,a'})(\cdot) \|^2_\cH \le \sum_{i=1}^k (1/k^2)\cdot 4=4/k,
\$
where the inequality uses the fact
\$
&\| \cK(Y_i,\cdot)
- (\mathbb{K}\rho^{\pi_i}_{h,a,a'})(\cdot) \|^2_\cH\\
&\quad=\cK(Y_i, Y_i) + \EE_{Y\sim\rho^{\pi_i}_{h,a,a'},  Y'\sim\rho^{\pi_i}_{h,a,a'}}[ \cK(Y,Y') ]
-2\cdot\EE_{Y\sim\rho^{\pi_i}_{h,a,a'}}[\cK(Y_i, Y)] \le 4
\$
following the same argument of \eqref{9121122}. Then, invoking Lemma \ref{pinelis} with 
\$
c^2=4/k \quad \text{and}\quad
\varepsilon=\sqrt{8\log(2KHA^2/\delta)} \cdot k^{-1/2}
\$ 
for any $\delta>0$, with probability at least $1-\delta/(KHA^2)$, it holds that
\#\label{9121132}
\|\cM_k\|_\cH=
 \| \mathbb{K}\hat{\cD}^k_{h,a,a'} -   \mathbb{K}\rho^k_{h,a,a'} \|_\cH
\le \sqrt{8\log(2KHA^2/\delta)} \cdot k^{-1/2}.
\#
Here, the equality uses the definition of $\rho^k_{h,a,a'}=\rho^{\overline{\pi}_k}_{h,a,a'}$. Recall that $\overline{\pi}_k$ is the mixing policy that uniformly selects a policy from $\{\pi_0, \ldots, \pi_{k-1}\}$ at random, which implies
\$
\rho^k_{h,a,a'}=\frac{1}{k}\sum_{i=0}^{k-1} \rho^{\pi_i}_{h,a,a'}
\quad\text{and}\quad
\mathbb{K}\rho^k_{h,a,a'}=\frac{1}{k}\sum_{i=0}^{k-1} \mathbb{K}\rho^{\pi_i}_{h,a,a'}.
\$
Then, by further applying the union bound, we have that, with probability at least $1-\delta$, the inequality in \eqref{9121132} holds for any $(k,h,a,a')\in[K]\times\{2,\ldots,H\}\times\cA^2$. Combining such an upper bound with \eqref{22511255}, \eqref{22511406} and \eqref{9121122}, we have
\#\label{22511259}
\|\hat{w}^k_{h,a,a'} - w^k_{h,a,a'}\|_2
\le d^{1/2}_{\rm o}/\alpha \cdot \sqrt{ 8\log(2KHA^2/\delta)} \cdot k^{-1/2}.
\#

Combining \eqref{22511241} and \eqref{22511259}, we know that, for any $\delta>0$, we have
\$
\| \hat{\rho}^k_{h,a,a'} - \rho^k_{h,a,a'} \|_1
\le d_{\rm o}^{3/2}/\alpha \cdot \sqrt{ 8\log(2KHA^2/\delta)} \cdot k^{-1/2},
\$
for any $(k,h,a,a')\in[K]\times\{2,\ldots,H\}\times\cA^2$, with probability at least $1-\delta$. Therefore, we conclude the proof of Lemma \ref{event-lemma}.
\end{proof}

\subsection{Proof of Lemma \ref{vop-lemma}}\label{vop-lemma-p}
\begin{proof}
By the definition of $\mathbb{V}^\theta_{h,a}$ in \eqref{vopdef}, we have
\#\label{22510619}
\EE_{X\sim p}[ (\mathbb{F}^\theta_{h,a}f)(X) ]
=\int_{\cO^3} f(x) \cdot (\mathbb{V}^\theta_{h,a} \rho)(x) \ud x.
\#
for any $f\in L^\infty(\cO^3)$ and $\rho\in\Delta(\cO^3)$. By combining \eqref{22510619} and Lemma \ref{esteq-lemma-g}, we have
\$
\int_{\cO^3} f(x) \cdot (\mathbb{V}^\theta_{h,a} \rho^{\theta,\pi}_{h,a,a'}
 - \rho^{\theta,\pi}_{h,a,a'})(x) \ud x,
\$
for any $f\in L^\infty(\cO^3)$, which implies
\$
\| \mathbb{V}^\theta_{h,a} \rho^{\theta,\pi}_{h,a,a'}
 -  \rho^{\theta,\pi}_{h,a,a'}\|_1=0.
\$

For any $\rho\in L^1(\cO^3)$ such that $\|\rho\|_1=1$, we have
\#\label{22510837}
\| \mathbb{V}^\theta_{h,a} \rho \|_1
&=\int_{\cO^3}
\Bigl|\int_{\cO^2} \cB^{\theta}_{h, a}(o_h,\tilde{o}_h, \tilde{o}_{h+1})
 \cdot \rho(o_{h-1}, o_h, o_{h+1})
 \ud o_h\ud o_{h+1} \Bigr| \ud o_{h-1}\ud \to_h \ud \to_{h+1}\\
 &\le 
\int_{\cS\times\cO^5} 
p_{\theta}(\tilde{\bo}_h=\tilde{o}_h, \tilde{\bo}_{h+1}=\tilde{o}_{h+1} \,|\, \tilde{\bs}_h=\tilde{s}_h, \tilde{\ba}_h=a) \cdot |\cZ^\theta_h(\tilde{s}_h, o_h)| \notag\\
&\quad\qquad\qquad\qquad\qquad 
\cdot |\rho(o_{h-1}, o_h, o_{h+1})|
\ud \tilde{s}_h \ud o_h\ud o_{h+1}  \ud o_{h-1}\ud \to_h \ud \to_{h+1} \notag\\
&=
\int_{\cS\times\cO^3} 
 |\cZ^\theta_h(\tilde{s}_h, o_h)|
\cdot |\rho(o_{h-1}, o_h, o_{h+1})|
\ud \tilde{s}_h \ud o_h\ud o_{h+1}  \ud o_{h-1}, \notag
\#
where the inequality is by the definition of $\cB^\theta_{h,a}$ in \eqref{b-func-def}. Combining \eqref{22510837} with \eqref{22510404} from the proof of Lemma \ref{esteq-lemma-g}, we have
\$
\| \mathbb{V}^\theta_{h,a} \rho \|_1 \le
\int_{\times\cO^3} 
\gamma
\cdot |\rho(o_{h-1}, o_h, o_{h+1})|
 \ud o_h\ud o_{h+1}  \ud o_{h-1} =\gamma,
\$
which implies $\|\mathbb{V}^\theta_{h,a}\|_{1\rightarrow1}\le \gamma$.

Therefore, we conclude the proof of Lemma \ref{vop-lemma}.
\end{proof}

\subsection{Property of the Bridge Operator}
\begin{lemma}[Bridge Property]\label{bridge-lm}
Recall that we denote by $\sigma_{h-1}$ the event
\$
\overbtau_{h-1}=\overtau_{h-1}, \quad \ba_{h-1}=a_{h-1}.
\$
For any $(h,\theta,\overtau_{h-1},a_{h-1})\in[H]\times\Theta\times\overline{\Gamma}_{h-1}\times\cA$, we have
\$
\EE_\theta[\cZ^\theta_h(\tilde{s}_h, \bo_h) \,|\, \sigma_{h-1}]=p_{\theta}(\tilde{\bs}_h=\tilde{s}_h \,|\, \sigma_{h-1}).
\$
\end{lemma}
\begin{proof}
By the tower property of the expectation and the Markov property of the POMDP, we have
\#\label{220507109}
\EE_\theta[\cZ^\theta_h(\tilde{s}_h, \bo_h) \,|\, \sigma_{h-1}]
=\EE_\theta\bigl[ \EE_\theta[\cZ^\theta_h(\tilde{s}_h, \bo_h) \,|\, \bs_{h-1}, \ba_{h-1}=a_{h-1}] \,\big|\,\sigma_{h-1}\bigr].
\#
Note that, for any $s_{h-1}\in\cS$, we have
\#\label{invstep1}
&\EE_\theta[\cZ^\theta_h(\tilde{s}_h, \bo_h) 
\,|\, \bs_{h-1}=s_{h-1}, \ba_{h-1}=a_{h-1}]\\
&\quad=\int_\cO \cZ^\theta_h(\tilde{s}_h, o_h) \cdot p_\theta(\bo_h=o_h\,|\,\bs_{h-1}=s_{h-1}, \ba_{h-1}=a_{h-1})\ud o_h \notag\\
&\quad=\int_{\cO\times\cS} \cZ^\theta_h(\tilde{s}_h, o_h) \cdot 
\cE^\theta_h(o_h \,|\,s_h) 
 \cdot
p_\theta(\bs_h=s_h\,|\,\bs_{h-1}=s_{h-1}, \ba_{h-1}=a_{h-1}) \ud s_h \ud o_h. \notag
\#
We define the function $f: \cS\rightarrow\RR$ by
\$
f(s_h)=p_\theta(\bs_h=s_h\,|\,\bs_{h-1}=s_{h-1}, \ba_{h-1}=a_{h-1}),
\quad
\text{for any $s_h\in\cS$.}
\$
Then, we have $f\in\cF_{\rm s}$ and we can write the right-hand side of \eqref{invstep1} as
\$
(\mathbb{Z}^\theta_h \mathbb{O}^\theta_h f)(\tilde{s}_h)=f(\tilde{s}_h)
=p_\theta(\bs_h=\tilde{s}_h\,|\,\bs_{h-1}=s_{h-1}, \ba_{h-1}=a_{h-1})
\$
following Assumptions \ref{asmp1} and \ref{asmp2}. In other words, we have
\#\label{invstep2}
&\EE_\theta[\cZ^\theta_h(\tilde{s}_h, \bo_h) 
\,|\, \bs_{h-1}=s_{h-1}, \ba_{h-1}=a_{h-1}] \\
&\quad=p_\theta(\bs_h=\tilde{s}_h\,|\,\bs_{h-1}=s_{h-1}, \ba_{h-1}=a_{h-1}). \notag
\#
Combining \eqref{220507109} and \eqref{invstep2} and using the Markov property of the POMDP, we obtain
\$
&\EE_\theta[\cZ^\theta_h(\tilde{s}_h, \bo_h) \,|\, \sigma_{h-1}]\\
&\quad=
\int_\cS
p_\theta(\bs_h=\tilde{s}_h\,|\,\bs_{h-1}=s_{h-1}, \ba_{h-1}=a_{h-1})
\cdot p_\theta(\bs_{h-1}=s_{h-1}\,|\, \sigma_{h-1}) \ud s_{h-1}\\
&\quad=
\int_\cS p_\theta(\bs_h=\tilde{s}_h\,|\,\bs_{h-1}=s_{h-1}, \sigma_{h-1})
\cdot p_\theta(\bs_{h-1}=s_{h-1}\,|\, \sigma_{h-1}) \ud s_{h-1}\\
&\quad=
p_\theta(\bs_{h}=\tilde{s}_{h}\,|\, \sigma_{h-1}),
\$
which concludes the proof of lemma \ref{bridge-lm}.
\end{proof}

The following lemma is a variant of Lemma \ref{blemma}, which adds the state information to the expectation condition. 
\begin{lemma}[Variant of Lemma \ref{blemma}]\label{blemma-g}
For any $(h,\theta,\pi,s_{h-1},\overline{\tau}_{h-1},a_{h-1})\in[H]\times\Theta\times\Pi\times\cS\times\overline{\Gamma}_{h-1}\times\cA$ and $f\in L^\infty(\overline{\Gamma}_{h+1})$, we have
\$
\EE_\theta[(\mB^{\theta,\pi}_{h} f)(\overbtau_{h})-f(\overbtau_{h+1})
\,|\, \bs_{h-1}=s_{h-1}, \sigma_{h-1}] = 0.
\$
\end{lemma}
\begin{proof}
The proof is very similar to the proof of Lemma \ref{blemma}. Following the notation in Lemma \ref{blemma}, by the definition of $\mB^{\theta,\pi}_{h}$ in \eqref{bdef}, we have
\#\label{91362402}
&\EE_{\theta}[ (\mB^{\theta,\pi}_{h} f)(\overbtau_h)\,|\, \bs_{h-1}=s_{h-1}, \sigma_{h-1}] \\
&\quad=
\int_{\cS\times\cO^3} f \bigl(\overtau_h^\dagger, \pi(\tau_h^\dagger), \to_{h+1} \bigr) \cdot 
p_{\theta}\bigl(\bo_h=\tilde{o}_h, \bo_{h+1}=\tilde{o}_{h+1} \,|\, \bs_h=\tilde{s}_h, \ba_h=\pi(\tau_h^\dagger) \bigr) \notag\\
&\quad\qquad\qquad\qquad
\cdot \cZ^\theta_{h}(\tilde{s}_h, o_h)
\cdot p_\theta(\bo_h=o_h \,|\,\bs_{h-1}=s_{h-1}, \sigma_{h-1}) \ud o_h  \ud \to_h \ud \to_{h+1} \ud \tilde{s}_h.
 \notag
\#
Here, by \eqref{invstep1}-\eqref{invstep2} in the proof of Lemma \ref{bridge-lm}, we have
\$
&\int_\cO \cZ^\theta_{h}(\tilde{s}_h, o_h)
\cdot p_\theta(\bo_h=o_h \,|\,\bs_{h-1}=s_{h-1},\sigma_{h-1})\ud o_h  \\
&\quad=\int_\cO \cZ^\theta_{h}(\tilde{s}_h, o_h)
\cdot p_\theta(\bo_h=o_h \,|\,\bs_{h-1}=s_{h-1}, \ba_{h-1}=a_{h-1})\ud o_h  \\
&\quad= p_\theta(\bs_h=\tilde{s}_h \,|\, \bs_{h-1}=s_{h-1}, \ba_{h-1}=a_{h-1})\\
&\quad= p_\theta(\bs_h=\tilde{s}_h \,|\, \bs_{h-1}=s_{h-1}, \sigma_{h-1}).
\$
Thus, we can rewrite \eqref{91362402} as
\#\label{9136242}
&\EE_{\theta}[ (\mB^{\theta,\pi}_{h} f)(\overbtau_h)\,|\,\bs_{h-1}=s_{h-1}, \sigma_{h-1}] \\
&\quad=
\int_{\cS\times\cO^2} f \bigl(\overtau_h^\dagger, \pi(\tau_h^\dagger), \to_{h+1} \bigr) \cdot 
p_{\theta}\bigl(\bo_h=\tilde{o}_h, \bo_{h+1}=\tilde{o}_{h+1} \,|\, \bs_h=\tilde{s}_h, \ba_h=\pi(\overtau_h^\dagger) \bigr) \notag\\
&\quad\qquad\qquad\qquad
\cdot p(\bs_h=\tilde{s}_h \,|\,\bs_{h-1}=s_{h-1},\sigma_{h-1})  \ud \to_h \ud \to_{h+1} \ud \tilde{s}_h
 \notag\\
 &\quad=
\int_{\cO^2} f \bigl(\overtau_h^\dagger, \pi(\tau_h^\dagger), \to_{h+1} \bigr) \cdot 
p_{\theta,\pi}(\bo_h=\tilde{o}_h, \bo_{h+1}=\tilde{o}_{h+1} \,|\, \bs_{h-1}=s_{h-1}, \sigma_{h-1} )
\ud \to_h \ud \to_{h+1}. \notag
\#
where the second equality uses the independence between $(\bo_h, \bo_{h+1})$ and $\btau_{h-1}$ conditioning on $(\bs_{h}, \ba_h)$. Replacing the notations $\to_h$ and $\to_{h+1}$ of the integral variables on the right-hand side of \eqref{9136242} by $o_h$ and $o_{h+1}$, respectively, we obtain
\#\label{2205077542}
&\EE_{\theta}[ (\mB^{\theta,\pi}_{h} f)(\overbtau_h)\,|\, \bs_{h-1}=s_{h-1},\sigma_{h-1}] \\
&\quad=
\int_{\cO^2} f \bigl(\overtau_h, \pi(\tau_h), o_{h+1} \bigr) \cdot 
p_{\theta,\pi}(\bo_h=o_h, \bo_{h+1}=o_{h+1} \,|\, \bs_{h-1}=s_{h-1}, \sigma_{h-1} )
\ud o_h \ud o_{h+1}\notag\\
&\quad=
\EE_{\theta,\pi}[f(\overbtau_{h+1}) \,|\, \bs_{h-1}=s_{h-1}, \sigma_{h-1}],\notag
\#
where we denote $\overtau_h=(\overtau_{h-1}, a_{h-1}, o_h)$ and $\tau_h=(\tau_{h-1}, o_h)$. Therefore, we conclude the proof of Lemma \ref{blemma-g}.
\end{proof}

\subsection{Properties of the Value Functions}
\begin{lemma}\label{911118}
For any $(h,\pi,\theta,\overtau_{h})\in[H]\times\Pi\times\Theta\times\overline{\Gamma}_{h}$, it holds that
\#\label{911126}
&V^{\theta,\pi}_{h}(\overtau_{h})=
\int_\cS \EE_{\theta,\pi}\Bigl[\sum_{i=1}^H \br_i\,\Big|\,
\bs_h=s_h, \overbtau_{h-1}=\overtau_{h-1}, \ba_{h-1}=a_{h-1}
\Bigr]
\cdot \cZ^\theta_h(s_h,o_h) \ud s_h.
\#
Here, we denote $\overtau_{h}=(\overtau_{h-1},a_{h-1},o_h)$ following the definition in \eqref{def-full-history}.

\end{lemma}
\begin{proof}
We prove the lemma by induction over $h\in[H]$. When $h=H$, by the definition of the value function in \eqref{vdef} and the definition of $\mB^{\theta,\pi}_H$ in \eqref{bdef}, we have
\#\label{911126-1}
V^{\theta,\pi}_{H}(\overtau_{H})
&=(\mB^{\theta,\pi}_{H}R)(\overtau_{H}) \\
&=\int_{\cO^2}
\Bigl( r\bigl( \to_H, \pi(\tau^\dagger_H) \bigr) +\sum_{i=1}^{H-1} r(o_h,a_h) \Bigr) \cdot \cB^{\theta}_{h, \pi(\tau^\dagger_H)}(o_H,\to_H,\to_{H+1}) \ud \to_H\ud \to_{H+1} \notag \\
&=\int_{\cO}
\Bigl( r\bigl( \to_H, \pi(\tau^\dagger_H) \bigr) +\sum_{i=1}^{H-1} r(o_h,a_h) \Bigr) \cdot \Bigl( \int_\cO \cB^{\theta}_{h, \pi(\tau^\dagger_H)}(o_H,\to_H,\to_{H+1})  \ud \to_{H+1}\Bigr)\ud\to_H. \notag
\#
Recall that $\tau^\dagger_H$ is the tail-mirrored observation history defined in \eqref{tail-mirrored-def}. Note that, by the definition of $\{\cB^\theta_{H,a}\}_{a\in\cA}$ in \eqref{b-func-def}, we have
\#\label{911126-2}
&\int_\cO \cB^\theta_{h, \pi(\tau^\dagger_H)}(o_H,\to_H,\to_{H+1})\ud \to_{H+1}\\
&\quad=\int_{\cS\times\cO} 
 p_{\theta}(\tilde{\bo}_H=\tilde{o}_H, \tilde{\bo}_{H+1}=\to_{H+1} \,|\, \tilde{\bs}_H=s_H, \tilde{\ba}_H=a)
\cdot \cZ^\theta_{H}(s_H, o_H) \ud s_H\ud \to_{H+1} \notag\\
&\quad=\int_{\cS} 
p_{\theta}(\bo_h=\tilde{o}_H \,|\, \bs_H=s_H, \ba_H=a)
\cdot \cZ^\theta_{H}(s_H, o_H) \ud s_H, \notag
\#
where we use the fact that $\tilde{\bs}_H$, $\tilde{\ba}_H$, $\tilde{\bo}_{H}$, and $\tilde{\bo}_{H+1}$ in \eqref{b-func-def} have the same distribution of $\bs_H$, $\ba_H$, $\bo_H$, and $\bo_{H+1}$. Combining \eqref{911126-1} and \eqref{911126-2}, we have that \eqref{911126} holds for $h=H$.

Assume that \eqref{911126} holds when $h= j+1$ for some fixed $j\le H-1$. Then, by the definition of the value function in \eqref{vdef}, we have
\$
V^{\theta,\pi}_{j}(\overtau_{j})
&=(\mB^{\theta,\pi}_jV^{\theta,\pi}_{j+1})(\overtau_{j}),
\quad
\text{for any $\overtau_{j}\in\overline{\Gamma}_j$.}
\$
Applying the induction assumption and definition of $\mB^{\theta,\pi}_j$ in \eqref{bdef}, we obtain
\#\label{220507509}
(\mB^{\theta,\pi}_jV^{\theta,\pi}_{j+1})(\overtau_{j}) &=
\int_{\cS\times\cO^2} \EE_{\theta,\pi}\Bigl[\sum_{i=1}^H \br_i\,\Big|\,
\bs_{j+1}=s_{j+1}, \overbtau_{j}=\overtau^\dagger_{j}, \ba_{j}=\pi(\tau^\dagger_j)
\Bigr]
\cdot \cZ^\theta_h(s_{j+1},\to_{j+1}) \notag \\
&\qquad\qquad\qquad\qquad
\cdot \cB^{\theta}_{j, \pi(\tau^\dagger_j)} (o_j, \to_j, \to_{j+1})
 \ud s_{j+1} \ud \to_j \ud \to_{j+1}. 
\#
Recall that $\tau^\dagger_j$ and $\overtau^\dagger_j$ are the tail-mirrored observation history and full history defined in \eqref{tail-mirrored-def}, respectively.
Note that, by the definition of $\{\cB^\theta_{j,a}\}_{a\in\cA}$ in \eqref{b-func-def}, we have
\#\label{220507550}
 &\cB^{\theta}_{j, \pi(\tau^\dagger_j)} (o_j, \to_j, \to_{j+1})\\
 &\quad=\int_\cS p_{\theta}\bigl(\bo_j=\tilde{o}_j, \bo_{j+1}=\to_{j+1} \,|\, \bs_j=s_j, \ba_j=\pi(\tau^\dagger_j)\bigr)
\cdot \cZ^\theta_{j}(s_j, o_j)\ud s_j \notag\\
&\quad=\int_\cS 
\cE^\theta_h(\to_j\,|\,s_j)\cdot
p_{\theta}\bigl(\bo_{j+1}=\to_{j+1} \,|\, \bs_j=s_j, \ba_j=\pi(\tau^\dagger_j)\bigr)
\cdot \cZ^\theta_{j}(s_j, o_j)\ud s_j \notag
\#
Following the same argument in \eqref{invstep1}-\eqref{invstep2}, we have
\$
&\int_{\cO}\cZ^\theta_h(s_{j+1},\to_{j+1}) \cdot p_{\theta}\bigl(\bo_{j+1}=\to_{j+1} \,|\, \bs_j=s_j, \ba_j=\pi(\tau^\dagger_j)\bigr)
\ud \to_{j+1} \\
&\quad=
p_{\theta}\bigl(\bs_{j+1}=s_{j+1} \,|\, \bs_j=s_j, \ba_j=\pi(\tau^\dagger_j)\bigr),
\$
combining which with \eqref{220507550}, we obtain
\# \label{220507510}
&\int_\cO 
\cZ^\theta_h(s_{j+1},\to_{j+1})\cdot
\cB^{\theta}_{j, \pi(\tau^\dagger_j)} (o_j, \to_j, \to_{j+1}) \ud \to_{j+1} \notag\\
&\quad=
\int_\cS \cE^\theta_h(\to_j\,|\,s_j)\cdot p_{\theta}\bigl( \bs_{j+1}=s_{j+1} \,|\, \bs_j=s_j, \ba_j=\pi(\tau^\dagger_j)\bigr)
\cdot \cZ^\theta_{j}(s_j, o_j) \ud s_j \notag\\
&\quad=
\int_\cS p_{\theta}\bigl(\bo_j=\to_j, \bs_{j+1}=s_{j+1} \,|\, \bs_j=s_j, \ba_j=\pi(\tau^\dagger_j)\bigr)
\cdot \cZ^\theta_{j}(s_j, o_j) \ud s_j.
\#
Plugging \eqref{220507510} into the right-hand side of \eqref{220507509}, we obtain
\#\label{220507527}
&(\mB^{\theta,\pi}_jV^{\theta,\pi}_{j+1})(\overtau_{j}) \\
&\quad=
\int_{\cS^2\times\cO} \EE_{\theta,\pi}\Bigl[\sum_{i=1}^H \br_i\,\Big|\,
\bs_{j+1}=s_{j+1}, \overbtau_{j}=\overtau^\dagger_{j}, \ba_{j}=\pi(\tau^\dagger_j)
\Bigr] \notag \notag\\
&\quad\qquad\qquad
\cdot p_{\theta}\bigl(\bo_j=\to_j, \bs_{j+1}=s_{j+1} \,|\, \bs_j=s_j, \ba_j=\pi(\tau^\dagger_j)\bigr)
\cdot \cZ^\theta_{j}(s_j, o_j)
 \ud s_j \ud s_{j+1} \ud \to_j. \notag
\#
Using the Markov property of the POMDP, we can simplify the right-hand side of \eqref{220507527} to obtain
\$
(\mB^{\theta,\pi}_jV^{\theta,\pi}_{j+1})(\overtau_{j}) 
=\int_\cS \EE_{\theta,\pi}\Bigl[\sum_{i=1}^H \br_i\,\Big|\,
\bs_j=s_j, \overbtau_{j-1}=\overtau_{j-1}, \ba_{j-1}=a_{j-1}
\Bigr]
\cdot \cZ^\theta_j(s_j,o_j) \ud s_j,
\$
which implies that \eqref{911126} holds when $h=j$. 

Therefore, we conclude the proof of Lemma \ref{911118} by induction.
\end{proof}

\begin{lemma}\label{coro2}
For any $(h,\theta,\overtau_{h},\pi)\in[H]\times\Theta\times \overline{\Gamma}_h\times\Pi$, we have
\$
|V^{\theta,\pi}_{h}(\overtau_{h})| \le \gamma H.
\$
Recall that $\gamma$ is defined in Assumption \ref{asmp2}.
\end{lemma}
\begin{proof}
By Lemma \ref{911118} and Assumption \ref{asmp2}, we have
\$
|V^{\theta,\pi}_{h}(\overtau_{h})|&=
\Bigl|\int_\cS \EE_{\theta,\pi}\Bigl[\sum_{i=1}^H \br_i\,\Big|\,
\bs_h=s_h, \overbtau_{h-1}=\overtau_{h-1}, \ba_{h-1}=a_{h-1}
\Bigr]
\cdot \cZ^\theta_h(s_h,o_h) \ud s_h \Bigr| \\
&\le H 
\cdot 
\int_\cS |\cZ^\theta_h(s_h,o_h)| \ud s_h
=H \cdot \| \mZ^\theta_h \delta_{o_h} \|_1 \le \gamma H \cdot \|\delta_{o_h} \|_1 =\gamma H,
\$
for any $(h,\theta,\overtau_{h},\pi)\in[H]\times\Theta\times \overline{\Gamma}_h\times\Pi$. Here, $\delta_{o_h}$ is the Dirac delta function defined on $\cO$, whose value is zero everywhere except at $o_h$, and whose integral over $\cO$ is equal to one. Thus, we conclude the proof of Lemma \ref{coro2}.
\end{proof}

\subsection{Properties of the State-Dependent Error}
\begin{lemma}\label{regret-transform}
For any $(k,h)\in[K]\times\{2,\ldots,H\}$, we have
\$
\EE_{\theta^*,\overline{\pi}_k}[e^k_h(\bs_h)]
&\le \gamma^2H\cdot\sum_{a,a'\in\cA}
\bigl\| \mathbb{V}^{\theta_k}_{h,a'} \rho^{\overline{\pi}_k}_{h,a,a'} - \rho^{\overline{\pi}_k}_{h,a,a'}\bigr\|_1.
\$
\end{lemma}
\begin{proof}
By the definition of $e^k_h$ in \eqref{edef}, we have
\#\label{220509231}
e^k_h(s_{h-1})=\bigl|  
\EE_{\theta^*,\pi_k}[
(\mB^{\theta_k,\pi_k}_{h}V^{\theta_k,\pi_k}_{h+1})(\overbtau_h) - (\mB^{\theta^*,\pi_k}_{h} V^{\theta_k,\pi_k}_{h+1})(\overbtau_{h})\,|\,\bs_{h-1}=s_{h-1}] \bigr|,
\#
for any $s_{h-1}\in\cS$. Note that, by the definition of $\mB^{\theta,\pi}_{h}$ in \eqref{bdef}, we have
\#\label{220509449}
&\EE_{\theta^*,\pi_k}[
(\mB^{\theta_k,\pi_k}_{h}V^{\theta_k,\pi_k}_{h+1})(\overbtau_h) - (\mB^{\theta^*,\pi_k}_{h} V^{\theta_k,\pi_k}_{h+1})(\overbtau_{h})
\,|\,\bs_{h-1}=s_{h-1}, \overbtau_{h-1}=\overtau_{h-1}] \\
&\quad=\int_{\cO^3} V^{\theta_k,\pi_k}_{h+1}
\bigl(\overtau^\dagger_{h}, \pi_k(\tau^\dagger_{h}), \to_{h+1}\bigr)
\cdot \Delta \cB^{\theta_k,\theta^*}_{h, \pi_k(\tau^\dagger_{h})}(o_h,\to_h,\to_{h+1}) \notag\\
&\quad\qquad\qquad \cdot
p_{\theta^*}\bigl(\bo_h=o_h \,\big|\, \bs_{h-1}=s_{h-1}, \ba_{h-1}=\pi_k(\tau_{h-1})\bigr)\ud o_h \ud \to_h \ud \to_{h+1}, \notag
\#
for any $(s_{h-1},\overtau_{h-1})\in\cS\times\overline{\Gamma}_{h-1}$,
where $\overtau_{h}^\dagger=(\overtau_{h-1}, a_{h-1}, \tilde{o}_h)$,
$\tau_{h}^\dagger=(\tau_{h-1}, \tilde{o}_h)$, and
\#\label{220509317}
\Delta \cB^{\theta_k,\theta^*}_{h, \pi_k(\tau^\dagger_{h})}(o_h,\to_h,\to_{h+1})
=\cB^{\theta_k}_{h, \pi_k(\tau^\dagger_{h})}(o_h,\to_h,\to_{h+1})
- \cB^{\theta^*}_{h, \pi_k(\tau^\dagger_{h})}(o_h,\to_h,\to_{h+1}).
\#
By replacing the actions $\pi_k(\tau_{h-1})$ and $\pi_k(\tau^\dagger_{h})$ on the right-hand side of \eqref{220509449} by all possible action combinations, we have the inequality
\#\label{220509506}
&\bigl| \EE_{\theta^*,\pi_k}[
(\mB^{\theta_k,\pi_k}_{h}V^{\theta_k,\pi_k}_{h+1})(\overbtau_h) - (\mB^{\theta^*,\pi_k}_{h} V^{\theta_k,\pi_k}_{h+1})(\overbtau_{h})
\,|\,\bs_{h-1}=s_{h-1}, \overbtau_{h-1}=\overtau_{h-1}] \bigr| \\
&\quad\le \sum_{a,a'\in\cA} \Bigl|\int_{\cO^3} V^{\theta_k,\pi_k}_{h+1}
(\overtau^\dagger_{h}, a', \to_{h+1})
\cdot \Delta \cB^{\theta_k,\theta^*}_{h, a'}(o_h,\to_h,\to_{h+1}) \notag\\
&\quad\qquad\qquad\qquad\quad \cdot
p_{\theta^*}\bigl(\bo_h=o_h \,\big|\, \bs_{h-1}=s_{h-1}, \ba_{h-1}=a\bigr)\ud o_h \ud \to_h \ud \to_{h+1} \Bigr|. \notag
\#
Invoking Lemma \ref{coro2}, we can further upper bound the left-hand side of \eqref{220509506} as
\#\label{220509506}
&\bigl| \EE_{\theta^*,\pi_k}[
(\mB^{\theta_k,\pi_k}_{h}V^{\theta_k,\pi_k}_{h+1})(\overbtau_h) - (\mB^{\theta^*,\pi_k}_{h} V^{\theta_k,\pi_k}_{h+1})(\overbtau_{h})
\,|\,\bs_{h-1}=s_{h-1}, \overbtau_{h-1}=\overtau_{h-1}] \bigr| \\
&\quad\le \sum_{a,a'\in\cA} 
\sup_{(\overtau^\dagger_{h}, \to_{h+1})\in\overline{\Gamma}_h\times \cO} |V^{\theta_k,\pi_k}_{h+1}
(\overtau^\dagger_{h}, a', \to_{h+1})|
\cdot
\int_{\cO^2}  \Bigl| \int_\cO\Delta \cB^{\theta_k,\theta^*}_{h, a'}(o_h,\to_h,\to_{h+1}) \notag\\
&\quad\qquad\qquad\qquad\qquad\qquad\quad \cdot
p_{\theta^*}(\bo_h=o_h \,|\, \bs_{h-1}=s_{h-1}, \ba_{h-1}=a)\ud o_h  \Bigr| \ud \to_h \ud \to_{h+1} \notag \\
&\quad\le \sum_{a,a'\in\cA} 
\gamma H
\cdot
\int_{\cO^2}  \Bigl| \int_\cO\Delta \cB^{\theta_k,\theta^*}_{h, a'}(o_h,\to_h,\to_{h+1}) \notag\\
&\quad\qquad\qquad\qquad\qquad\qquad\quad \cdot
p_{\theta^*}(\bo_h=o_h \,|\, \bs_{h-1}=s_{h-1}, \ba_{h-1}=a)\ud o_h  \Bigr| \ud \to_h \ud \to_{h+1} \notag.
\#
Combining \eqref{220509231} and \eqref{220509506}, and applying Jensen's inequality, we obtain
\#\label{220509247}
e^k_h(s_{h-1})
&\le  \sum_{a,a'\in\cA} 
\gamma H
\cdot
\int_{\cO^2}  \Bigl| \int_\cO\Delta \cB^{\theta_k,\theta^*}_{h, a'}(o_h,\to_h,\to_{h+1}) \\
&\quad\qquad\qquad\qquad\qquad\quad \cdot
p_{\theta^*}(\bo_h=o_h \,|\, \bs_{h-1}=s_{h-1}, \ba_{h-1}=a)\ud o_h  \Bigr| \ud \to_h \ud \to_{h+1}, \notag
\#
for any $s_{h-1}\in\cS$. 

In the sequel, we characterize the expectation of both sides of \eqref{220509247} with respect to the marginal distribution of $\bs_{h-1}$. We define the function $f:\cS\rightarrow\RR$ by
\#\label{2205091147}
f(s_{h-1})=\int_\cO\Delta \cB^{\theta_k,\theta^*}_{h, a'}(o_h,\to_h,\to_{h+1})
\cdot p_{\theta^*, \overline{\pi}_k}(\bo_h=o_h , \bs_{h-1}=s_{h-1} \,|\, \ba_{h-1}=a)\ud o_h,
\#
for any $s_{h-1}\in\cS$. With $f$ defined above and the inequality in \eqref{220509247}, we have
\#\label{220509356}
\EE_{\theta^*,\overline{\pi}_k}[e^k_H(\bs_{h-1})]
=\gamma H\cdot \sum_{a,a'\in\cA} \int_{\cO} \|f\|_1 \ud \to_h \to_{h+1},
\#
where we take the expectation of both sides of \eqref{220509247} with respect to the marginal distribution of $\bs_{h-1}$, following the policy $\overline{\pi}_k$. Note that we can write the probability on the right-hand side of \eqref{2205091147} as
\$
&p_{\theta^*, \overline{\pi}_k}(\bo_h=o_h , \bs_{h-1}=s_{h-1} \,\big|\, \ba_{h-1}=a)\\
&\quad=
p_{\theta^*, \overline{\pi}_k}(\bs_{h-1}=s_{h-1} \,|\, \bo_h=o_h,\ba_{h-1}=a)
\cdot p_{\theta^*, \overline{\pi}_k}(\bo_h=o_h \,|\, \ba_{h-1}=a),
\$
which implies $f\in\cF'_{\rm s}\subset{\rm linspan}(\{\psi_i\}_{i=1}^{d_{\rm s}})$ following Assumption \ref{asmp1}. Then, by further applying Assumption \ref{asmp2}, we obtain
\#\label{220509357}
\|f\|_1 = \| \mZ^\theta_{h-1} \mO^\theta_{h-1} f\|_1
\le \gamma \cdot \| \mO^\theta_{h-1} f\|_1.
\#
With $f$ defined in \eqref{2205091147} and $\mO^\theta_{h-1}$ defined in \eqref{emissionopdef}, we can write
\#\label{220509351}
&\|\mO^\theta_{h-1} f\|_1 \\
&\quad=\int_\cO\Bigl|\int_{\cO\times\cS}\Delta \cB^{\theta_k,\theta^*}_{h, a'}(o_h,\to_h,\to_{h+1})
\cdot p_{\theta^*, \overline{\pi}_k}(\bo_h=o_h, \bs_{h-1}=s_{h-1} \,|\, \ba_{h-1}=a)  \notag\\
&\quad\qquad\qquad\qquad\qquad
\cdot\cE^\theta_{h-1}(o_{h-1}\,|\,s_{h-1}) \ud o_h \ud s_{h-1} \Bigr|
\ud o_{h-1} \notag\\
&\quad=\int_\cO\Bigl|\int_{\cO}\Delta \cB^{\theta_k,\theta^*}_{h, a'}(o_h,\to_h,\to_{h+1})
\cdot p_{\theta^*, \overline{\pi}_k}(\bo_{h-1}=o_{h-1}, \bo_h=o_h \,|\, \ba_{h-1}=a)\ud o_h \Bigr| \ud o_{h-1}  \notag
\#
Here, we can rewrite the probability on the right-hand side as
\$
&p_{\theta^*, \overline{\pi}_k}(\bo_{h-1}=o_{h-1}, \bo_h=o_h \,|\, \ba_{h-1}=a)\\
&\quad=\int_{\cO}
p_{\theta^*, \overline{\pi}_k}(\bo_{h-1}=o_{h-1}, \bo_h=o_h,
\bo_{h+1}=o_{h+1}
 \,|\, \ba_{h-1}=a, \ba_h=a') \ud o_{h+1}\\
 &\quad=\int_\cO \rho^{\overline{\pi}_k}_{h,a,a'}(o_{h-1},o_h,o_{h+1})\ud o_{h+1}.
\$
Recall that $\Delta\cB^{\theta_k,\theta^*}_{h,a'}$ is defined in \eqref{220509317}. Then, by applying the definition of $\mathbb{V}^{\theta}_{h,a'}$ in \eqref{vopdef} for $\theta=\theta_k$ and $\theta=\theta^*$ to the right-hand side of \eqref{220509351} and integrating for $\to_h, \to_{h+1}$ over $\cO^2$ for both sides, we have
\#\label{220509358}
\int_{\cO^3} \|\mO^\theta_{h-1} f\|_1 \ud \to_h \ud \to_{h+1}
=\| \mathbb{V}^{\theta_k}_{h,a'} \rho^{\overline{\pi}_k}_{h,a,a'}
- \mathbb{V}^{\theta^*}_{h,a'} \rho^{\overline{\pi}_k}_{h,a,a'} \|_1
=\| \mathbb{V}^{\theta_k}_{h,a'} \rho^{\overline{\pi}_k}_{h,a,a'}
-  \rho^{\overline{\pi}_k}_{h,a,a'} \|_1,
\#
where the second equality is by Lemma \ref{vop-lemma}. Then, by combining \eqref{220509356}, \eqref{220509357} and \eqref{220509358}, we obtain
\$
\EE_{\theta^*, \overline{\pi}_k}[e^k_h(\bs_{h-1})]
&\le\gamma H \cdot\sum_{a,a'\in\cA} \int_{\cO^3} \| f\|_1 \ud \to_h \ud \to_{h+1} \\
&\le 
\gamma^2 H \cdot\sum_{a,a'\in\cA} \int_{\cO^3} \| \mO^\theta_{h-1}f\|_1 \ud \to_h \ud \to_{h+1}\\
&=
\gamma^2 H \cdot\sum_{a,a'\in\cA} \| \mathbb{V}^{\theta_k}_h \rho^{\overline{\pi}_k}_{h,a,a'}
-  \rho^{\overline{\pi}_k}_{h,a,a'} \|_1,
\$
which concludes the proof of Lemma \ref{regret-transform}.
\end{proof}

\begin{lemma}\label{eupperbound}
For any $(k,h)\in[K]\times\{2,\ldots,H\}$ and $s_{h-1}\in\cS$, we have
\$
e^k_h(s_{h-1}) \le 2\gamma H.
\$
\end{lemma}
\begin{proof}
For any $(k,h,s_{h-1},\overtau_{h-1},a_{h-1})\in[K]\times\{2,\ldots,H\}\times\cS\times\overline{\Gamma}_{h-1}\times\cA$, we have
\#\label{2205121136}
&\bigl|\EE_{\theta^*}[( \mB^{\theta_k,\pi_k}_{h}V^{\theta_k,\pi_k}_{h+1} )(\overbtau_{h})
\,|\, \bs_{h-1}=s_{h-1}, \overbtau_{h-1}=\overtau_{h-1},
\ba_{h-1}=a_{h-1} ] \bigr| \\
&\quad=\bigl| \EE_{\theta^*}[ V^{\theta_k,\pi_k}_{h} (\overbtau_{h})
\,|\, \bs_{h-1}=s_{h-1}, \overbtau_{h-1}=\overtau_{h-1},
\ba_{h-1}=a_{h-1} ]  \bigr| \le \gamma H, \notag
\#
where the equality uses the definition of the value function in \eqref{vdef} and the inequality is by Lemma \ref{coro2}. Similarly, by Lemma \ref{blemma-g}, we have
\#\label{2205121137}
&\bigl| \EE_{\theta^*}[ (\mB^{\theta^*,\pi_k}_{h} V^{\theta_k,\pi_k}_{h+1})(\overbtau_{h})
\,|\, \bs_{h-1}=s_{h-1}, \overbtau_{h-1}=\overtau_{h-1},
\ba_{h-1}=a_{h-1} ] \bigr| \\
&\quad=\bigl| \EE_{\theta^*, \pi_k}[ V^{\theta_k,\pi_k}_{h+1}(\overbtau_{h+1})
\,|\, \bs_{h-1}=s_{h-1}, \overbtau_{h-1}=\overtau_{h-1},
\ba_{h-1}=a_{h-1} ] \bigr|  \le \gamma H. \notag
\#
Combining \eqref{2205121136} and \eqref{2205121137}, and using the triangle inequality, we have
\$
\bigl| \EE_{\theta^*}[ (\mB^{\theta_k,\pi_k}_{h}V^{\theta_k,\pi_k}_{h+1} - \mB^{\theta^*,\pi_k}_{h} V^{\theta_k,\pi_k}_{h+1})(\overbtau_{h})
\,|\, \bs_{h-1}=s_{h-1}, \overbtau_{h-1}=\overtau_{h-1},
\ba_{h-1}=a_{h-1} ] \bigr| \le 2\gamma H,
\$
which, by Jensen's inequality, implies
\$
e^k_h(s_{h-1})= \bigl|  \EE_{\theta^*,\pi_k}[(\mB^{\theta_k,\pi_k}_{h}V^{\theta_k,\pi_k}_{h+1}
 - \mB^{\theta^*,\pi_k}_{h} V^{\theta_k,\pi_k}_{h+1})(\overbtau_{h})
\,|\, \bs_{h-1}=s_{h-1}]  \bigr| \le 2\gamma H.
\$
Therefore, we conclude the proof of Lemma \ref{eupperbound}.
\end{proof}

\subsection{Concentration Inequality}

\begin{lemma}\label{pinelis}
Suppose that $\{\cM_j\}_{j\ge1}$ is a martingale defined on a Hilbert space $\cH$. For any $c>0$, if we have
\#\label{912514}
\sum_{j=1}^\infty \|\cM_{j+1}-\cM_j\|^2_\cH \le c^2,
\#
then, for any $\varepsilon>0$, it holds that
\$
\cP\Bigl(\sup_{j\ge1}\|\cM_j\|_\cH \ge\varepsilon \Bigr)\le 2\exp\Bigl\{-\frac{\varepsilon^2}{2c^2}\Bigr\}.
\$
\end{lemma}
\begin{proof}
The lemma is a special case of Theorem 3.5 in \cite{pinelis1994optimum} (see also, Theorem 3 in \cite{pinelis1992approach}), which is a more general result for martingales in Banach spaces. 
\end{proof}

%% file: main.bbl
\begin{thebibliography}{40}
\expandafter\ifx\csname natexlab\endcsname\relax\def\natexlab#1{#1}\fi
\expandafter\ifx\csname url\endcsname\relax
  \def\url#1{\texttt{#1}}\fi
\expandafter\ifx\csname urlprefix\endcsname\relax\def\urlprefix{}\fi

\bibitem[{Agarwal et~al.(2020)Agarwal, Henaff, Kakade and Sun}]{agarwal2020pc}
\text{Agarwal, A.}, \text{Henaff, M.}, \text{Kakade, S.} and \text{Sun, W.}
  (2020).
\newblock {PC-PG}: Policy cover directed exploration for provable policy
  gradient learning.
\newblock \textit{arXiv preprint arXiv:2007.08459}.

\bibitem[{Arulkumaran et~al.(2019)Arulkumaran, Cully and
  Togelius}]{arulkumaran2019alphastar}
\text{Arulkumaran, K.}, \text{Cully, A.} and \text{Togelius, J.} (2019).
\newblock Alphastar: An evolutionary computation perspective.
\newblock In \textit{Proceedings of the genetic and evolutionary computation
  conference companion}.

\bibitem[{Auer et~al.(2008)Auer, Jaksch and Ortner}]{auer2008near}
\text{Auer, P.}, \text{Jaksch, T.} and \text{Ortner, R.} (2008).
\newblock Near-optimal regret bounds for reinforcement learning.
\newblock In \textit{Advances in Neural Information Processing Systems}.

\bibitem[{Ayoub et~al.(2020)Ayoub, Jia, Szepesvari, Wang and
  Yang}]{ayoub2020model}
\text{Ayoub, A.}, \text{Jia, Z.}, \text{Szepesvari, C.}, \text{Wang, M.} and
  \text{Yang, L.} (2020).
\newblock Model-based reinforcement learning with value-targeted regression.
\newblock In \textit{International Conference on Machine Learning}.

\bibitem[{Azar et~al.(2017)Azar, Osband and Munos}]{azar2017minimax}
\text{Azar, M.~G.}, \text{Osband, I.} and \text{Munos, R.} (2017).
\newblock Minimax regret bounds for reinforcement learning.
\newblock In \textit{International Conference on Machine Learning}.

\bibitem[{Azizzadenesheli et~al.(2016)Azizzadenesheli, Lazaric and
  Anandkumar}]{azizzadenesheli2016reinforcement}
\text{Azizzadenesheli, K.}, \text{Lazaric, A.} and \text{Anandkumar, A.}
  (2016).
\newblock Reinforcement learning of {POMDP}s using spectral methods.
\newblock In \textit{Conference on Learning Theory}.

\bibitem[{Bennett and Kallus(2021)}]{bennett2021proximal}
\text{Bennett, A.} and \text{Kallus, N.} (2021).
\newblock Proximal reinforcement learning: Efficient off-policy evaluation in
  partially observed markov decision processes.
\newblock \textit{arXiv preprint arXiv:2110.15332}.

\bibitem[{Berner et~al.(2019)Berner, Brockman, Chan, Cheung, Debiak, Dennison,
  Farhi, Fischer, Hashme, Hesse et~al.}]{berner2019dota}
\text{Berner, C.}, \text{Brockman, G.}, \text{Chan, B.}, \text{Cheung, V.},
  \text{Debiak, P.}, \text{Dennison, C.}, \text{Farhi, D.}, \text{Fischer, Q.},
  \text{Hashme, S.}, \text{Hesse, C.} \text{et~al.} (2019).
\newblock Dota 2 with large scale deep reinforcement learning.
\newblock \textit{arXiv preprint arXiv:1912.06680}.

\bibitem[{Brown and Sandholm(2018)}]{brown2018superhuman}
\text{Brown, N.} and \text{Sandholm, T.} (2018).
\newblock Superhuman {AI} for heads-up no-limit poker: Libratus beats top
  professionals.
\newblock \textit{Science}, \textbf{359} 418--424.

\bibitem[{Cai et~al.(2020)Cai, Yang, Jin and Wang}]{cai2020provably}
\text{Cai, Q.}, \text{Yang, Z.}, \text{Jin, C.} and \text{Wang, Z.} (2020).
\newblock Provably efficient exploration in policy optimization.
\newblock In \textit{International Conference on Machine Learning}.

\bibitem[{Cassandra et~al.(1996)Cassandra, Kaelbling and
  Kurien}]{cassandra1996acting}
\text{Cassandra, A.~R.}, \text{Kaelbling, L.~P.} and \text{Kurien, J.~A.}
  (1996).
\newblock Acting under uncertainty: Discrete bayesian models for mobile-robot
  navigation.
\newblock In \textit{Proceedings of IEEE/RSJ International Conference on
  Intelligent Robots and Systems}.

\bibitem[{Du et~al.(2021)Du, Kakade, Lee, Lovett, Mahajan, Sun and
  Wang}]{du2021bilinear}
\text{Du, S.~S.}, \text{Kakade, S.~M.}, \text{Lee, J.~D.}, \text{Lovett, S.},
  \text{Mahajan, G.}, \text{Sun, W.} and \text{Wang, R.} (2021).
\newblock Bilinear classes: A structural framework for provable generalization
  in {RL}.
\newblock \textit{arXiv preprint arXiv:2103.10897}.

\bibitem[{Golowich et~al.(2022)Golowich, Moitra and
  Rohatgi}]{golowich2022planning}
\text{Golowich, N.}, \text{Moitra, A.} and \text{Rohatgi, D.} (2022).
\newblock Planning in observable {POMDP}s in quasipolynomial time.
\newblock \textit{arXiv preprint arXiv:2201.04735}.

\bibitem[{Guo et~al.(2016)Guo, Doroudi and Brunskill}]{guo2016pac}
\text{Guo, Z.~D.}, \text{Doroudi, S.} and \text{Brunskill, E.} (2016).
\newblock A {PAC} {RL} algorithm for episodic {POMDP}s.
\newblock In \textit{Artificial Intelligence and Statistics}.

\bibitem[{Hauskrecht and Fraser(2000)}]{hauskrecht2000planning}
\text{Hauskrecht, M.} and \text{Fraser, H.} (2000).
\newblock Planning treatment of ischemic heart disease with partially
  observable markov decision processes.
\newblock \textit{Artificial Intelligence in Medicine}, \textbf{18} 221--244.

\bibitem[{Hsu et~al.(2012)Hsu, Kakade and Zhang}]{hsu2012spectral}
\text{Hsu, D.}, \text{Kakade, S.~M.} and \text{Zhang, T.} (2012).
\newblock A spectral algorithm for learning hidden {M}arkov models.
\newblock \textit{Journal of Computer and System Sciences}, \textbf{78}
  1460--1480.

\bibitem[{Jaeger(2000)}]{jaeger2000observable}
\text{Jaeger, H.} (2000).
\newblock Observable operator models for discrete stochastic time series.
\newblock \textit{Neural computation}, \textbf{12} 1371--1398.

\bibitem[{Jin et~al.(2018)Jin, Allen-Zhu, Bubeck and Jordan}]{jin2018q}
\text{Jin, C.}, \text{Allen-Zhu, Z.}, \text{Bubeck, S.} and \text{Jordan,
  M.~I.} (2018).
\newblock Is {Q}-learning provably efficient?
\newblock In \textit{Advances in Neural Information Processing Systems}.

\bibitem[{Jin et~al.(2020{\natexlab{a}})Jin, Kakade, Krishnamurthy and
  Liu}]{jin2020sample}
\text{Jin, C.}, \text{Kakade, S.~M.}, \text{Krishnamurthy, A.} and \text{Liu,
  Q.} (2020{\natexlab{a}}).
\newblock Sample-efficient reinforcement learning of undercomplete {POMDP}s.
\newblock \textit{arXiv preprint arXiv:2006.12484}.

\bibitem[{Jin et~al.(2020{\natexlab{b}})Jin, Yang, Wang and
  Jordan}]{jin2020provably}
\text{Jin, C.}, \text{Yang, Z.}, \text{Wang, Z.} and \text{Jordan, M.~I.}
  (2020{\natexlab{b}}).
\newblock Provably efficient reinforcement learning with linear function
  approximation.
\newblock In \textit{Conference on Learning Theory}.

\bibitem[{Kakade et~al.(2020)Kakade, Krishnamurthy, Lowrey, Ohnishi and
  Sun}]{kakade2020information}
\text{Kakade, S.}, \text{Krishnamurthy, A.}, \text{Lowrey, K.}, \text{Ohnishi,
  M.} and \text{Sun, W.} (2020).
\newblock Information theoretic regret bounds for online nonlinear control.
\newblock \textit{arXiv preprint arXiv:2006.12466}.

\bibitem[{Kallus et~al.(2021)Kallus, Mao and Uehara}]{kallus2021causal}
\text{Kallus, N.}, \text{Mao, X.} and \text{Uehara, M.} (2021).
\newblock Causal inference under unmeasured confounding with negative controls:
  A minimax learning approach.
\newblock \textit{arXiv preprint arXiv:2103.14029}.

\bibitem[{Kozuno et~al.(2021)Kozuno, M{\'e}nard, Munos and
  Valko}]{kozuno2021model}
\text{Kozuno, T.}, \text{M{\'e}nard, P.}, \text{Munos, R.} and \text{Valko, M.}
  (2021).
\newblock Model-free learning for two-player zero-sum partially observable
  {M}arkov games with perfect recall.
\newblock \textit{arXiv preprint arXiv:2106.06279}.

\bibitem[{Kwon et~al.(2021)Kwon, Efroni, Caramanis and Mannor}]{kwon2021rl}
\text{Kwon, J.}, \text{Efroni, Y.}, \text{Caramanis, C.} and \text{Mannor, S.}
  (2021).
\newblock {RL} for latent {MDP}s: Regret guarantees and a lower bound.
\newblock \textit{arXiv preprint arXiv:2102.04939}.

\bibitem[{Mnih et~al.(2015)Mnih, Kavukcuoglu, Silver, Rusu, Veness, Bellemare,
  Graves, Riedmiller, Fidjeland, Ostrovski et~al.}]{mnih2015human}
\text{Mnih, V.}, \text{Kavukcuoglu, K.}, \text{Silver, D.}, \text{Rusu, A.~A.},
  \text{Veness, J.}, \text{Bellemare, M.~G.}, \text{Graves, A.},
  \text{Riedmiller, M.}, \text{Fidjeland, A.~K.}, \text{Ostrovski, G.}
  \text{et~al.} (2015).
\newblock Human-level control through deep reinforcement learning.
\newblock \textit{nature}, \textbf{518} 529--533.

\bibitem[{Nair and Jiang(2021)}]{nair2021spectral}
\text{Nair, Y.} and \text{Jiang, N.} (2021).
\newblock A spectral approach to off-policy evaluation for {POMDP}s.
\newblock \textit{arXiv preprint arXiv:2109.10502}.

\bibitem[{Osband et~al.(2016)Osband, Van~Roy and
  Wen}]{osband2016generalization}
\text{Osband, I.}, \text{Van~Roy, B.} and \text{Wen, Z.} (2016).
\newblock Generalization and exploration via randomized value functions.
\newblock In \textit{International Conference on Machine Learning}.

\bibitem[{Pearl(2009)}]{pearl2009causal}
\text{Pearl, J.} (2009).
\newblock Causal inference in statistics: An overview.
\newblock \textit{Statistics surveys}, \textbf{3} 96--146.

\bibitem[{Pinelis(1992)}]{pinelis1992approach}
\text{Pinelis, I.} (1992).
\newblock An approach to inequalities for the distributions of
  infinite-dimensional martingales.
\newblock In \textit{Probability in Banach Spaces}. Springer.

\bibitem[{Pinelis(1994)}]{pinelis1994optimum}
\text{Pinelis, I.} (1994).
\newblock Optimum bounds for the distributions of martingales in {B}anach
  spaces.
\newblock \textit{The Annals of Probability} 1679--1706.

\bibitem[{Rafferty et~al.(2011)Rafferty, Brunskill, Griffiths and
  Shafto}]{rafferty2011faster}
\text{Rafferty, A.~N.}, \text{Brunskill, E.}, \text{Griffiths, T.~L.} and
  \text{Shafto, P.} (2011).
\newblock Faster teaching by {POMDP} planning.
\newblock In \textit{International Conference on Artificial Intelligence in
  Education}. Springer.

\bibitem[{Shi et~al.(2020)Shi, Miao and Tchetgen}]{shi2020selective}
\text{Shi, X.}, \text{Miao, W.} and \text{Tchetgen, E.~T.} (2020).
\newblock A selective review of negative control methods in epidemiology.
\newblock \textit{Current epidemiology reports}, \textbf{7} 190--202.

\bibitem[{Silver et~al.(2016)Silver, Huang, Maddison, Guez, Sifre, Van
  Den~Driessche, Schrittwieser, Antonoglou, Panneershelvam, Lanctot
  et~al.}]{silver2016mastering}
\text{Silver, D.}, \text{Huang, A.}, \text{Maddison, C.~J.}, \text{Guez, A.},
  \text{Sifre, L.}, \text{Van Den~Driessche, G.}, \text{Schrittwieser, J.},
  \text{Antonoglou, I.}, \text{Panneershelvam, V.}, \text{Lanctot, M.}
  \text{et~al.} (2016).
\newblock Mastering the game of {G}o with deep neural networks and tree search.
\newblock \textit{nature}, \textbf{529} 484--489.

\bibitem[{Silver et~al.(2017)Silver, Schrittwieser, Simonyan, Antonoglou,
  Huang, Guez, Hubert, Baker, Lai, Bolton et~al.}]{silver2017mastering}
\text{Silver, D.}, \text{Schrittwieser, J.}, \text{Simonyan, K.},
  \text{Antonoglou, I.}, \text{Huang, A.}, \text{Guez, A.}, \text{Hubert, T.},
  \text{Baker, L.}, \text{Lai, M.}, \text{Bolton, A.} \text{et~al.} (2017).
\newblock Mastering the game of {G}o without human knowledge.
\newblock \textit{nature}, \textbf{550} 354--359.

\bibitem[{Smola and Sch{\"o}lkopf(1998)}]{smola1998learning}
\text{Smola, A.~J.} and \text{Sch{\"o}lkopf, B.} (1998).
\newblock \textit{Learning with kernels}, vol.~4.
\newblock Citeseer.

\bibitem[{Vlassis et~al.(2012)Vlassis, Littman and
  Barber}]{vlassis2012computational}
\text{Vlassis, N.}, \text{Littman, M.~L.} and \text{Barber, D.} (2012).
\newblock On the computational complexity of stochastic controller optimization
  in {POMDP}s.
\newblock \textit{ACM Transactions on Computation Theory (TOCT)}, \textbf{4}
  1--8.

\bibitem[{Xiong et~al.(2021)Xiong, Chen, Gao and Zhou}]{xiong2021sublinear}
\text{Xiong, Y.}, \text{Chen, N.}, \text{Gao, X.} and \text{Zhou, X.} (2021).
\newblock Sublinear regret for learning {POMDP}s.
\newblock \textit{arXiv preprint arXiv:2107.03635}.

\bibitem[{Yang and Wang(2020)}]{yang2020reinforcement}
\text{Yang, L.} and \text{Wang, M.} (2020).
\newblock Reinforcement learning in feature space: Matrix bandit, kernels, and
  regret bound.
\newblock In \textit{International Conference on Machine Learning}.

\bibitem[{Zhang and Bareinboim(2016)}]{zhang2016markov}
\text{Zhang, J.} and \text{Bareinboim, E.} (2016).
\newblock Markov decision processes with unobserved confounders: A causal
  approach.
\newblock Tech. rep., Technical report, Technical Report R-23, Purdue AI Lab.

\bibitem[{Zhou et~al.(2021)Zhou, Gu and Szepesvari}]{zhou2021nearly}
\text{Zhou, D.}, \text{Gu, Q.} and \text{Szepesvari, C.} (2021).
\newblock Nearly minimax optimal reinforcement learning for linear mixture
  {M}arkov decision processes.
\newblock In \textit{Conference on Learning Theory}.

\end{thebibliography}
